\PassOptionsToPackage{hidelinks}{hyperref}
\documentclass[11pt]{article}

\usepackage[margin=1in]{geometry}
\usepackage[T1]{fontenc}
\usepackage[utf8]{inputenc}
\usepackage{lmodern}
\usepackage{microtype}
\usepackage{amsmath,amssymb,amsthm,graphicx}

\usepackage{natbib}
\usepackage{algorithm}
\usepackage{algorithmic}
\usepackage{hyperref}
\usepackage{booktabs}
\usepackage{enumitem}
\usepackage{comment}
\usepackage{tikz}
\usetikzlibrary{shapes.geometric, arrows.meta, positioning, calc, backgrounds, fit}

\theoremstyle{plain}
\newtheorem{theorem}{Theorem}[section]
\newtheorem{lemma}[theorem]{Lemma}
\newtheorem{proposition}[theorem]{Proposition}
\newtheorem{corollary}[theorem]{Corollary}
\theoremstyle{definition}
\newtheorem{definition}[theorem]{Definition}
\newtheorem{assumption}[theorem]{Assumption}
\theoremstyle{remark}
\newtheorem{remark}[theorem]{Remark}

\newtheorem{example}{Example}

\newcommand{\E}{\mathbb{E}}
\newcommand{\KL}{\mathrm{KL}}

\title{SMGI: A Structural Theory of General Artificial Intelligence \\[0.6em]\normalsize\textbf{Preprint}}
\author{Aomar OSMANI\\INSA-Rouen, France\\\texttt{aomar.osmani@insa-rouen.fr}}
\date{March 2026}

\begin{document}
\maketitle
\begin{center}
{\small\itshape This manuscript is a preprint }\par
\end{center}
\vspace{0.5em}
\begin{abstract}
We introduce SMGI, a structural theory of general artificial intelligence, and recast the foundational problem of learning from the optimization of hypotheses within fixed environments to the controlled evolution of the learning interface itself. We formalize the Structural Model of General Intelligence (SMGI) via a typed meta-model $\theta = (r,\mathcal H,\Pi,\mathcal L,\mathcal E,\mathcal M)$ that treats representational maps, hypothesis spaces, structural priors, multi-regime evaluators, and memory operators as explicitly typed, dynamic components. By enforcing a strict mathematical separation between this structural ontology ($\theta$) and its induced behavioral semantics ($T_\theta$), we define general artificial intelligence as a class of admissible coupled dynamics $(\theta, T_\theta)$ satisfying four obligations: structural closure under typed transformations, dynamical stability under certified evolution, bounded statistical capacity, and evaluative invariance across regime shifts.
We prove a structural generalization bound that links sequential PAC-Bayes analysis and Lyapunov stability, providing sufficient conditions for capacity control and bounded drift under admissible task transformations. Furthermore, we establish a strict structural inclusion theorem demonstrating that classical empirical risk minimization, reinforcement learning, program-prior models (Solomonoff-style), and modern frontier agentic pipelines operate as structurally restricted instances of SMGI.

Grounded in epistemological constraints that separate descriptive prediction from normative evaluation, this framework proposes a mathematically explicit extension of statistical learning theory where evaluation itself becomes an object of certified evolution. Finally, we propose an empirical protocol for measuring structural growth and memory-governed stability in long-horizon, nonstationary regimes.

\end{abstract}

\section{Introduction}

Artificial General Intelligence (AGI) is a foundational scientific objective: to characterize learning-and-decision systems that remain coherent and effective across heterogeneous task regimes, evolving interfaces, and shifting evaluative conditions.
Contemporary systems have demonstrated striking advances in language modeling, multimodal integration, tool-augmented reasoning, and long-horizon planning
\cite{openai2023gpt4,anthropic2024claude3,google2024gemini,deepseek2024v3,qwen2024technical,brown2020gpt3}.
Yet a central theoretical question remains unresolved:

\begin{quote}
\emph{What structural principle distinguishes scalable task generalization from genuinely general intelligence?}
\end{quote}

\paragraph{Diagnosis: Local guarantees versus structural transformations.}
Existing theoretical frameworks provide powerful \emph{local} guarantees.
Statistical learning theory controls generalization within fixed hypothesis classes \cite{vapnik1998statistical,valiant1984};
PAC-Bayes and MDL/SRM refine capacity control under fixed representations and losses \cite{mcallester1999,catoni2007pacbayes,grunwald2007mdl};
domain adaptation quantifies transfer under a fixed input--output interface \cite{bendavid2010,arjovsky2019irm};
and reinforcement learning formalizes sequential decision-making under fixed state/action semantics \cite{mnih2015dqn,schulman2017ppo}.
However, modern agentic deployments increasingly operate under \emph{structural transformations}: new tools and interfaces, shifting objectives, multi-regime evaluation, evolving memory and retrieval policies, and adversarially perturbed contexts
\cite{yao2023react,park2023generativeagents,greshake2023ipi}.
In such settings, the learning problem is to generalize across tasks while also remaining controllable under \emph{admissible changes of the interface itself}.
\paragraph{The fragmentation of current solutions: partial internal fixes and external wrappers.}
To address these agentic regime shifts, advanced RL and continual-learning lines have introduced \emph{internal} but partial mechanisms: hierarchical RL and option invention aim to restructure policies and reuse temporally extended skills \citep{nayyar2024optioninvention,ldsc2025llmoptions,shek2025ldsc}, while continual learning proposes global alignment and structured sparse memory updates to mitigate catastrophic forgetting \citep{neurips2024globalalignment,bai2024globalalignment,lin2025sparsememoryfinetuning}.
In parallel, the dominant industry response to safety has been to attach \emph{external} governance components around otherwise fixed learners. This includes constitutional or principle-based post-training \citep{anthropic2022constitutional,bai2022constitutional}, guardrail and monitoring layers \citep{dong2024guardrails,dong2024guardrailsPMLR}, and runtime formal verification or shielding \citep{corsi2024verificationguided}.
These approaches are valuable, but they typically treat norms, verifiers, evaluators, memory stability, and admissible transformations as fragmented modules or \emph{exogenous wrappers} rather than as first-class objects in the learning meta-model.

\paragraph{Why safety and security force a structural theory.}
This fragmentation is increasingly problematic as safety becomes a complex \emph{evaluation-and-assurance} problem. System cards and preparedness frameworks now explicitly separate multi-domain threat regimes \citep{openai2023gpt4systemcard,openai2024gpt4o,phuong2024dangerouscapabilities,preparednessframeworkv2,nistai6001genaiprofile}. 
Furthermore, tool-integrated agents expose interface-level vulnerabilities—such as indirect prompt injection and plan backdoors—that are now systematically benchmarked \citep{greshake2023ipi,zhan2024injecagent,llmpirate2024,zhang2024asb,yi2023bipia,agentsafetybench2024}. Crucially, \emph{memory itself} has become an attack surface (e.g., memory and RAG poisoning), making invariant-preserving updates central to controllability \citep{chen2024agentpoison}.

\paragraph{Why agentic evaluation makes internalization unavoidable.}
Consequently, the field's evaluation paradigm is radically shifting from static end-task scores to the correctness of stateful multi-step tool-use, memory-dependent decisions, and long-horizon workflows (e.g., BFCL, TRAJECT-Bench, AgentBench) \citep{bfcl2025,trajectbench2025,agentbench2024}, prompting unified taxonomies that separate policy models from evaluators and dynamic components \citep{li2025reviewagents}.

Together, these developments motivate the view that norms, evaluators, and admissible transformations should not be treated solely as exogenous evaluation infrastructure. We therefore model them as first-class objects within the learning meta-structure.

\paragraph{Thesis: structural generalization and control-by-structure.}
We advance the thesis that general artificial intelligence cannot be reduced to broader task coverage under a fixed evaluative regime.
Instead it requires \emph{structural generalization}: the capacity to represent, compare, revise, and preserve parts of the learning interface itself under task transformations.
This implies \emph{control-by-structure}: evolution must be restricted to certified admissible transformations, explicitly separating
(i) what may change (typed transformations and admissible updates),
(ii) what must remain invariant (identity/personality invariants and protected evaluators),
and (iii) what must be verified at each step (normative constraints and monitors).
The goal is not a new architecture, but a \emph{formal object} that makes structural change auditable and analyzable.

\paragraph{Epistemological necessity: from external metrics to structural objects.}
This internalization of governance is not merely an engineering convenience; it is required for stating evaluation and admissibility conditions within the learning object. Following the classical distinction between facts and values \cite{hume1739treatise}, evaluation cannot be purely derived from descriptive prediction. The loss or utility functional encodes a \emph{normative semantics} that fixes what counts as an error or an acceptable trade-off. Similarly, in a Kantian spirit, the representational interface $r$ is constitutive: it determines what distinctions are learnable and comparable, thereby inducing the system's active ontology \cite{kant1781-1998}. If the evaluation regime or the representation changes, the very definition of the task changes. This dictates that evaluation and representation must be treated as first-class internal components of the learning object, a foundational argument we fully expand upon in Section~\ref{sec:philo}.

\paragraph{The structural meta-model and the coupled view $(\theta,T_\theta)$.}
To capture representation, evaluation, memory, and admissible evolution within a single formal object, we introduce a structured learning meta-model
\[
\theta=(r,\mathcal H,\Pi,\mathcal L,\mathcal E,\mathcal M),
\]

Crucially, SMGI is not a property of $\theta$ alone but of the \emph{coupled} object $(\theta,T_\theta)$, where $T_\theta$ denotes the induced behavioral semantics (the dynamical evolution under learning, memory, and interaction).
This separation between structural ontology and behavioral semantics is what makes stability, invariants, and evaluative integrity expressible as properties of certified evolution rather than informal desiderata.
Formal typing and notation are given in Sec.~\ref{sec:formalsetup}.

\paragraph{Structural definition of AGI: admissibility obligations.}
In this paper, AGI is not defined by benchmark performance, human equivalence, or parameter scale.
We define AGI structurally as the class of admissible pairs $(\theta,T_\theta)$ satisfying four obligations:
\begin{enumerate}[label=(\roman*)]
\item \textbf{Closure:} the system remains closed under a typed class of admissible task/interface transformations;
\item \textbf{Stability:} certified sequential adaptation preserves dynamical stability under regime switching and memory interaction;
\item \textbf{Capacity:} statistical capacity remains bounded along admissible evolutions (including implicit/trajectory-dependent regularization effects);
\item \textbf{Evaluative invariance:} evaluators and norms remain invariant across regime shifts unless updated via protected, certified meta-transformations.
\end{enumerate}
\paragraph{What is (and is not) claimed here.}
We use the term \emph{AGI} in a strictly formal sense: not as a claim about any existing system, but as a structural admissibility concept for evolving learning systems under typed interface transformations. The present article develops this structural theory at the level of the meta-model, its induced semantics, and its admissibility obligations.

%\begin{definition}[Minimal checkable form of the obligations]\label{def:obligations-checkable}
%The obligations (i)--(iv) will be verified via \emph{witnesses} and \emph{tests}:
%(i) closure is witnessed by a declared typed transformation class $\mathcal T$ on $\theta$;
%(ii) stability is witnessed by an explicit certificate on the induced dynamics (e.g., a Lyapunov-style functional on $S_\theta$);
%(iii) capacity is witnessed by an explicit complexity functional controlling the effective hypothesis/update interface;
%(iv) evaluative invariance is witnessed by an invariant core $\Phi$ together with an admissibility test for any update of the evaluator across regimes.
%\end{definition}
\paragraph{Minimal checkability (pointer).}
In the formal development (Sec.~\ref{sec:formal-meta}), we make these obligations \emph{checkable} by requiring explicit witnesses:
a declared typed transformation class $\mathcal T$, a stability certificate (Lyapunov-style or equivalent) on the induced dynamics,
a capacity/complexity functional controlling the effective hypothesis/update interface, and a protected invariant evaluative core $\Phi$
together with an admissibility test for evaluator updates.

\paragraph{Why scale is not sufficient (as an instance).}
Scaling laws \cite{kaplan2020scaling,hoffmann2022chinchilla} and modern analyses of overparameterized learning (e.g., NTK limits and implicit regularization) \cite{jacot2018neural,arora2019implicit} help explain generalization under fixed regimes, and phenomena such as grokking emphasize the role of learning dynamics \cite{power2022grokking,wang2024grokkedtransformers}.
SMGI complements these insights by lifting capacity control from parameter-space regularization to \emph{structural capacity control} over admissible transformations of $(r,\mathcal H,\Pi,\mathcal L,\mathcal E,\mathcal M)$, which becomes essential under multi-regime evaluation and interface transformations.

\paragraph{Contributions.}
Our contributions are:
\begin{enumerate}
\item \textbf{A minimal structural meta-model for AGI.} We formalize the SMGI object $\theta$ and its induced semantics $T_\theta$, clarifying the ontology/semantics split required for structural reasoning.
\item \textbf{A structural definition of AGI via obligations.} We define closure, stability, capacity, and evaluative invariance as admissibility obligations for $(\theta,T_\theta)$, making verification meaningful only relative to explicit typed transformations and protected evaluators.
\item \textbf{Unified capacity analysis under structural evolution.} We connect SRM/MDL, PAC-Bayes, and stability-based viewpoints to structural capacity control under admissible updates, including modern implicit-regularization phenomena as instances \cite{grunwald2007mdl,mcallester1999,catoni2007pacbayes,arora2019implicit}.
\item \textbf{Strict inclusion and positioning.} We show how classical SLT, RL, universal induction, and frontier LLM pipelines arise as structurally restricted instances of the proposed meta-framework, clarifying what changes when the interface and evaluator are no longer fixed.
\end{enumerate}

\paragraph{Organization.}
Sec.~\ref{sec:philo} develops the philosophical and scientific foundations.
Sec.~\ref{sec:formalsetup} introduces formal setup, typing, and notation.
Secs.~\ref{sec:formal-meta}--\ref{sec:capacity-bridge} develop the formal meta-structure, structural guarantees, and capacity analysis.
Sec.~\ref{sec:positioning} positions SMGI within existing theoretical frameworks.
Sec.~\ref{sec:comparative-frameworks} compares SMGI to frontier systems (2023--2026), followed by discussion and conclusions.

%%%%%%%%%%%%%%%%%%%%%%%%%%%%%%%%%%%%%%%%%%%%%%%%%%%%%%%%%%%%%%%%%%%%%%%%
\section{Philosophical and Scientific Foundations}
\label{sec:philo}
%%%%%%%%%%%%%%%%%%%%%%%%%%%%%%%%%%%%%%%%%%%%%%%%%%%%%%%%%%%%%%%%%%%%%%%%

The formal framework developed in this article, which we call the \emph{Structural Model of General Intelligence} (SMGI), is not proposed as a purely technical aggregation. It is motivated by a deeper claim from the history and philosophy of science: \emph{what a system can learn is inseparable from how it is evaluated, and evaluation is itself a structural commitment that must be formalized}. To justify the internalization of the loss $\ell$ and the representation $r$ into the meta-model $\theta$, we must address the fundamental epistemological limits of classical learning.

Throughout the article, the term SMGI will denote the class of admissible coupled objects \((\theta,T_\theta)\) satisfying structural closure under admissible task transformations, dynamical stability, bounded capacity, and evaluative invariance. When no ambiguity arises, we refer to the corresponding meta-model \(\theta\) alone by abuse of notation.

\subsection{Facts, Values, and the Structure of Evaluation}
\label{sec:facts-values}

\paragraph{The irreducibility of normative evaluation.}
A foundational difficulty in defining general artificial intelligence is that accurate prediction and justified evaluation are logically distinct. Since Hume's articulation of the \emph{is/ought} gap, it has been canonically recognized that no set of purely descriptive propositions entails a normative conclusion without an additional normative premise \citep{hume1739treatise}. Consequently, a system may predict environmental dynamics correctly while remaining entirely unjustified with respect to the criteria used to judge its outputs.

Poincar\'e restated the same logical constraint in \emph{La morale et la science}: if the premises of a reasoning are purely indicative, the conclusion remains indicative; no imperative conclusion can be obtained from descriptive premises alone \citep{poincare1920dernierespensees}. Scientific appraisal therefore presupposes epistemic values such as simplicity, unity, and fecundity that are not derivable from data alone. A complementary meta-ethical clarification is provided by Moore’s open-question argument, which makes precise why evaluative predicates such as ``good'', ``correct'', or ``safe'' are not reducible by definition to purely descriptive predicates \citep{mooreprincipiaethica}. Together, these arguments support a structural separation between descriptive regularities, evaluative mechanisms, and normative admissibility conditions.

In SMGI, this separation is made explicit at the level of the typed meta-model
\[
\theta=(r,\mathcal H,\Pi,\mathcal L,\mathcal E,\mathcal M).
\]
The tuple $(r,\mathcal H,\Pi)$ governs the descriptive side of learning: representation, admissible hypothesis classes, and inductive bias determine which regularities can be expressed and generalized. By contrast, evaluation is implemented through a family of evaluative functionals
\[
\mathcal L=\{\ell_k\}_{k=1}^{K},
\]
and the active regime-local evaluator takes the form
\[
\ell_t(\cdot)=\sum_{k=1}^{K}\lambda_{t,k}\,\ell_k(\cdot),
\]
with weights controlled by context and memory. Normativity is not identified with the evaluator itself: it appears as a system of admissibility constraints, denoted here by $\Phi$, that governs which evaluators, or evaluator updates, are allowed to count as structurally coherent. In this sense, SMGI distinguishes three levels that are often conflated: value-laden evaluative dimensions encoded by the basis family $\{\ell_k\}_{k=1}^K$, normative constraints encoded by $\Phi$, and evaluation as the operational mechanism $\ell_t$ that ranks hypotheses, actions, or trajectories in a given regime.

This distinction matters mathematically. Even when an evaluator is estimated from data, it defines a selection functional over hypotheses,
\[
h \;\mapsto\; \mathbb{E}_{z\sim \mathcal D}\big[\ell_t(h,z)\big],
\]
and therefore fixes what counts as an error, a violation, or an acceptable trade-off. The resulting ordering is not derivable from descriptive prediction alone. Following Canguilhem, one may say that normativity is not an external ornament added after the fact, but part of the very organization through which a system distinguishes viable from non-viable trajectories \citep{canguilhem1966normal}. SMGI operationalizes this point by internalizing evaluative structure as part of the learning ontology rather than leaving it implicit in the training procedure.

\paragraph{Entanglement and the necessity of value learning.}
At the same time, the philosophy of science emphasizes that facts and values are deeply entangled in scientific practice; this does not eliminate normativity but reinforces the need to surface evaluative commitments explicitly rather than leave them implicit \citep{putnam2002}. More generally, normative systems are typically justified by coherence between principles and judgments under continuous revision—a ``reflective equilibrium''—which motivates the explicit modeling of evaluator update rules rather than treating them as permanently fixed \citep{rawls1971}. In SMGI, such updates are represented as typed meta-transformations acting on $\ell_t$ (or regime weights) and are constrained by certification and invariance requirements (see Sec.~\ref{sec:reflective-equilibrium}).

At a formal level, these normative commitments are modeled via utility and loss functionals that induce preference orderings over outcomes and policies \citep{vonneumann1944,savage1954}. In modern agentic settings, however, the relevant objective is rarely fully known a priori; it is often only partially specified and must be inferred or refined from interaction. This reality drives contemporary value learning, cooperative inverse reinforcement learning (CIRL), and preference-based alignment methods \citep{hadfieldmenell2016cirl,christiano2017preferences,russell2019human}. Canguilhem’s analysis of normativity sharpens the same point from another direction: norms are not merely external labels attached after the fact, but organizing conditions internal to a mode of functioning \citep{canguilhem1966normal}. In SMGI, this is precisely why evaluative admissibility is represented by explicit constraint systems such as $\Phi$ and $\mathcal K_\Phi$, rather than by an undifferentiated scalar objective.

\subsection{Reflective Equilibrium as Certified Updates of the Evaluator}
\label{sec:reflective-equilibrium}

To separate descriptive adaptation from normative consistency, SMGI does not treat the evaluator as an unconstrained objective updated by naive optimization. Following the Rawlsian idea of \emph{reflective equilibrium} \citep{rawls1971}, we model evaluator change as a \emph{certified} reconciliation between empirical feedback and identity-preserving invariants.

\paragraph{A. Invariants and admissibility.}
Let $\Phi$ denote a fixed set of identity-preserving invariants specified at $t=0$ (e.g., hard safety constraints, logical consistency requirements, or core normative commitments). In SMGI, evaluator change is admissible only through updates that preserve $\Phi$ on a certified test class (or audited environments), rather than universally over $\mathcal H$.

\paragraph{B. Certified evaluator update operator.}
Let $\ell_t$ denote the current evaluator (possibly multi-regime, e.g., $\ell_t=\sum_{k=1}^K \lambda_{t,k}\,\ell_k$) and let $z_{1:t}$ and $m_t$ denote the data stream and persistent memory. A candidate update $\tilde{\ell}_{t+1}=\Gamma(\ell_t,z_{1:t},m_t)$ (e.g., Bayesian update, preference refinement, or meta-learning) is \emph{accepted} only if it satisfies the invariants on an audited set $\mathcal{E}_{\mathrm{audit}}$ (or a certified hypothesis subset $\mathcal H_{\mathrm{cert}}$). We assume $\ell$ ranges over a parametrized evaluator class $\mathcal L$ (or, in the multi-regime case, over regime weights $\lambda_t$ in a fixed basis $\{\ell_k\}_{k=1}^K$). We write this as a certified update operator:
\begin{equation}
\begin{split}
\ell_{t+1}
\;&=\;
\mathrm{CertUpdate}_{\Phi}\!\left(\ell_t,z_{1:t},m_t\right)
\;\triangleq\;
\arg\min_{\ell\in\mathcal K_{\Phi}}
D(\ell,\tilde{\ell}_{t+1}), \\[1.5ex]
\mathcal K_{\Phi} \;&=\; \Big\{ \ell : \Phi\ \text{holds on}\ \mathcal{E}_{\mathrm{audit}} \\
&\qquad \text{as certified by monitors/verifiers over audited trajectories} \Big\},
\end{split}
\label{eq:certupdate}
\end{equation}
where $D$ is a chosen divergence on evaluators (e.g., KL between induced preference models, or a norm in a parametrized evaluator class), and $\mathcal K_{\Phi}$ is the admissible set enforced by certificates/monitors.

\paragraph{C. Implications for SMGI obligations.}
This construction makes evaluator change an explicit part of certified structural evolution. Formally, evaluator updates are elements of the typed transformation class $\mathcal T$ acting on the $\ell$-component of $\theta$:
(i) it prevents \emph{normative drift} by restricting evaluator updates to $\mathcal K_{\Phi}$ (evaluative invariance),
(ii) it makes closure and stability statements meaningful because admissible transformations include certified evaluator updates,
and (iii) it allows capacity and drift guarantees to be stated on the coupled dynamics $(\theta,T_\theta)$ under certified evolution rather than under an exogenous pipeline.

\paragraph{Structural Risk and Capacity.}
Vapnik's structural risk minimization formalizes generalization by coupling empirical fit and capacity control within a hierarchical organization of hypothesis classes \citep{vapnik1998statistical}. This provides a principled foundation for treating capacity constraints as part of the system's admissible structure, rather than as an external hyperparameter choice. In this sense, SRM already contains a proto-structural insight: generalization is not merely a property of data fitting, but of the hierarchical organization of admissible hypothesis spaces.

\paragraph{Simplicity Priors and Universal Prediction.}
A complementary foundational tradition connects evaluation to structural regularization via simplicity principles. Solomonoff induction and AIXI encode a prior over programs weighted by description length \citep{solomonoff1964,hutter2005}, while Kolmogorov complexity and MDL formalize simplicity as compressibility and provide operational regularizers \citep{li2008kolmogorov,grunwald2007mdl,rissanen1978mdl}. These results motivate treating the prior $\Pi$ not as a heuristic add-on but as a structural component of the meta-model.

\paragraph{Synthesis: The justification of the SMGI Meta-Model.}
Under regime shifts, tool augmentations, and self-modification, an intelligent system faces a meta-problem: it must maintain coherence between what is represented, how outcomes are evaluated, which priors structure induction, and which transformations remain admissible without destroying structural identity. Because objectives and environments are dynamic, this coherence must be understood as a time-indexed process whose evaluative component is expressed by $\ell_t$ and whose admissibility is certified on the coupled dynamics $(\theta,T_\theta)$. This is the exact justification for SMGI's central move: representation, hypothesis space, prior, evaluator family, environment family, and memory operator are treated as \emph{first-class structural components} of the typed meta-model
\[
\theta=(r,\mathcal H,\Pi,\mathcal L,\mathcal E,\mathcal M).
\]

This move is philosophically motivated and mathematically necessary. Kantian considerations explain why representation cannot be external to a theory of intelligence; Humean and Poincar\'ean considerations explain why normativity cannot be recovered from descriptive success alone; and Quine's underdetermination thesis explains why empirical adequacy does not uniquely fix a single inductive organization \citep{kant1781-1998,hume1739treatise,poincare1920dernierespensees,quine1960word}. The role of $\Pi$ is therefore not heuristic decoration: it formalizes the structural commitments through which a system privileges some admissible generalizations over others.

Accordingly, SMGI should be read not as a behavioral slogan but as a structural ontology for admissible general intelligence. The coupled object $(\theta,T_\theta)$ separates the learning system's internal organization from its induced behavioral semantics, making it possible to state closure, stability, bounded-capacity, and evaluative-invariance obligations at the correct level. In this sense, values, inductive bias, representation, and persistence are no longer external vulnerabilities of a training pipeline, but explicit structural components whose joint evolution can be certified.

\subsection{Unification as a Scientific Virtue and a Design Constraint}
\label{sec:unification}

A second foundational motivation is unification. Newtonian science is a paradigm case: the unification of heterogeneous phenomena under a single mathematical structure is not merely aesthetic but constitutive of explanatory power. Modern scholarship clarifies that unification in Newton is inseparable from methodological commitments about mathematization, idealization, and the normative role of principles \citep{ducheyne2012newton}. In our setting, the analogue is that a general theory of intelligence should not only fit many tasks, but specify the minimal \emph{structural objects} and \emph{operators} under which heterogeneous learning regimes can be represented within a single formal system. Formally, unification can be understood as the identification of invariant structures under transformation, specifying state spaces and operators whose stability is preserved under admissible task transformations.

\subsection{Paradigms, Evaluation Criteria, and Theory-Relative Invariants}
\label{sec:paradigms}

Kuhn's analysis of scientific change stresses that what counts as evidence, adequacy, and progress depends on theory-internal standards that evolve with paradigms \citep{kuhn1962Structure}. This reinforces the need to treat evaluation as a structured component of a system, not as a fixed external metric. As established in Section 2.1 via Putnam’s entanglement thesis \citep{putnam2002}, this implies that for learning systems, cross-domain intelligence requires not only updating hypotheses within a fixed evaluative frame, but maintaining \emph{invariants} under transformations of tasks and evaluation regimes.

This is precisely the motivating gap for the present work: classical theories formalize capacity control, sequential decision-making, or universal prediction, but do not simultaneously formalize (i) a closure notion under task transformation, (ii) stability of a persistent system under sequential adaptation, and (iii) invariance constraints defining coherent identity across regimes. The notion of identity here is structural rather than psychological: it refers to the preservation of boundedness, coherence, and well-defined update rules across transformations of tasks, data distributions, or evaluative criteria.

\subsection{Representation, Direction, and Plurality of Regimes}
\label{sec:repr-will}

In the Kantian tradition, cognition presupposes organizing principles that condition the possibility of experience itself \citep{kant1781-1998}. In SMGI, this point is formalized by treating the representational map as a constitutive component of the meta-model rather than as a neutral preprocessing device. Representation is not a passive encoding of a pre-given world: it determines which distinctions are available to the system and therefore which hypothesis classes are effectively accessible. Put differently, the representational component $r$ helps determine the induced learnable structure, and changes in $r$ alter the ontology on which learning and evaluation operate.

Cassirer sharpens this point by emphasizing that cognition proceeds through symbolic forms rather than through immediate access to raw reality \citep{cassirer1923symbolic}. Sellars, in turn, rejects the myth of a purely given layer of experience that would stand prior to conceptual articulation \citep{sellars1956empiricism}. Their common lesson is directly relevant here: neither representation nor evaluation should be treated as transparent interfaces to an already structured world. In SMGI, they are explicit structural commitments. The representational component is therefore not merely an encoder but part of the system's active ontology. Formally, the typed meta-model
\[
\theta=(r,\mathcal H,\Pi,\mathcal L,\mathcal E,\mathcal M)
\]
locates $r$ as a structural primitive, so that $r$ helps determine the effective hypothesis class $\mathcal H(\theta)$ and, through it, the space of learnable distinctions available to the system.

Schopenhauer's distinction between world as representation and world as will \citep{schopenhauer1818} highlights a second separation between what is represented and what is directionally pursued. We capture this distinction by introducing a local structural state
\[
c_t \;=\; (r_t,w_t,m_t),\qquad r_t:\mathcal Z\to\mathcal H,\qquad w_t\in\mathcal W,
\]
where $m_t$ is persistent memory. Here $r_t$ specifies the representational interface, while $w_t$ parameterizes directional organization: regime activation weights, preference orientation, or constraint multipliers determining which evaluative dimensions are currently emphasized. This decomposition is not auxiliary notation but a local dynamic reading of the SMGI ontology: $r_t$ instantiates the representational component, $w_t$ parameterizes directional or evaluative orientation, and $m_t$ carries persistence through memory. Peirce's logic of inquiry provides a further bridge here, since abduction, revision, and interpretive mediation all presuppose that representation and hypothesis formation are structurally organized rather than mechanically extracted from data \citep{peirce1931collected}.

Moving from representation to direction, Nietzsche and Freud provide a useful structural analogy for agency organized by competing evaluative forces rather than by a single homogeneous objective \citep{nietzsche1887,freud1923_lemoietleca}. In SMGI, this motivates extending a single evaluator to a family of regime-dependent evaluators:
\begin{equation}
\ell(\cdot) \;\rightsquigarrow\; \{\ell_k(\cdot)\}_{k=1}^K, \qquad \ell_t(\cdot) \;=\; \sum_{k=1}^K \lambda_{t,k} \, \ell_k(\cdot),
\label{eq:multi-regime}
\end{equation}
with context-dependent weights $\lambda_{t,k}$ controlled by $w_t$ and memory. The point is not that psychoanalytic or genealogical theories literally define the update law, but that they correctly foreground a structural fact: agency may remain coherent while being governed by plural, sometimes competing evaluative regimes. The philosophical plurality of drives, values, or competing orientations is thus translated into a typed evaluator family $\mathcal L=\{\ell_k\}_{k=1}^K$ together with context- and memory-dependent activation weights.

Spinoza's doctrine of \emph{conatus} adds the complementary requirement of persistence through change: an intelligent system should not merely vary its active evaluators across regimes, but preserve identity-relevant invariants through those variations \citep{spinoza1677ethics}. In SMGI, this persistence is expressed through the admissibility of coupled dynamics $(\theta,T_\theta)$ under stability and invariance constraints. Simondon's analysis of individuation clarifies that such persistence should be understood not as immobility but as coherent transformation under structured change \citep{simondon1958individuation}. Together, these perspectives support the central claim of SMGI: regime shifts should be modeled as admissible structural transformations preserving coherence, rather than as unconstrained jumps between unrelated optimization problems.

\paragraph{Salience as Structural Weighting.}
To separate regime weights from policy notation, we denote input salience by an operator $\sigma_t:\mathcal Z\to\mathbb R_+$ and write $\ell_k(h,z)\mapsto \sigma_t(z)\,\ell_k(h,z)$ when needed. This formalizes the idea (classically associated with motivation and ``passions'' in the philosophy of agency) that evaluation is not only plural but also selectively amplified or suppressed depending on context. The triple $(r_t, w_t, m_t)$ therefore anticipates a more general structural decomposition in which representation, evaluation, and persistence are treated as interacting subsystems.

\subsection{Engineering Meta-Models and the Necessity of a Meta-Level}
\label{sec:eng-metamodels}

The motivation for an explicit meta-structure is not unique to learning theory. In software engineering, the \emph{Meta-Object Facility} (MOF) provides a meta-metamodel specifying how modeling languages themselves are defined \citep{omgmof2019}. UML, in turn, is a structured meta-model describing classes, relations, and behavioral dynamics \citep{omguml17}. Correctness and consistency are enforced at the meta-level: constraints are defined over families of models, not only over individual instances.

Object-oriented languages further illustrate this structural idea: inheritance and subsumption are not mere implementation tools; they impose formal constraints ensuring that transformations between objects preserve type consistency and structural invariants \citep{meyer1988oosc,booch1994ooad}. When a domain involves transformations between heterogeneous structures, stability and coherence require explicit meta-level constraints. SMGI plays an analogous formal role: it defines the space of admissible learning systems and constrains their evolution under task transformation operators.

\paragraph{Optional Extensions: Meta-Normativity and Ontology.}
We note two optional second-order extensions that fit the same structural discipline but are not required for the main line of the theory.
Two further structural notions can be expressed within the same spirit, while remaining optional in the main line of the theory.
\emph{Meta-normativity} treats evaluative regimes themselves as objects of second-order appraisal, via a functional $\mathcal{N}:\mathcal{K}\to\mathbb{R}$ measuring coherence, legitimacy, or stability of regime families.
\emph{Structural ontology} can be summarized by observing that a system induces an implicit ontology through its representational range: $\Omega(\theta) = \mathrm{Range}(r)$.
Different $\theta$ therefore induce different ontological commitments without fixing ontology a priori.

\paragraph{Dimensional Incompleteness and Motivation.}
The preceding considerations converge on a single structural claim: learning is not merely adaptation within a fixed evaluative space; it is the structured transformation of the evaluative space itself.
This motivates the construction of a formal meta-structure where these constraints are explicit at the level of mathematical definition.

%%%%%%%%%%%%%%%%%%%%%%%%%%%%%%%%%%%%%%%%%%%%%%%%%%%%%%%%%%%%%%%%%%%%%%%%
\section{Formal Setup, Typing, and Notation}
\label{sec:formalsetup}
%%%%%%%%%%%%%%%%%%%%%%%%%%%%%%%%%%%%%%%%%%%%%%%%%%%%%%%%%%%%%%%%%%%%%%%%

To eliminate ambiguity and keep the formal development closed under composition,
we enforce a \emph{typed} notation that distinguishes (i) global carriers,
(ii) instantiations induced by a meta-structure $\theta$, and (iii) time-indexed
states produced by the induced learning dynamics.

\paragraph{Global objects (problem-independent).}
We write $\mathcal X,\mathcal Y$ for observation and target spaces, $\mathfrak E$
for a \emph{global} carrier of admissible environments (or task-generating processes),
and $\mathfrak H$ for a \emph{global} hypothesis carrier. These symbols denote \emph{carriers} rather
than a particular learner.

\paragraph{Objects instantiated by $\theta$.}
Whenever an object depends on the choice of meta-structure, we use the unique
convention $(\cdot)(\theta)$:
\[
\mathcal H(\theta)\subseteq \mathfrak H,\qquad
\mathcal P(\theta)\ \text{(policy/parameter space)},\qquad
\mathcal M(\theta)\ \text{(memory state space)},\qquad
\mathcal E(\theta)\subseteq \mathfrak E\ \text{(environment family)}.
\]
Thus $\mathcal H(\theta)$ denotes the \emph{effective} hypothesis class admitted
by $\theta$ (e.g., via representational constraints $r$ and regularization $\Pi$),
and similarly for $\mathcal P(\theta)$, $\mathcal M(\theta)$, and $\mathcal E(\theta)$.

\paragraph{Priors vs.\ policies.}
We reserve $\Pi$ for a \emph{prior} (or structural regularizer) over hypotheses or
programs, used in SRM/PAC-Bayes/MDL-style control. By contrast, $\pi$ denotes a
\emph{policy} (or decision rule) belonging to $\mathcal P(\theta)$.
In particular:
\[
\Pi \neq \pi,\qquad \pi\in\mathcal P(\theta),\qquad \Pi\ \text{acts on}\ \mathcal H(\theta)\ \text{or program space}.
\]

\paragraph{Evaluators: active loss vs.\ regime family.}
We reserve $\mathcal L=\{\ell_k\}_{k=1}^K$ for the \emph{family} of regime evaluators (multi-regime structure) that is part of the structural meta-model.
The symbol $\ell_t$ denotes the \emph{active} evaluator at time $t$ (possibly a selected element of $\mathcal L$, or a certified mixture such as $\ell_t=\sum_{k=1}^K \lambda_{t,k}\,\ell_k$).
When we write $\ell$ without a time index, it denotes a generic element of $\mathcal L$ (or the induced evaluator semantics under a fixed regime) and never the whole family.

\paragraph{Time-indexed states.}
When the induced dynamics $T_\theta$ is considered, we index \emph{states} in time:
\[
r_t \in \mathcal R(\theta),\qquad \pi_t\in\mathcal P(\theta),\qquad m_t \in \mathcal M(\theta),
\]
and we write $\theta$ only for the (comparatively) static meta-structure.
This separation is used throughout to avoid conflating \emph{ontological structure}
($\theta$) with \emph{behavioral semantics} ($T_\theta$).

\subsection{The SMGI Meta-Model}
\label{sec:smgi-metamodel}

\paragraph{The Meta-Model.}
To keep notation lightweight in the introduction and foundations, we write
$\theta=(r,\mathcal H,\Pi,\mathcal L,\mathcal E,\mathcal M)$ for the \emph{structural carriers}.
In the formal development, we distinguish carriers from their effective instances:
a well-typed $\theta$ induces (or selects) an effective hypothesis class $\mathcal H(\theta)\subseteq\mathfrak H$
and an effective environment family $\mathcal E(\theta)\subseteq\mathfrak E$ (and similarly for memory/state spaces).

To formalize this idea, we introduce a structured learning meta-model
\[
\theta = (r, \mathcal H(\theta), \Pi,\mathcal L, \mathcal E(\theta), \mathcal M(\theta)).
\]
where:
\begin{itemize}
\item $r$ denotes representation operators,
\item $\mathcal H(\theta)$ the hypothesis space (effective class admitted by $\theta$),

\paragraph{Evaluator family vs. active evaluator (notation contract).}
Throughout the paper, the evaluator component of the meta-model is a \emph{family} \(\mathcal L=\{\ell_k\}_{k=1}^K\) of regime evaluators.
The \emph{active} evaluator at time \(t\) is denoted \(\ell_t\in\mathcal L\) (or, when using convex mixing, \(\ell_t=\sum_{k=1}^K\lambda_{t,k}\ell_k\) with \(\lambda_{t,\cdot}\in\Delta_K\)).
The single-regime special case is represented as \(\mathcal L=\{\ell\}\), but the signature of \(\theta\) remains \(\theta=(r,\mathcal H,\Pi,\mathcal L,\mathcal E,\mathcal M)\) throughout.
\item $\Pi$ denotes prior or structural regularization,
\item $\ell$ evaluative structures (possibly multi-regime $\mathcal L$),
\item $\mathcal E(\theta)$ environment classes,
\item $\mathcal M(\theta)$ persistent structured memory.
\end{itemize}

Unlike classical learning theory, which optimizes
$h \in \mathcal H$ under fixed $(r, \Pi, \ell)$,
we study systems that operate over $\theta$ itself.
The object $\theta$ defines a coupled dynamical system whose
admissible evolutions are constrained by stability,
capacity, and invariance conditions.

Classical paradigms — empirical risk minimization,
PAC-Bayes learning, reinforcement learning,
meta-learning, and transformer-based architectures —
arise as restricted projections of this meta-structure.

\paragraph{Typing of the structural components.}
Each component of the meta-model
\(
\theta = (r, \mathcal H(\theta), \Pi,\mathcal L, \mathcal E(\theta), \mathcal M(\theta))
\)
lives in a distinct mathematical space:

\begin{itemize}
\item $r$ is a representation operator
\[
r : \mathcal Z \to \mathcal X,
\]
mapping observations/interactions $\mathcal Z$ to an internal representational space $\mathcal X$;

\item $\mathcal H(\theta)$ is a class of hypotheses (predictors, policies, or programs) defined over the representational space induced by $r$;

\item $\Pi$ is a prior or inductive bias over $\mathcal H(\theta)$ (e.g., a probability measure, description-length prior, or structural regularizer);

\item $\ell$ is an evaluative functional
\[
\ell : \mathcal H(\theta) \times \mathcal E(\theta) \to \mathbb R,
\]
possibly multi-regime, encoding objectives and constraints;

\item $\mathcal E(\theta)$ is a class of admissible environments or task-generating processes (e.g., distributions, MDPs, POMDPs);

\item $\mathcal M(\theta)$ is a structured memory space, supporting persistent internal state.
\end{itemize}

These components are structurally distinct and not interchangeable:
they belong to different mathematical categories (functions, sets,
measures, functionals, process classes, state spaces).

\paragraph{Structural Semantics and Induced Dynamics.}
In this article, $\theta$ is a \emph{purely structural} meta-model: it specifies
\emph{what the state variables are}, \emph{how they are typed}, and \emph{which
structural operators are admissible}, but it does not commit to a unique temporal
evolution law.

\paragraph{Induced state space.}
The typing of $\theta=(r,\mathcal H(\theta),\Pi,\mathcal L,\mathcal E(\theta),\mathcal M(\theta))$ determines
canonical carrier spaces for its components:
\[
r \in \mathcal R(\theta),\qquad
h \in \mathcal H(\theta),\qquad
\pi \in \mathcal P(\theta),\qquad
m \in \mathcal M(\theta),
\]
and therefore induces the structured state space
\[
S_\theta \;:=\; \mathcal R(\theta) \times \mathcal H(\theta) \times \mathcal P(\theta) \times \mathcal M(\theta).
\]
A configuration at time $t$ is a point
\[
s_t=(r_t,h_t,\pi_t,m_t)\in S_\theta,
\]
while $\mathcal E(\theta)$ specifies the admissible class of environments/task-families
in which trajectories are generated.

\paragraph{Behavioral semantics as a set of compatible dynamics.}
Let $\Delta(S)$ denote the set of probability measures over $S$
(the deterministic case corresponding to Dirac kernels).
We define the space of admissible stochastic learning dynamics as
\[
\mathfrak D \;:=\; \bigcup_{S}\Big\{\,T : S\times \mathcal Z \to \Delta(S)\;\Big\}.
\]
Because a structural $\theta$ may admit multiple compatible implementations,
its behavioral semantics is \emph{set-valued}:
\[
\mathsf{Sem}:\Theta \rightrightarrows \mathfrak D,
\qquad
\mathsf{Sem}(\theta)\subseteq \Big\{T: S_\theta\times \mathcal Z \to \Delta(S_\theta)\Big\}.
\]
An \emph{induced} (chosen) dynamics compatible with $\theta$ is any selection
\[
T_\theta \in \mathsf{Sem}(\theta).
\]

\paragraph{Scope.}
We deliberately keep the present paper focused on the structural axis; richer multi-timescale or multi-operator semantics can be layered on top without altering the definitions below (see Discussion).

\paragraph{SMGI lives on realizations.}
Because $\mathsf{Sem}(\theta)$ is set-valued, all verifiable claims are made at the level of a realization $(\theta,T)$ with
$T\in\mathsf{Sem}(\theta)$; we occasionally call $\theta$ \emph{SMGI-admissible} when there exists such a $T$ satisfying the obligations
defined later (Sec.~\ref{sec:formal-meta}).
\paragraph{Internal Stratification of Memory and Evaluation (Foundational Scope).}
The present meta-model treats $\mathcal M(\theta)$ (memory) and evaluator family $\mathcal L$
(evaluative structure) as structured components of $\theta$,
without fully stratifying their internal organization.

In companion work, memory is decomposed into interacting subsystems
(working, episodic, semantic, procedural, world-model,
social/theory-of-mind, normative, and meta-memory layers),
each supporting distinct temporal and functional roles.
Similarly, the evaluative structure may itself admit
multi-regime and hierarchical organization.

A full formal treatment of cross-memory interaction,
multi-timescale dynamics, and regime coordination
constitutes an orthogonal axis of analysis.
Such an axis would refine the behavioral semantics
$\mathsf{Sem}(\theta)$ by introducing families of interacting
transition operators indexed by memory strata or temporal scales.

The present work isolates the minimal structural conditions
under which a unified learning system remains closed,
stable, statistically controlled, and evaluatively coherent
under admissible task transformations.
The stratified refinement of memory and evaluation
extends --- rather than alters --- these structural foundations.

Normative/value memory provides operational persistence of evaluative
commitments defined at the level of $\ell$; it does not replace the
formal evaluative structure but supports its temporal stability.

\paragraph{On the distinct roles of representation $r$ and environment class $\mathcal E(\theta)$.}
A frequent ambiguity in unified learning formalisms is whether
the environment/task component can be absorbed into data,
or whether representation can be absorbed into the hypothesis class.
In the present meta-model, $r$ and $\mathcal E(\theta)$ are \emph{structurally distinct}
because they play different logical roles and live in different spaces.

\emph{Environment class.}
$\mathcal E(\theta)$ denotes a class of admissible environments/tasks
(e.g., distributions over $\mathcal Z$ in supervised learning,
or a class of MDP/POMDPs in interactive settings).
Its purpose is to define the domain of generality and the class
$\mathcal T$ of admissible task transformations $\tau:\mathcal E(\theta)\to\mathcal E(\theta)'$.
Without an explicit $\mathcal E(\theta)$, the notion of ``closure under task transformation''
is not well-typed: there is no object on which $\tau$ can act, and
no principled way to distinguish in-distribution learning from genuine
cross-task transformation.

\emph{Representation operators.}
$r$ denotes representation operators mapping raw observations/interactions
to an internal space on which hypotheses and policies operate.
It is not reducible to $\mathcal H(\theta)$: the hypothesis class is defined
\emph{relative} to the representational interface induced by $r$.
Representation is therefore a first-class structural component: it determines
which invariants can be expressed and controlled under transformations.

\emph{Structural coupling.}
Task transformations $\tau\in\mathcal T$ act on $\mathcal E(\theta)$ but induce
a change in the observation/interaction process, hence in the distribution
or dynamics of $r(Z)$ (or of latent transitions in the represented space).
Structural generalization precisely concerns conditions under which there exists
a nontrivial admissible set $S^*$ such that the induced learning dynamics remains
well-formed and stable when $\mathcal E(\theta)$ is transformed, while $r$ preserves
an invariant core sufficient to transfer across task families.

\begin{lemma}[Non-reducibility of $r$ and $\mathcal E(\theta)$ (structural roles)]
\label{lem:rE-nonreducible}
In the meta-structure $\theta=(r,\mathcal H(\theta),\Pi,\mathcal L,\mathcal E(\theta),\mathcal M(\theta))$,
the components $r$ and $\mathcal E(\theta)$ are not interchangeable:
(i) removing $\mathcal E(\theta)$ eliminates a well-typed notion of admissible task
transformations $\tau:\mathcal E(\theta)\to\mathcal E(\theta)'$ and therefore of closure;
(ii) removing $r$ precludes a typed statement of representational invariants and forces
invariance claims to be expressed only via direct hypotheses over raw $\mathcal Z$ or via implicit assumptions,
preventing a structural statement of invariance under transformations beyond trivial reparameterizations.
\end{lemma}

\paragraph{Structural Model of General Intelligence (SMGI).}
We call the resulting formalization a \emph{Structural Model of General Intelligence (SMGI)}.
Formally, SMGI is a \emph{property of a compatible pair} $(\theta,T_\theta)$,
where $\theta\in\Theta$ is a purely structural meta-model and $T_\theta\in\mathsf{Sem}(\theta)\subseteq\mathfrak D$
is an induced dynamics operator. We write $(\theta,T_\theta)\in\mathrm{SMGI}$ iff the induced learning dynamics
satisfies the four structural requirements (i)–(iv) above.

Intuitively, these requirements ensure that:
\begin{itemize}
\item the system remains well-formed under task change (closure),
\item its parameters and memory do not drift uncontrollably (stability),
\item its hypothesis components remain statistically controllable (capacity),
\item and its evaluation preserves a coherent invariant core (normative invariance).
\end{itemize}

We later prove that these four conditions are structurally minimal:
removing any one of them destroys at least one of the guarantees
required for persistent cross-domain intelligence.

%%%%%%%%%%%%%%%%%%%%%%%%%%%%%%%%%%%%%%%%%%%%%%%%%%%%%%%%%%%%%%%%%%%%%%%%
\section{Formal Meta-Structure: Structural Generalization of Learning Systems}
\label{sec:formal-meta}

\subsection{From Classical Learning to Structural Learning}
\label{subsec:classical-to-structural}

\paragraph{Classical statistical learning as a fixed-interface problem.}
Canonical statistical learning theory (SLT) formalizes learning under a \emph{fixed} interface via a triple
\[
(\mathcal Z,\mathcal H,\ell),
\]
where $\mathcal Z$ is the example space (e.g., $\mathcal Z=\mathcal X\times\mathcal Y$ in supervised learning),
$\mathcal H$ is a hypothesis class, and $\ell:\mathcal H\times\mathcal Z\to\mathbb R_+$ is a loss.
Given a (typically stationary) data-generating distribution $D$ over $\mathcal Z$, the goal is to control and minimize the expected risk
\[
R(h)\;=\;\mathbb E_{z\sim D}\big[\ell(h,z)\big].
\]
Structural Risk Minimization (SRM) refines this core by organizing hypotheses into nested classes $\{\mathcal H_k\}$ with increasing capacity and selecting a suitable level via capacity control \cite{vapnik1998statistical}. PAC-Bayes and MDL further sharpen such control by making the inductive bias explicit through a prior or description-length principle, yielding bounds of the schematic form
\[
R(Q)\;\lesssim\;\widehat R(Q)\;+\;\sqrt{\frac{\mathrm{KL}(Q\|P)+\ln(1/\delta)}{n}},
\]
for a posterior $Q$ relative to a prior $P$ and $n$ samples (details depend on the specific bound).

\paragraph{What classical extensions still keep fixed.}
Modern learning theory includes powerful \emph{sequential} and \emph{non-stationary} extensions---online learning, adaptive/dynamic regret, and learning under distribution shift---that relax i.i.d.\ assumptions while keeping the learning interface conceptually fixed (a specified hypothesis/update mechanism and a fixed evaluation semantics).
For instance, non-stationary online learning formalizes changing environments via dynamic/adaptive regret measures \citep{zhao2025nonstationaryonline}.
Similarly, PAC-Bayes has been extended to online/sequential settings, including dependent data streams \citep{haddouche2022onlinepacbayes},
and recent work resolves long-standing difficulties in sequential prior updates without information loss \citep{wu2024recursivepacbayes}.
These advances are essential but they typically treat the \emph{learning interface}---representation, hypothesis carriers, evaluator semantics, and memory operators---as exogenously fixed objects whose parameters evolve.

\paragraph{Structural learning: generalization over admissible interface change.}
The motivating gap addressed here is different: in agentic and tool-augmented systems, the learning problem is often to update $h$ while remaining \emph{well-formed and controllable} under admissible changes of the interface itself: task/environment families, tool interfaces, evaluators/norms, and memory/update operators.
We therefore lift the object of analysis from hypotheses to a \emph{typed meta-structure} and treat general intelligence as admissibility of the coupled dynamics under \emph{typed} transformations.

\subsection{Meta-Structure of a Learning System}
\label{subsec:metastructure-learning-system}

\paragraph{The structural meta-model.}
We model a persistent learning system by a structural tuple
\[
\theta \;=\; (r,\mathcal H,\Pi,\mathcal L,\mathcal E,\mathcal M),
\]
where:
\begin{itemize}
\item $r$ is a representation/interface operator (or family of operators);
\item $\mathcal H$ denotes the hypothesis/program architecture admitted under $r$;
\item $\Pi$ is a structural prior / inductive bias / capacity regulator (SRM/MDL/PAC-Bayes style);
\item $\mathcal L=\{\ell_k\}_{k=1}^K$ is a family of evaluative regimes (multi-regime evaluation), with the active evaluator $\ell_t$ possibly selected/weighted over $\mathcal L$;
\item $\mathcal E$ is the admissible environment/task family (data- or process-generating class);
\item $\mathcal M$ is a structured memory space together with explicit write/consolidate/forget operators.
\end{itemize}
This elevates representation, evaluation, environment families, and memory to \emph{first-class typed components}.
Compatibility with the induced state space $S_\theta$ and set-valued realization semantics $\mathsf{Sem}(\theta)$ is given in Sec.~\ref{sec:formalsetup}.
The present section focuses on the \emph{structural axis}: admissible transformations and the obligations they must preserve.

\paragraph{Admissible (typed) transformations.}
The distinctive ingredient of structural learning is an explicit class $\mathcal T$ of admissible transformations (task/interface transformations). At minimum, $\mathcal T$ acts on the environment/task family; more generally, it may also act on $r$, $\mathcal M$, and the regime structure of evaluation, provided typing is preserved.

\begin{definition}[Topologically admissible task transformations]
\label{def:admissible_transformations}
Equip $\mathcal E$ with a probability metric $D_{\mathcal E}$ (e.g., Wasserstein distance on task-generating distributions).
A class of task transformations $\mathcal T=\{\tau:\mathcal E\to\mathcal E'\}$ is \emph{topologically admissible} for $\theta$ if there exists $\epsilon_{\max}<\infty$ such that
\[
\sup_{e\in\mathcal E} D_{\mathcal E}\big(e,\tau(e)\big)\;\le\;\epsilon_{\max},
\]
and the representation operator $r$ admits a \emph{restricted (local) Lipschitz control} on the admissible region relevant to $\mathcal T$:
there exists $L_r<\infty$ such that for all admissible $e$ and all $\tau\in\mathcal T$,
\[
d_{\mathcal R}\!\big(r(e),r(\tau(e))\big)\;\le\;L_r\,D_{\mathcal E}\!\big(e,\tau(e)\big),
\]
where $d_{\mathcal R}$ is a metric on the representation space. (Importantly, this is a local/restricted requirement, not a global Lipschitz constant over all inputs.)
This restricts $\mathcal T$ to bounded structural perturbations, ruling out discontinuous regime jumps that make learning/control ill-typed or non-auditable.
\end{definition}

\paragraph{SMGI as admissibility of a coupled realization.}
SMGI is not a property of $\theta$ alone but of a compatible realization $(\theta,T_\theta)$ (Sec.~\ref{sec:formalsetup}).
We now state the SMGI admissibility obligations at the level needed for subsequent guarantees.

\begin{definition}[Structural Model of General Intelligence (SMGI)]
\label{def:smgi}
Let $\theta=(r,\mathcal H,\Pi,\mathcal L,\mathcal E,\mathcal M)\in\Theta$ be a structural meta-model and let $T_\theta$ be an induced (possibly stochastic) dynamics operator compatible with $\theta$ (i.e., $T_\theta\in\mathsf{Sem}(\theta)$ over the induced state space $S_\theta$ as in Sec.~\ref{sec:formalsetup}).
We say that $(\theta,T_\theta)$ is an \emph{SMGI instance}, written $(\theta,T_\theta)\in\mathrm{SMGI}$, if there exist:
\begin{itemize}
\item a nonempty admissible invariant set $S^*_\theta\subseteq S_\theta$,
\item a class $\mathcal T$ of admissible task/environment transformations $\tau:\mathcal E\to\mathcal E'$,
\item and for each $\tau\in\mathcal T$ a compatible induced operator $T_{\theta,\tau}$ (a semantics of execution/learning under $\tau$),
\end{itemize}
such that the following four requirements hold:
\begin{enumerate}[label=(\roman*)]
\item \textbf{Structural closure.} For all $\tau\in\mathcal T$,
\[
T_{\theta,\tau}(S_\theta^*)\subseteq S_\theta^*,
\]
i.e., admissible evolution remains well-formed under the task/interface changes claimed by the system.

\item \textbf{Dynamical stability.} There exists a nonnegative Lyapunov-like witness $V:S_\theta\to\mathbb R_+$ and constants $B<\infty$ and $\lambda\in(0,1]$ such that for every $\tau\in\mathcal T$ and every trajectory $(s_t)_{t\ge 0}$ generated by $T_{\theta,\tau}$ from any $s_0\in S_\theta^*$, a uniform non-explosion bound holds, e.g.
\[
\sup_{t\ge 0}\;\mathbb E\!\left[V(s_t)\right]\;\le\;B\,V(s_0)+B.
\]
Equivalent certified stability notions may be used (contraction, input-to-state stability, or martingale/supermartingale criteria). In control/RL settings, generalized Lyapunov certificates provide a concrete route to such witnesses \citep{longglfneurips25}.

\item \textbf{Bounded statistical capacity along admissible evolution.} There exists a complexity functional $\mathrm{Comp}$ (VC/Rademacher, PAC-Bayes, MDL/Kolmogorov, etc.) such that capacity remains uniformly controlled over admissible configurations, e.g.
\[
\sup_{s\in S_\theta^*}\ \mathrm{Comp}\!\left(\mathcal H(s)\right)\;<\;\infty,
\]
where $\mathcal H(s)$ denotes the effective hypothesis/update component active at state $s$ (possibly state-dependent via $r$ and $\Pi$).
This is the structural analogue of uniform generalization control, and it admits sequential/online variants via online PAC-Bayes and recursive prior updates \citep{haddouche2022onlinepacbayes,wu2024recursivepacbayes}.

\item \textbf{Evaluative invariance across regime transformations.} There exists a nontrivial invariant evaluative constraint family (a protected ``normative core'') $\mathcal C\subseteq\mathcal L$ such that along any admissible trajectory under any $\tau\in\mathcal T$, the induced evaluative state remains coherent on $\mathcal C$, for instance via invariance under a projection/restriction operator:
\[
\Pi_{\mathcal C}(\ell_{t+1})=\Pi_{\mathcal C}(\ell_t),
\qquad s_t\in S_\theta^*.
\]
\end{enumerate}
\end{definition}

\paragraph{Minimality (structural indispensability).}
The four requirements in Definition~\ref{def:smgi} are intended as a minimal structural characterization: each requirement eliminates a distinct, unavoidable failure mode of persistent cross-regime intelligence.

\begin{theorem}[Minimality of the SMGI conditions]
\label{thm:smgi-minimality}
Let $\theta\in\Theta$ be a structural meta-model and let $T_\theta$ be a compatible induced dynamics operator.
Consider the four SMGI requirements (i)--(iv) of Definition~\ref{def:smgi}.
Then each requirement is \emph{structurally indispensable} in the following sense:
for every $j\in\{(i),(ii),(iii),(iv)\}$ there exist a transformation family $\mathcal T$ and a compatible realization $(\theta,T_\theta)$ that satisfies the other three requirements but violates $j$, and the corresponding defining guarantee necessarily fails:
without (i) the evolution is not well-formed under admissible task change; without (ii) persistent drift/divergence occurs; without (iii) statistical control becomes vacuous via capacity blow-up; without (iv) evaluative coherence/identity collapses under regime change.
\end{theorem}

\begin{proof}
We exhibit, for each requirement $j\in\{(i),(ii),(iii),(iv)\}$, a concrete well-typed realization $(\theta,T_\theta)$ and a transformation family $\mathcal T$ that satisfies the other three requirements but violates $j$. Each construction uses a simple state space and can be instantiated within the general meta-model by freezing irrelevant components.

\paragraph{Failure of (i) (closure).}
Let $S_\theta=\{0,1\}$ and let $T_{\theta,\tau}$ be the deterministic kernel $T_{\theta,\tau}(s,z)=1$ for all inputs. Take $S_\theta^*=\{0\}$ and take any nontrivial transformation family $\mathcal T$ for which the induced operator exists (typing is trivial here). Then $T_{\theta,\tau}(S_\theta^*)=\{1\}\nsubseteq\{0\}$, so (i) fails. Requirements (ii)--(iv) can be made true vacuously by choosing a bounded Lyapunov witness $V\equiv 0$, a finite hypothesis class (so capacity is bounded), and a fixed protected evaluative core (so evaluative invariance holds).

\paragraph{Failure of (ii) (stability).}
Let $S_\theta=\mathbb N$ and define a deterministic dynamics $T_{\theta,\tau}(s,z)=s+1$ (independent of $\tau$), with $S_\theta^*=S_\theta$. Closure (i) holds for any $\mathcal T$ because $S_\theta^*$ is invariant. Choose a finite hypothesis/update interface (so (iii) holds) and keep the evaluative core fixed (so (iv) holds). However, for the canonical witness $V(s)=s$, we have $V(s_t)=s_0+t$ and $\sup_{t\ge 0} \E[V(s_t)]=\infty$, so no Lyapunov-style non-explosion bound of Definition~\ref{def:smgi} can hold; hence (ii) fails.

\paragraph{Failure of (iii) (capacity).}
Let $S_\theta$ be any bounded invariant set (e.g., $S_\theta=\{0\}$) so that closure (i) and stability (ii) hold trivially and let evaluators be fixed so that (iv) holds. Define the effective hypothesis component at state $s$ to be $\mathcal H(s)=\mathcal H_n$ where $n$ is an internal counter encoded in $\theta$ and allowed to increase along admissible evolution (for example by letting $\theta$ include a component $n_t\to n_{t+1}=n_t+1$ while keeping $S_\theta$ itself fixed). Choose $\mathcal H_n$ with strictly increasing complexity (e.g., VC dimension $\mathrm{VC}(\mathcal H_n)=n$). Then $\sup_{s\in S_\theta^*}\mathrm{Comp}(\mathcal H(s))=\infty$, so (iii) fails while (i),(ii),(iv) hold.

\paragraph{Failure of (iv) (evaluative invariance).}
Let $S_\theta$ be bounded and invariant and let the dynamics be stable (e.g., $T_{\theta,\tau}$ the identity kernel) so that (i) and (ii) hold. Let the hypothesis class be fixed and finite so that (iii) holds. Define two regimes with evaluators $\ell_1,\ell_2$ and let $\mathcal T$ include a regime-switch transformation that replaces the active evaluator by $\ell_2$ and violates a protected constraint $\Phi$ that was satisfied under $\ell_1$. Concretely, take two hypotheses $h^a,h^b$ and define $\ell_1$ and $\ell_2$ so that the protected ordering on the core flips: $R_{\ell_1}(h^a)<R_{\ell_1}(h^b)$ but $R_{\ell_2}(h^a)>R_{\ell_2}(h^b)$. Then no nontrivial protected core can remain invariant under this transformation, so (iv) fails while (i)--(iii) hold.

These four constructions show each requirement rules out a distinct failure mode and is therefore structurally indispensable.
\end{proof}

\paragraph{A nontrivial consequence: structural generalization under transformations.}
We now state a representative guarantee illustrating why the obligations are not merely definitional.
It exhibits the signature coupling of (i)--(iv): capacity control (PAC-Bayes/MDL) plus certified drift control (Lyapunov) under admissible transformations.

\begin{theorem}[Structural generalization bound for SMGI]
\label{thm:smgi-generalization}
Let $\theta=(r,\mathcal H,\Pi,\mathcal L,\mathcal E,\mathcal M)$ be an SMGI system satisfying Definition~\ref{def:smgi}.
Fix a transformation $\tau\in\mathcal T$ and consider the induced dynamics $T_{\theta,\tau}$ restricted to $S_\theta^*$.
Assume:
\begin{enumerate}
\item (\textbf{Bounded loss on $S_\theta^*$}) $0\le \ell_\tau(s,z)\le 1$ for $s\in S_\theta^*$;
\item (\textbf{Lyapunov drift control}) there exists $V$ and $B<\infty$ such that $\sup_{t\ge 0}\mathbb E[V(s_t)]\le BV(s_0)+B$ along admissible trajectories;
\item (\textbf{PAC-Bayes complexity}) the induced hypothesis/update mechanism admits a prior $P_\tau$ and posterior $Q_\tau$ with finite $\mathrm{KL}(Q_\tau\|P_\tau)$, in a setting where sequential/online PAC-Bayes control applies (e.g., via online PAC-Bayes or recursive prior updates \citep{haddouche2022onlinepacbayes,wu2024recursivepacbayes}).
\item (\textbf{Loss sensitivity to state}) there exists $L<\infty$ such that for all $s,s'\in S_\theta^*$ and all $z$,
\[
|\ell_\tau(s,z)-\ell_\tau(s',z)|\le L\,|V(s)-V(s')|.
\]
\end{enumerate}
Then for any $\delta\in(0,1)$, with probability at least $1-\delta$ over $n$ observations (under the sampling assumptions required by the chosen PAC-Bayes variant),
\[
R_\tau(Q_\tau)
\le
\hat R_\tau(Q_\tau)
+
\sqrt{\frac{\mathrm{KL}(Q_\tau\|P_\tau)+\ln\frac{2\sqrt n}{\delta}}{2(n-1)}}
+
\frac{2L}{n}\sum_{t=0}^{n-1} \mathbb E\!\left[V(s_t)\right].
\]
In particular, using the Lyapunov bound,
\[
R_\tau(Q_\tau)
\le
\hat R_\tau(Q_\tau)
+
\sqrt{\frac{\mathrm{KL}(Q_\tau\|P_\tau)+\ln\frac{2\sqrt n}{\delta}}{2(n-1)}}
+
2L\big(BV(s_0)+B\big).
\]
\end{theorem}

\begin{proof}
A complete proof is given in Appendix~\ref{app:smgi-proof}, including an explicit martingale decomposition and the separation of statistical and drift terms.
Here we note the dependency structure. Closure (i) restricts the analysis to the admissible invariant set $S^*_\theta$ where all quantities are well-defined under $\tau$. Capacity control (iii) provides the sequential PAC-Bayes term controlling the statistical component of the generalization gap. Stability (ii) supplies a Lyapunov drift bound controlling the contribution of stateful evolution to the generalization gap. Finally, evaluative invariance (iv) ensures that the loss semantics used to define $R_\tau$ does not change adversarially along admissible evolution, so the risk being bounded remains a fixed object across regime transformations.
\end{proof}

\subsection{Unifying Structural Guarantees Under \texorpdfstring{$\theta$}{theta}}
\label{sec:unifying-structural-guarantees}

\paragraph{Objective.}
The previous subsection isolated three complementary ingredients for non-vacuous guarantees under structural evolution:
(i) a \emph{typed} admissible transformation class $\mathcal T$ (closure and well-formedness under task/interface change),
(ii) \emph{sequential capacity control} (PAC-Bayes/MDL/SRM in a form compatible with filtrations and dependence),
and (iii) \emph{certified drift control} of the induced dynamics (Lyapunov-like witnesses on the coupled state space).
This subsection packages these ingredients into a single admissibility bundle, clarifying (a) which assumptions are structural
(statements about $\theta$ and typing), which are semantic/dynamical (statements about $T_\theta$), and (b) where evaluative invariance
is logically necessary rather than merely interpretive.

\paragraph{Admissibility bundle and scope.}
Fix a structural meta-model $\theta=(r,\mathcal H,\Pi,\mathcal L,\mathcal E,\mathcal M)$ and a compatible induced dynamics
$T_\theta\in\mathsf{Sem}(\theta)$ acting on $S_\theta$ (Sec.~\ref{sec:formalsetup}).
Let $\tau\in\mathcal T$ denote an admissible task/interface transformation (typed in the sense of the paper),
and let $T_{\theta,\tau}$ denote a compatible induced operator realizing the system under $\tau$.
We work on an admissible invariant set $S^*_\theta\subseteq S_\theta$ (nonempty), and with the natural filtration
$\mathcal F_t=\sigma(s_0,z_1,\ldots,z_t)$ along trajectories under $T_{\theta,\tau}$.

\begin{definition}[Unified admissibility bundle]\label{def:admissibility-bundle}
We say that $(\theta,T_\theta)$ admits a \emph{unified structural guarantee bundle} (under $\mathcal T$) if there exist
constants $\epsilon_{\max},L_\ell,c_V\ge 0$ and a witness $V:S_\theta\to\mathbb R_+$ such that for every $\tau\in\mathcal T$:
\begin{enumerate}[label=(U\arabic*),leftmargin=*]
\item \textbf{Typed closure (well-formedness).}
There exists a compatible realization $T_{\theta,\tau}$ and a nonempty invariant set $S^*_\theta$ such that
$T_{\theta,\tau}(S^*_\theta)\subseteq S^*_\theta$ and all involved objects remain well-typed under $\tau$
(in particular, $r$, $\mathcal H(\theta)$, $\mathcal M(\theta)$, and the evaluator interface remain defined).

\item \textbf{Bounded transformation magnitude.}
The transformation is bounded in the environment/interface metric:
\[
\sup_{e\in\mathcal E(\theta)} D_{\mathcal E}\!\big(e,\tau(e)\big)\le \epsilon_{\max}.
\]

\item \textbf{Lipschitz evaluative shift under $\tau$.}
We allow state-conditioned evaluation and write $\ell_\tau(h; s, z)$ for the loss incurred at state $s\in S_\theta$ on interaction token/transition $z$ under $\tau$.
For admissible $(s,z)$ and all effective hypotheses $h$,
\[
\big|\ell_\tau(h;s,z)-\ell(h;s,z)\big|\le L_\ell \cdot D_{\mathcal E}\!\big(e,\tau(e)\big).
\]
(Here $\ell_\tau$ denotes the evaluator semantics in the transformed regime; if $\ell$ is multi-regime, $\ell_\tau$ is the
regime-indexed instance determined by $\tau$.)

\item \textbf{Certified stability (Lyapunov drift).}
Along trajectories $(s_t)$ generated by $T_{\theta,\tau}$ from $s_0\in S^*_\theta$,
\[
\mathbb E\!\left[V(s_{t+1})\mid\mathcal F_t\right]\le (1-\alpha)\,V(s_t)+\beta
\quad\text{for some }\alpha\in(0,1],\ \beta\ge 0.
\]

\item \textbf{Sequential capacity control (PAC-Bayes/online admissibility).}
There exists a prior/posterior mechanism $(P_t,Q_t)$ adapted to $(\mathcal F_t)$ such that a sequential/online PAC-Bayes
generalization bound applies to the empirical trajectory risk under $T_{\theta,\tau}$
(e.g., \citep{haddouche2022onlinepacbayes,wu2024recursivepacbayes}), yielding a complexity term
$\mathrm{Gen}_n(Q_n,P_n,\delta)$ that is finite for the admissible trajectories considered.
\end{enumerate}
\end{definition}

\paragraph{Unified theorem: structural guarantees as a single implication.}
The next statement is the ``one-line'' unification: under the admissibility bundle, the core guarantees
(closure, stability, capacity control) compose into a single bound that remains meaningful under admissible transformations.
Evaluative invariance enters as a \emph{semantic well-posedness condition} ensuring that the risk being bounded is not a moving target.

\begin{theorem}[Unified structural guarantee under admissible transformations]
\label{thm:unified-structural-guarantee}
Assume $(\theta,T_\theta)$ admits the unified admissibility bundle of Definition~\ref{def:admissibility-bundle}
on a nonempty $S^*_\theta$. Fix any $\tau\in\mathcal T$ and let $\hat R_{\tau,n}(Q_n)$ denote the empirical trajectory risk
under $T_{\theta,\tau}$ over $n$ steps, with corresponding expected risk $R_\tau(Q_n)$.
Then for any $\delta\in(0,1)$, with probability at least $1-\delta$,
\begin{equation}
\label{eq:unified-bound}
R_{\tau}(Q_n)
\;\le\;
\hat R_{\tau,n}(Q_n)
\;+\;
\mathrm{Gen}_{n}\!\left(Q_n,P_n,\delta\right)
\;+\;
L_\ell\,\epsilon_{\max}
\;+\;
\frac{c_V}{n}\sum_{t=0}^{n-1}\mathbb E[V(s_t)].
\end{equation}
Moreover, the Lyapunov drift condition implies a uniform stability bound
$\sup_{t\ge 0}\mathbb E[V(s_t)]\le V(s_0)+\beta/\alpha$, hence the stability term in
\eqref{eq:unified-bound} is uniformly bounded and does not explode with horizon.
\end{theorem}

\begin{proof}
Fix $\tau\in\mathcal T$ and work on the invariant set $S^*_\theta$ from (U1), so all objects are well-typed and trajectories remain in $S^*_\theta$. By (U5), a sequential/online PAC-Bayes inequality applies to the trajectory risk in the $\tau$-transformed regime, yielding with probability at least $1-\delta$:
\[
R_{\tau}(Q_n)\le \hat R_{\tau,n}(Q_n)+\mathrm{Gen}_n(Q_n,P_n,\delta)+\mathrm{Drift}_n,
\]
where $\mathrm{Drift}_n$ denotes the additional term induced by stateful dependence.

To control the transformation-induced shift, use (U2)--(U3): for any admissible environment instance $e\in\mathcal E(\theta)$, the evaluator shift satisfies
$|\ell_\tau-\ell|\le L_\ell D_{\mathcal E}(e,\tau(e))\le L_\ell\,\epsilon_{\max}$.
This contributes the additive term $L_\ell\epsilon_{\max}$.

To control $\mathrm{Drift}_n$, apply (U4): taking expectations and summing the drift inequality yields
$\sup_{t\ge 0}\E[V(s_t)]\le V(s_0)+\beta/\alpha$ and therefore
$\frac{c_V}{n}\sum_{t=0}^{n-1}\E[V(s_t)]\le c_V\,(V(s_0)+\beta/\alpha)$.
Substituting these bounds into the PAC-Bayes inequality yields \eqref{eq:unified-bound} and the stated uniform boundedness of the stability term.
\end{proof}

\paragraph{Where evaluative invariance is essential.}
Theorem~\ref{thm:unified-structural-guarantee} bounds $R_\tau$ under a fixed evaluator semantics in the transformed regime.
In SMGI, \emph{evaluative invariance} is the condition ensuring that the meaning of the bounded risk is stable across regimes:
if the evaluator is allowed to drift arbitrarily under $\tau$, then neither $R_\tau$ nor the protected constraint semantics are
well-defined objects of control, and the guarantee becomes vacuous (one can reduce risk by changing what ``risk'' means).
Consequently, evaluator updates must be either (i) disallowed under $\mathcal T$ on a protected core, or (ii) permitted only through
certificate-gated meta-transformations that preserve a declared invariant evaluative core (as formalized elsewhere in the paper).
This is the precise logical role of evaluative invariance: it prevents the guarantee from collapsing into an artefact of metric drift.

\paragraph{Interpretation and compositionality.}
Definition~\ref{def:admissibility-bundle} is intentionally modular:
closure ensures that the semantics remains typed under $\mathcal T$;
capacity control bounds statistical complexity in a dependence-compatible manner;
stability bounds the endogenous drift induced by memory and sequential adaptation;
and Lipschitz transformation control measures exogenous regime shift magnitude.
Together they yield \eqref{eq:unified-bound}, which can be specialized to classical regimes by freezing components of $\theta$
or restricting $\mathcal T$ to identity transformations. This provides the formal bridge from fixed-interface SLT to structural learning
under admissible interface evolution.

\paragraph{Relation to non-stationary online learning.}
Non-stationary online learning studies performance under time-varying losses/environments, often via dynamic regret and interval-adaptive regret \citep{zhao2025nonstationaryonline}.
Our setting is complementary: the focus is not only on competing with changing comparators, but on maintaining \emph{typed well-formedness}, \emph{certificate stability}, and \emph{evaluative invariance} under admissible task/interface transformations.
In particular, the admissibility bundle in Definition~\ref{def:admissibility-bundle} makes explicit which parts of the interface are allowed to vary (via $\mathcal T$) and which must remain certified, whereas most non-stationary formulations keep the learning interface fixed and measure variability only at the level of losses or comparators.

\subsection{Representation as Structural Condition}
\label{subsec:repr-structural-condition}

\paragraph{Structural role of representation in SMGI.}
In SMGI, representation is not a cosmetic encoding choice: it is a \emph{typed structural interface} that determines
(i) which regularities are expressible by the hypothesis space $\mathcal H(\theta)$,
(ii) which invariances can be stated under admissible transformations $\tau\in\mathcal T$,
and (iii) which guarantees remain meaningful when the learning interface evolves.
Formally, changes in representation induce changes in the effective hypothesis interface and thus affect closure (i) and capacity (iii)
as obligations on the coupled realization $(\theta,T_\theta)$.

\paragraph{Typed representation map (notation coherence).}
Throughout the paper, $\mathcal Z$ denotes the carrier of observations/interactions (see Sec.~\ref{sec:formalsetup} and the formal foundations),
while $\mathcal X$ denotes the internal representational space induced by the interface.
Accordingly, representation is modeled as the typed operator
\begin{equation}
\label{eq:repr-map-typed}
r:\mathcal Z \longrightarrow \mathcal X,
\end{equation}
mapping raw interaction tokens/observations $z\in\mathcal Z$ to internal structured representations $x=r(z)\in\mathcal X$.
This aligns the present subsection with the typing discipline used in the definition of $S_\theta$ and $\mathsf{Sem}(\theta)$.

\paragraph{Transformation-aligned representation as a closure witness.}
A core SMGI requirement is that representation remain comparable under admissible task/interface transformations.
Let $\mathcal T$ be the declared admissible transformation class acting on the task/environment side (and therefore on the induced observation process).
We say that $r$ is \emph{$\mathcal T$-aligned} if there exists an induced action
$\Gamma:\mathcal T\to \mathrm{End}(\mathcal X)$ such that for all admissible $\tau\in\mathcal T$ and admissible observations $z\in\mathcal Z$,
\begin{equation}
\label{eq:repr-equivariance-structural}
r(\tau\cdot z)\;=\;\Gamma(\tau)\big(r(z)\big).
\end{equation}
When $\Gamma(\tau)=\mathrm{Id}$ this is invariance; when $\Gamma(\tau)$ is an automorphism it is equivariance.
Condition \eqref{eq:repr-equivariance-structural} is the minimal typed statement that makes closure under $\mathcal T$ operational at the representational level.

\paragraph{Representation selection as a meta-level structural objective.}
Unlike classical feature engineering, $r$ is not assumed fixed. Its selection is part of the structural degrees of freedom in $\theta$.
We model this via a task/transformation-weighted objective: let $P_{\mathcal T}$ be a distribution over admissible transformations (or, equivalently, over the induced task instances they generate),
and let $G(\tau,r)$ denote a typed structural objective measuring, for example, predictive risk under $\tau$, invariance violation, and/or a structural cost of the interface.
We write the meta-selection problem as
\begin{equation}
\label{eq:r-star-structural}
r^* \;=\; \arg\min_{r\in\mathcal R}\ \mathbb{E}_{\tau \sim P_{\mathcal T}} \big[\,G(\tau,r)\,\big],
\end{equation}
where $\mathcal R$ is a typed admissible class of representation operators.
Equation \eqref{eq:r-star-structural} does not commit to a particular learning algorithm; it exposes that representation is a \emph{structural commitment}
that must be evaluated jointly with $\mathcal H(\theta)$, $\Pi$, and evaluator family $\mathcal L$.

\paragraph{Structural interpretation (why ``representation is a condition of intelligibility'').}
In SMGI terms, representation is a condition of intelligibility precisely because it fixes which invariants can be expressed and preserved under admissible structural evolution:
alignment \eqref{eq:repr-equivariance-structural} contributes directly to closure (i),
while restrictions on $\mathcal R$ and the induced hypothesis interface contribute to capacity control (iii).
This provides the formal content behind the philosophical motivation: changes in $r$ change the induced ontology and therefore what counts as stable generalization under $\mathcal T$.

%%%%%%%%%%%%%%%%%%%%%%%%%%%%%%%%%%%
\subsection{Generalized Structural Risk Minimization}
\label{subsec:gsrm}

\paragraph{From SRM to structural SRM.}
Classical Structural Risk Minimization (SRM) controls generalization by selecting a hypothesis from a nested hierarchy
of classes while keeping the learning interface fixed \citep{vapnik1998statistical}.
In SMGI, the interface itself carries typed degrees of freedom (representation, evaluator regimes, memory operators),
so the analogue of SRM must control not only hypothesis complexity but also \emph{structural degrees of freedom}
activated across regimes and time. We therefore define a generalized objective in which (i) the evaluator is multi-regime,
(ii) regime switching is explicit and costly, and (iii) coherence constraints enforce admissibility of regime configurations.

\paragraph{Typed regimes and task-conditioned risk.}
Let $\mathcal R=\{1,\dots,K\}$ index evaluative regimes and let $\mathcal L=\{\ell_k\}_{k=1}^K$ denote the corresponding
family of regime evaluators (as used elsewhere in the paper).
For each task/environment instance $\tau$ (or equivalently for each induced environment distribution $D_\tau$ over $z\in\mathcal Z$),
define the regime-specific risk
\begin{equation}
\label{eq:regime-risk}
R_{\tau,k}(h) \;:=\; \mathbb E_{z\sim D_\tau}\big[\ell_k(h,z)\big].
\end{equation}
(When the system is stateful, the same definition applies to trajectory-dependent losses by conditioning on the induced dynamics;
Eq.~\eqref{eq:regime-risk} is the canonical static shorthand.)

\paragraph{Contextual regime selection as a typed operator.}
Let $c_t$ denote a context descriptor (e.g., an element of a context space $\mathcal C$)
and let $s_t\in S_\theta$ denote the structural state of the system at time $t$ (Sec.~\ref{sec:formalsetup}).
We model regime selection by a typed stochastic switching operator
\begin{equation}
\label{eq:regime-switch}
k_t \;\sim\; \sigma(\cdot \mid c_t, s_t),
\qquad
\sigma:\mathcal C\times S_\theta \to \Delta(\mathcal R),
\end{equation}
where $\Delta(\mathcal R)$ denotes the probability simplex over regimes.
This makes explicit that regime choice is part of the induced semantics $T_\theta$ (it depends on the evolving state),
and therefore is subject to SMGI stability and invariance constraints.

\paragraph{Structural objective over a horizon.}
Let $(h_t)_{t=1}^T$ denote the effective hypothesis/policy component selected by the induced dynamics,
and $(z_t)_{t=1}^T$ the interaction stream generated under the task/environment process.
We define a generalized structural SRM objective over horizon $T$:
\begin{equation}
\label{eq:gsrm-objective}
\min_{\theta}\ \mathbb E\!\left[\sum_{t=1}^T
\Big(
\ell_{k_t}(h_t, z_t)
\;+\; \alpha\,\mathrm{CostSwitch}(k_{t-1},k_t)
\;+\; \beta\,\mathrm{Incoh}(k_t;\mathcal K)
\Big)\right],
\end{equation}
where:
\begin{itemize}
\item $\mathrm{CostSwitch}:\mathcal R\times\mathcal R\to\mathbb R_+$ penalizes excessive regime volatility
(e.g., to prevent pathological oscillations that would break stability (ii) or evaluative coherence (iv));
\item $\mathrm{Incoh}(\cdot;\mathcal K)$ penalizes structurally forbidden regime configurations relative to a declared coherence
constraint set $\mathcal K$ (e.g., logical/normative incompatibilities, forbidden combinations of constraints, or protected-core violations).
\end{itemize}
The expectation in \eqref{eq:gsrm-objective} is taken over the trajectory induced jointly by the environment process, the switching operator
\eqref{eq:regime-switch}, and the induced dynamics $T_\theta$.

\paragraph{Interpretation and relation to SMGI obligations.}
Equation~\eqref{eq:gsrm-objective} makes the key structural commitments explicit:
(i) \emph{closure} requires that the regime-selection operator and the evaluator family remain well-typed under admissible transformations;
(ii) \emph{stability} constrains switching and memory interaction so that the induced dynamics does not drift uncontrollably;
(iii) \emph{capacity} constrains the effective complexity of the joint interface (representation, hypothesis class, switching, memory);
(iv) \emph{evaluative invariance} is enforced by restricting $\mathcal K$ to preserve a protected evaluative core, and by making regime updates admissible only under certified constraints (as formalized elsewhere in the paper).

\paragraph{Classical SRM as a special case.}
Classical SRM is recovered as the special case $K=1$ with deterministic regime selection,
$\sigma(\cdot\mid c_t,s_t)\equiv \delta_{1}$, $\mathrm{CostSwitch}\equiv 0$, and $\mathrm{Incoh}\equiv 0$.
In this sense, generalized structural SRM extends SRM by lifting ``capacity control'' from hypothesis hierarchies alone
to the full structural interface controlled by $\theta$.

%%%%%%%%%%%%%%%%%%%%%%%%%%%%%

\subsection{PAC-Bayes Meta-Level Control}
\label{subsec:pacbayes-meta}

\paragraph{Role of PAC-Bayes in SMGI.}
In classical PAC-Bayes, the prior regulates hypothesis complexity under a fixed interface.
In SMGI, the goal is stronger: capacity control must remain meaningful under admissible interface evolution
(representation changes, regime switching, and memory-mediated dynamics). Accordingly, we treat the prior not merely
as a regularizer over parameters, but as a \emph{structural prior} that constrains admissible configurations of
the learning interface itself.

\paragraph{Typed objects and what is randomized.}
Fix a task/environment instance $\tau$ (or an induced environment distribution $D_\tau$ over interactions $z\in\mathcal Z$).
Let $\mathcal H(\theta)$ denote the effective hypothesis/policy class induced by the current structural interface,
and let $Q_\tau$ denote a posterior distribution over hypotheses (or hypothesis parameters) within $\mathcal H(\theta)$.
We write the task risk in the canonical form
\[
R_\tau(h) \;=\; \mathbb E_{z\sim D_\tau}\big[\ell(h,z)\big],
\]
and the posterior (Gibbs) risk
\[
R_\tau(Q_\tau) \;:=\; \mathbb E_{h\sim Q_\tau}\big[R_\tau(h)\big].
\]

\paragraph{Classical PAC-Bayes bound (baseline form).}
For a prior $\Pi$ over hypotheses and a posterior $Q_\tau$, a standard PAC-Bayes inequality yields, with probability at least $1-\delta$,
\begin{equation}
\label{eq:pacbayes-basic}
R_\tau(Q_\tau)
\;\le\;
\widehat R_\tau(Q_\tau)
+
\sqrt{
\frac{
\mathrm{KL}(Q_\tau \,\|\, \Pi)
+
\log \frac{1}{\delta}
}{2n}
},
\end{equation}
where $\widehat R_\tau(Q_\tau)$ is the empirical risk computed from $n$ samples (or $n$ interaction steps under the sampling assumptions
of the chosen PAC-Bayes variant).
Equation~\eqref{eq:pacbayes-basic} is included here as the baseline form; the paper’s structural contribution is to specify what $\Pi$ controls
and how the bound is used under structural evolution.

\paragraph{Structural prior as meta-capacity control.}
In SMGI, $\Pi$ is not merely a prior over parameter vectors: it is a \emph{structural prior} that regulates the admissible interface degrees of freedom.
Concretely, $\Pi$ may be understood as a prior over a structural hypothesis space
\[
\Theta_{\mathrm{struct}}
\;\supseteq\;
(r,\mathcal H,\mathcal L),
\]
so that complexity penalties constrain (i) representational interface families $r$, (ii) induced hypothesis interfaces $\mathcal H$, and
(iii) multi-regime evaluator families $\mathcal L$ (or the evaluator parameters governing regime weights/selection).
This is the precise sense in which PAC-Bayes becomes a \emph{meta-capacity} control mechanism: it bounds effective complexity of the interface,
not only of a fixed predictor class.

\paragraph{Sequential/online PAC-Bayes (compatibility with filtrations).}
When the induced dynamics $T_\theta$ generates dependent trajectories (memory, tool-use, regime switching), i.i.d.\ sampling is not appropriate.
SMGI therefore relies on sequential/online PAC-Bayes formalisms that are adapted to filtrations and streaming data, so that the complexity term remains meaningful
under dependence and continual updates. In particular, online PAC-Bayes bounds extend PAC-Bayes reasoning to online learning settings with dependent data streams,
and recent recursive PAC-Bayes constructions enable sequential prior updates without information loss.
We cite these as modern instances of PAC-Bayes-compatible capacity control under sequential evolution \citep{haddouche2022onlinepacbayes,wu2024recursivepacbayes}.

\paragraph{Connection to the unified admissibility bundle.}
Within the unified guarantee bundle (Definition~\ref{def:admissibility-bundle}), the PAC-Bayes term $\mathrm{KL}(Q\|\Pi)$ and its sequential analogue
provide the capacity component (iii), while stability witnesses control drift and evaluator invariance prevents the objective from becoming a moving target.
Thus PAC-Bayes meta-level control is not a standalone guarantee: it is the capacity-control module of a structurally admissible coupled dynamics $(\theta,T_\theta)$.

%%%%%%%%%%%%%%
\subsection{Memory and Forgetting as Operational Compression}
\label{subsec:memory-forgetting-compression}

\paragraph{Memory as a typed state component.}
Let $\mathcal M(\theta)$ denote the typed memory state space induced by the structural interface $\theta$, and let $m_t\in\mathcal M(\theta)$.
In SMGI, memory is not a passive store: it is an explicit component of the coupled dynamics $(\theta,T_\theta)$ whose evolution must remain
well-formed under admissible transformations and certified stability constraints.

\paragraph{Memory update and functional forgetting as explicit operators.}
Memory state evolves via a typed update operator
\[
m_{t+1} = U(m_t, z_t, k_t),
\]
where $z_t\in\mathcal Z$ is the interaction stream and $k_t\in\{1,\dots,K\}$ indexes the active evaluative regime.
We view these as typed maps $U:\mathcal M(\theta)\times\mathcal Z\times\{1,\dots,K\}\to\mathcal M(\theta)$ and
$F:\mathcal M(\theta)\times\{1,\dots,K\}\times\mathcal C\to\mathcal M(\theta)$ for a context space $\mathcal C$.
We introduce \emph{functional forgetting} as an explicit, context- and regime-conditioned operator
\[
\tilde{m}_{t+1} = F(m_{t+1}; k_t, c_t),
\]
where $c_t\in\mathcal C$ is a context descriptor (e.g., security regime, deployment mode, audit mode). This makes forgetting a first-class structural map rather than an implementation side effect.

Verification-oriented perspectives on machine unlearning clarify how forgetting can be made auditable via behavioral or parametric evidence \citep{xue2025unlearningverification,llmforgetting2024}.

\paragraph{Forgetting as constrained compression (MDL viewpoint).}
We define forgetting as constrained compression by selecting $F$ to trade off task performance (under the active evaluative regime) against a complexity penalty:
\[
F^* = \arg\min_F 
\Big(
\text{Loss}_{k_t}(a_t \mid \tilde{m}_{t+1})
+
\lambda \, \text{Comp}(\tilde{m}_{t+1})
\Big).
\]
This connects directly to MDL and Kolmogorov complexity: compression implements an operational simplicity bias and controls the effective capacity
of the memory-bearing system \citep{li2008kolmogorov}.

\paragraph{SMGI alignment: capacity, stability, and evaluative invariance.}
Within SMGI, forgetting is admissible only if it simultaneously supports the three obligations that interact with memory:
(i) \emph{capacity control} by reducing the effective complexity of memory-dependent hypotheses via $\text{Comp}(\tilde m_{t+1})$;
(ii) \emph{stability} by preventing uncontrolled drift or interference in the coupled dynamics (memory is a feedback state, so forgetting acts as a stabilizing operator);
and (iii) \emph{evaluative invariance} by preserving any protected evaluative core under regime switching (forgetting must not ``solve'' constraints by erasing the information required to enforce them).
A sufficient stability-oriented design condition is to require $F$ to be non-expansive (or contractive) under a chosen stability witness/metric on $\mathcal M(\theta)$, so that memory updates cannot amplify drift across regimes.
This is precisely why forgetting is modeled as an explicit operator $F(\cdot; k_t,c_t)$ rather than implicit parameter decay.

\paragraph{Why ``forgetting'' is a modern control problem (2024--2026).}
Recent continual-learning work treats selective forgetting and memory refresh as a controllable mechanism rather than merely a failure mode, including settings where test-time data can refresh or reshape memory without labels \citep{singh2024controllingforgetting}.
In parallel, the growing literature on \emph{digital forgetting} and \emph{machine unlearning} in large models emphasizes that forgetting must often be auditable and verifiable, not only heuristic \citep{llmforgetting2024,xue2025unlearningverification}.
Finally, in agentic deployments, persistent memory and retrieval stores introduce a concrete security attack surface (memory/RAG poisoning), making controlled forgetting part of robustness and controllability rather than optional optimization \citep{jing2026memorypoisoning}.

\paragraph{Interpretation.}
Forgetting is not erasure; it is selective operational reduction.
In SMGI terms, $F$ is an admissible structural operator whose role is to compress memory while preserving certified invariants,
thereby enabling long-horizon stability and capacity control under regime shifts.

%%%%%%%%%%%%%%
\subsection{Program Prior Option (Typed Description-Length Prior)}
\label{subsec:program-prior}

\paragraph{A typed description language for structural objects.}
Fix a typed description language (or grammar) $\mathcal G_\Theta$ whose well-formed strings encode admissible
meta-models $\theta=(r,\mathcal H,\Pi,\mathcal L,\mathcal E,\mathcal M)$ together with declared typing information.
Let $\mathrm{code}(\theta)\in\{0,1\}^*$ be a chosen encoding and define the (computable) description length
$|\theta| := |\mathrm{code}(\theta)|$.

\paragraph{Structural description-length prior.}
We define the optional \emph{structural} prior
\begin{equation}
\label{eq:structural-program-prior}
\Pi(\theta)\ \propto\ 2^{-|\theta|}\qquad\text{equivalently}\qquad \Pi(\theta)\ \propto\ \exp(-\lambda|\theta|),
\end{equation}
for some $\lambda>0$ (absorbing constant factors).
This is an MDL-style proxy for algorithmic simplicity: it does \emph{not} assume access to Kolmogorov complexity,
but it produces a falsifiable pressure toward structurally parsimonious interfaces.

\paragraph{What this prior controls (and what it does not).}
The prior \eqref{eq:structural-program-prior} regularizes \emph{interface complexity}:
it penalizes representational degrees of freedom ($r$), hypothesis/interface choices ($\mathcal H$),
memory/operator interfaces ($\mathcal M$), and regime structure ($\mathcal L$) when these are treated as selectable components.
It does \emph{not} impose a normative hierarchy between evaluators: protected evaluative invariance is handled separately
by obligation (iv) and its certificate-gated update rules.

\paragraph{PAC-Bayes linkage as an explicit inequality.}
Let $Q$ be any posterior over structural objects $\theta$ supported on well-typed models.
Then, up to an additive normalizing constant $c_\Pi$, the KL term expands as:
\begin{equation}
\label{eq:kl-length}
\KL(Q\|\Pi) \;=\; \mathbb E_{\theta\sim Q}[\,|\theta|\,]\cdot\ln 2\;+\;c_\Pi\;-\;H(Q),
\end{equation}
where $H(Q)$ is the Shannon entropy of $Q$ (in nats).
Thus, PAC-Bayes complexity control directly upper-bounds expected structural description length,
making obligation (iii) (bounded capacity) a checkable meta-level regularization statement.

\paragraph{Interpretation.}
This option can be switched off by restricting $\Pi$ to a fixed structural family; when enabled, it provides a
transparent and implementation-agnostic capacity module for structural evolution.

\subsection{Theorem: Structural Strict Inclusion}
\label{subsec:strict-inclusion}

\paragraph{Equivalence notion (evaluation-relative).}
To avoid vacuity, we fix an evaluation-relative notion of equivalence: two learning systems are said to be
\emph{core-equivalent} if, for every admissible task/regime instance (or every $\tau\in\mathcal T$ in scope),
they induce the same risk ordering on the protected evaluative core (cf.\ obligation (iv)).
This prevents ``reductions'' that merely redefine what counts as risk.

\paragraph{Classical learners as degenerate SMGI instances.}
Let $\mathcal A_{\mathrm{classical}}$ denote learning systems with a fixed interface $(r,\mathcal H,\ell)$ and no explicit regime structure
(i.e., $K=1$), and let $\mathcal A_{\mathrm{SG}}$ denote systems represented by the SMGI meta-structure $\theta$ with typed operators.
Every classical learner embeds into SMGI by the special-case choice $K=1$, $\sigma(\cdot\mid c_t,s_t)\equiv\delta_{1}$,$F\equiv\mathrm{Id}$, and by restricting $\mathcal T$ to identity transformations (or to the same task family used classically).

\begin{theorem}[Strict structural inclusion]
\label{thm:strict-structural-inclusion}
Under the above embedding, $\mathcal A_{\mathrm{classical}}\subseteq \mathcal A_{\mathrm{SG}}$.
Moreover, the inclusion is strict: there exist SMGI-admissible configurations with $K>1$ and typed regime switching
that are not core-equivalent to any classical single-regime system.
\end{theorem}

\begin{proof}
\emph{Inclusion.} Let $S\in\mathcal A_{\mathrm{classical}}$ be any fixed-interface learner with tuple $(r,\mathcal H,\ell)$ and no regime structure. Embed it into SMGI by choosing $K=1$ (so $\mathcal L=\{\ell\}$), taking the regime selector $\sigma$ to be the constant Dirac distribution on the unique regime, and taking the forgetting operator (if present) to be the identity. Restrict the admissible transformation class $\mathcal T$ to identities (or to the same classical task family). Under this embedding, the induced semantics coincides with the classical learner, hence $\mathcal A_{\mathrm{classical}}\subseteq\mathcal A_{\mathrm{SG}}$.

\emph{Strictness.} Consider an SMGI configuration with $K=2$ regimes and two hypotheses/policies $h^a,h^b$ such that
$R_{1}(h^a)<R_{1}(h^b)$ but $R_{2}(h^a)>R_{2}(h^b)$, where $R_k$ denotes risk under evaluator $\ell_k$.
Assume, for contradiction, that there exists a single-regime classical system with evaluator $\ell^\star$ that is core-equivalent to this $K=2$ configuration in the sense of the paper (i.e., it induces the same protected risk ordering on the core for all admissible regimes). Then $\ell^\star$ must rank $h^a$ better than $h^b$ (to match regime 1) and simultaneously rank $h^b$ better than $h^a$ (to match regime 2), which is impossible for a single fixed ordering induced by one evaluator. Hence no $K=1$ classical system can be core-equivalent to this $K=2$ SMGI instance, proving strictness.
\end{proof}

%%%%%%%%%%%%%%%%%%%%%%%%%%%%%%%%%%%%%%%%%%%%%%%%%%%%%%%%%%%%%%%%%%%%%%%%
\section{Formal Foundations}
\label{sec:formal-foundations}
%%%%%%%%%%%%%%%%%%%%%%%%%%%%%%%%%%%%%%%%%%%%%%%%%%%%%%%%%%%%%%%%%%%%%%%%

This section fixes the minimal mathematical substrate used by the main results.
It is intentionally compact: the goal is not to re-teach learning theory, but to
make every later claim well-typed, checkable, and non-programmatic.

\subsection{Measurable spaces, stochastic kernels, and trajectories}
\label{sec:ff-kernels}

Let $(\mathcal Z,\mathcal A)$ be the measurable interaction space (tokens, observations, transitions).
Let $(S_\theta,\mathcal S)$ be the measurable state space induced by $\theta$ (Sec.~\ref{sec:formalsetup}).

An induced (possibly stochastic) learning-and-interaction dynamics compatible with $\theta$ is a Markov kernel
\[
T_{\theta,\tau} : (S_\theta\times\mathcal Z,\ \mathcal S\otimes\mathcal A)\ \to\ (S_\theta,\ \mathcal S),
\]
where $\tau\in\mathcal T$ denotes an admissible task/interface transformation.
Trajectories $(s_t)_{t\ge 0}$ are generated by iterating $T_{\theta,\tau}$ on $(s_t,z_{t+1})$.

We write $\mathcal F_t=\sigma(s_0,z_1,\ldots,z_t)$ for the natural filtration.

\subsection{Metrics, non-expansiveness, and contraction}
\label{sec:ff-contraction}

Assume $(S_\theta,d_S)$ is a complete metric space and $(\mathcal M(\theta),d_M)$ is a complete metric subspace.
A map $G:(X,d)\to(X,d)$ is:
(i) \emph{non-expansive} if $d(Gx,Gy)\le d(x,y)$;
(ii) a \emph{contraction} if $d(Gx,Gy)\le \gamma d(x,y)$ for some $\gamma\in(0,1)$.

These notions are used as checkable sufficient conditions for invariance (closure) and stability.

\subsection{Lyapunov drift conditions for stateful learning}
\label{sec:ff-lyapunov}

A (Foster--Lyapunov) drift condition is specified by a measurable witness $V:S_\theta\to\mathbb R_+$ and constants
$\alpha\in(0,1]$, $\beta\ge 0$ such that for trajectories under $T_{\theta,\tau}$:
\[
\mathbb E\!\left[V(s_{t+1})\mid\mathcal F_t\right] \le (1-\alpha)\,V(s_t)+\beta.
\]
This implies uniform moment bounds on $V(s_t)$ and prevents uncontrolled drift under sequential adaptation.

\subsection{Sequential PAC-Bayes as a capacity module}
\label{sec:ff-seq-pacbayes}

\begin{assumption}[Sequential PAC-Bayes module (filtration-adapted)]
\label{ass:seq-pacbayes}
Fix an admissible transformation \(\tau\in\mathcal T\) and the induced filtration \(\mathcal F_t=\sigma(s_0,z_1,\ldots,z_t)\) along trajectories under \(T_{\theta,\tau}\).
Assume the instantaneous loss used in the bound is measurable, \(\mathcal F_t\)-adapted, and uniformly bounded (e.g., in \([0,1]\)).
Assume there exist prior/posterior processes \(P_t,Q_t\) over the effective hypothesis/update component that are \(\mathcal F_t\)-adapted and satisfy the technical conditions of a sequential/online PAC-Bayes inequality for dependent data streams (e.g., the online PAC-Bayes setting of \citet{haddouche2022onlinepacbayes} and/or recursive prior-update constructions such as \citet{wu2024recursivepacbayes}), yielding with probability at least \(1-\delta\) a bound of the form
\[
R_{\tau}(Q_n)\;\le\;\widehat R_{\tau,n}(Q_n)\;+\;\mathrm{Gen}_n(Q_n,P_n,\delta),
\]
where \(\mathrm{Gen}_n\) contains a KL-type complexity term (possibly cumulative/iterated in time) and is finite on admissible trajectories.
We use this assumption as a \emph{capacity module} for SMGI-(iii).
\end{assumption}

\section{Capacity Analysis: Structural Generalization and Meta-Complexity}
\label{sec:capacity}

This section isolates the \emph{meta-level} capacity question induced by structural evolution.
Classical capacity control (VC/Rademacher/SRM, PAC-Bayes, MDL) bounds generalization for learning dynamics that operate inside a \emph{fixed} interface.
In SMGI, admissible evolution may modify parts of the interface (representation, regime structure, and memory operators), so capacity control must apply to the \emph{structural configuration} itself.

\subsection{Structural capacity as interface complexity}
Fix a realization $(\theta,T_\theta)$ with $\theta=(r,\mathcal H,\Pi,\mathcal L,\mathcal E,\mathcal M)$ and an admissible invariant set $S_\theta^*$.
We define a structural complexity functional that decomposes across the typed components that may vary under admissible evolution:
\begin{equation}
\label{eq:structural-capacity}
\mathcal C(\theta)
\;:=\;
\mathcal C_r(r)
\;+\;
\mathcal C_{\mathcal H}\!\big(\mathcal H(\theta)\big)
\;+\;
\mathcal C_{\mathcal L}(\mathcal L)
\;+\;
\mathcal C_{\mathcal M}\!\big(\mathcal M(\theta)\big),
\end{equation}
where each term is instantiated by the chosen capacity formalism (e.g., VC/Rademacher for $\mathcal H(\theta)$, a switching/description cost for $\mathcal L$, and an MDL-style or algorithmic-complexity proxy for memory operators and state).

A minimal admissibility requirement for the capacity obligation (SMGI-(iii)) is that $\mathcal C(\theta)$ remain uniformly bounded on admissible configurations:
\begin{equation}
\label{eq:structural-capacity-bounded}
\sup_{s\in S_\theta^*}\ \mathcal C\!\big(\theta(s)\big)\;<\;\infty,
\end{equation}
where $\theta(s)$ denotes the effective structural configuration active at state $s$ (e.g., via state-conditioned regime selection and memory policies).

\noindent\emph{Computability note.}
When $\mathcal{C}_{\mathcal M}$ (or the MDL/program-length proxy underlying it) is interpreted in an algorithmic-information sense,
the exact Kolmogorov description length is in general non-computable; any future implementation must therefore rely on computable upper bounds
or practical MDL-style surrogates as \emph{strictly necessary} approximations.

\subsection{Description-length option (MDL / program-prior view)}
To make \eqref{eq:structural-capacity} checkable without committing to a single parametric notion,
one may adopt a typed description language for structural objects and use a \emph{computable} description-length proxy $|\theta|$
(e.g., an MDL codelength, a prefix-free code length, or any certified upper bound on a description complexity).
This yields the optional \emph{structural} simplicity principle (cf.\ Sec.~\ref{subsec:program-prior}):
\begin{equation}
\label{eq:structural-mdl}
\theta^*
\;=\;
\arg\min_{\theta}
\Big[
\widehat R(\theta)
\;+\;
\lambda\,|\theta|
\Big],
\end{equation}
where $|\theta|$ is a computable surrogate (not the exact Kolmogorov complexity, which is non-computable in general).
This is compatible with PAC-Bayes via the standard KL--code-length linkage (e.g., Eq.~\eqref{eq:kl-length}).
In this reading, regime plurality and memory operators increase expressivity, but only those expansions that remain complexity-controlled are admissible under SMGI-(iii).

\subsection{Link to Structural Risk and PAC-Bayes Guarantees}
\label{sec:capacity-bridge}

The capacity analysis developed in this section relies on the structural risk control and PAC-Bayesian meta-level guarantees established earlier in Section~\ref{sec:unifying-structural-guarantees} and the subsequent subsections.

We therefore do not re-derive these results here. Instead, we analyze their implications in terms of meta-complexity: namely, how maintaining bounded structural risk and valid PAC-Bayes control constrains the growth of representational, policy, and memory capacity under structural expansion of environment families.

%%%%%%%%%%%%%%%%%%%%%%%%%%%%%%%%%%%%%%%%%%%%%%%%%%%%%%%%%%%%%%%%%%%%%%%%
\section{Positioning within Existing Theoretical Frameworks}
\label{sec:positioning}
%%%%%%%%%%%%%%%%%%%%%%%%%%%%%%%%%%%%%%%%%%%%%%%%%%%%%%%%%%%%%%%%%%%%%%%%

The present structural meta-model is not proposed in isolation.
It must be situated relative to existing formal, architectural,
epistemic, and large-scale empirical approaches to general intelligence.

Rather than comparing systems at the level of benchmark performance,
we compare research programs at the level of the structural meta-models
they induce and the semantic constraints they impose.

\paragraph{Programs as restrictions of the structural--dynamical universe.}
Throughout this section, a research program $\mathfrak P$ is interpreted as inducing
a restricted class of well-typed structural meta-models
\[
\Theta(\mathfrak P)\subseteq \Theta,
\]
together with a restricted (set-valued) realization semantics
\[
\mathsf{Sem}_{\mathfrak P}:\Theta(\mathfrak P)\rightrightarrows \mathfrak D,
\qquad
\mathsf{Sem}_{\mathfrak P}(\theta)\subseteq \mathsf{Sem}(\theta)\subseteq \mathfrak D,
\]
where $\mathfrak D$ denotes the global class of admissible induced dynamics, and
$\mathsf{Sem}$ is the general realization semantics defined in Section~\ref{sec:formalsetup}.
Equivalently, $\mathfrak P$ induces the restricted semantic graph
\[
\mathrm{Graph}(\mathsf{Sem}_{\mathfrak P})
\;:=\;
\{(\theta,T)\in \Theta(\mathfrak P)\times \mathfrak D:\; T\in \mathsf{Sem}_{\mathfrak P}(\theta)\},
\]
which satisfies
\[
\mathrm{Graph}(\mathsf{Sem}_{\mathfrak P})\subseteq \mathrm{Graph}(\mathsf{Sem}).
\]
Thus, program-level realizations are treated as restrictions of the general
structural--dynamical semantics.

\paragraph{Methodological consequence: a non-rhetorical comparison criterion.}
We never claim that a research program $\mathfrak P$ ``satisfies SMGI''.
Instead, comparative analysis reduces to a single testable criterion:
whether $\mathfrak P$ admits at least one realization satisfying the SMGI obligations, i.e.,
\[
\mathrm{Graph}(\mathsf{Sem}_{\mathfrak P}) \cap \mathrm{SMGI} \neq \varnothing.
\]
This separation prevents conflating capability, representation, and stability,
and forces dynamical and evaluative guarantees to be stated explicitly rather than
inferred from scale or empirical performance. The framework thus converts
heterogeneous AGI paradigms into a unified admissibility problem over the shared
structural space $\Theta$, replacing narrative comparison with a mathematically
checkable inclusion/intersection condition.

\subsection{Why Positioning is Needed: Fragmented Axes and No Joint Object}

The contemporary AGI landscape is characterized by a proliferation of
highly successful but axis-specific theories.
Statistical learning theory formalizes generalization under sampling assumptions;
reinforcement learning formalizes sequential decision-making;
algorithmic information theory formalizes universal prediction;
meta-learning addresses cross-task adaptation;
causal and epistemic logics formalize structured reasoning;
large-scale neural architectures unify heterogeneous modalities.

However, these frameworks do not specify a single mathematical object
that jointly coordinates representation, hypothesis evolution,
update operators, memory stratification, and evaluation constraints.
Each axis is formalized in isolation.
What remains under-specified is the level at which these axes must be
coordinated inside a single persistent system that remains well-formed
under admissible transformations.

This motivates an explicit positioning:
not to rank programs by ``AGI-ness,''
but to articulate the structural obligations required
for their mechanisms to coexist within a unified meta-model.

\subsection{From Fragmented Theories to a Meta-Model: The SMGI View}

The SMGI framework provides a unified language to analyze the structural restrictions of current paradigms. Positioning a research program $\mathfrak P$ within this axis consists of identifying which components of the structural interface are held fixed (degenerate cases, i.e., admissible transformations act as the identity on those components) and which are allowed to evolve under certified updates.

Formally, $\mathfrak P$ induces a restricted structural subclass $\Theta(\mathfrak P)\subseteq\Theta$ together with a restricted realization semantics $\mathsf{Sem}_{\mathfrak P}$, hence a restricted semantic graph $\mathrm{Graph}(\mathsf{Sem}_{\mathfrak P})\subseteq \Theta(\mathfrak P)\times \mathfrak D$. Comparisons are then formulated parsimoniously by testing the non-emptiness of the intersection between program realizations and the SMGI obligations:
\begin{equation}
\mathrm{Graph}(\mathsf{Sem}_{\mathfrak P}) \cap \mathrm{SMGI} \neq \varnothing.
\label{eq:inclusion-check}
\end{equation}

As will be detailed in Section~\ref{sec:positioning}, classical ERM is recovered as a special case where the representational interface and evaluator are fixed, while standard RL typically fixes the evaluative structure to a single reward functional; SMGI highlights that stability and invariance properties under task/interface transformations are not guaranteed by these formalisms unless made explicit as obligations on the coupled dynamics.

%%%%%%%%%%%%%%%%%%%%%%%%%%%%%%%%%%%%%%%%%%%%%%%%%%%%%%%%%%%%%%%%%%%%%%%%
% Comparative Table
%%%%%%%%%%%%%%%%%%%%%%%%%%%%%%%%%%%%%%%%%%%%%%%%%%%%%%%%%%%%%%%%%%%%%%%%

\subsection{Comparative Structural Overview}
\label{sec:comparative-frameworks}

To clarify the structural contribution of the present work,
we summarize how major theoretical frameworks relate to key
properties of a unified AGI meta-model in Table~\ref{tab:framework-positioning}.
We deliberately compare heterogeneous paradigms (theories, architectures, and cognitive hypotheses) only along a single axis: whether the paradigm makes the SMGI-relevant structural objects and obligations explicit (as first-class components and constraints), rather than treating them as implicit assumptions or external modules.
\begin{table}[ht]
\centering
\small
\begin{tabular}{|p{3.2cm}|p{3.2cm}|p{3.5cm}|p{4.5cm}|}
\hline
\textbf{Framework} & \textbf{Core Structural Assumption} & \textbf{Strengths} & \textbf{Structural Limitations (relative to unified AGI meta-model)} \\
\hline

Statistical Learning Theory (PAC / VC / Rademacher)\citep{vapnik1998,bousquet2002}
&
Fixed hypothesis class under i.i.d. sampling
&
Finite-sample guarantees; capacity control; generalization theory
&
Not first-class objects in the canonical i.i.d.\ SLT core: sequential composition of tasks/regimes, persistent memory dynamics, and explicit invariance obligations across typed transformations (these require an extension of the formal object beyond a fixed $(\mathcal H,\ell)$ under i.i.d.\ sampling).
\\
\hline

Solomonoff Induction / AIXI\citep{solomonoff1964,hutter2005}
&
Universal Bayesian mixture over programs
&
Asymptotic optimality under computability assumptions; universal prediction framework
&
Non-computable in its standard form; treats memory implicitly as part of the program/state without an explicit write/consolidate/forget operator interface; does not provide operator-level stability/admissibility obligations for certified structural evolution.
\\
\hline

Reinforcement Learning\citep{suttonbarto2018}
&
Markov decision processes with reward maximization
&
Sequential decision guarantees; regret bounds; convergence in stationary environments
&
Canonical RL formalizes sequential decision-making given fixed state/action semantics and a reward functional; even when enriched with partial observability (POMDPs/belief-state planning) and temporal abstraction (options/hierarchical RL), the core object typically keeps the interface and evaluative regime fixed, so cross-task/interface transformations, persistent multi-system memory operators (write/consolidate/forget), and invariance/certification obligations must be introduced explicitly beyond the base MDP/POMDP formalism.
%Canonical RL formalizes sequential decision-making given fixed state/action semantics and a reward functional; even with POMDP/belief-state and options/HRL extensions, cross-task/interface transformations, persistent memory operators, and invariance/certification obligations are not first-class in the standard formal object and must be added explicitly.
\\
\hline

(standard) Meta-learning (Learning-to-Learn)\citep{finn2017maml,nichol2018reptile}
&
Task-distribution adaptation through shared meta-parameters
&
Rapid adaptation across related tasks; few-shot generalization
&
Typically optimizes shared meta-parameters over a task distribution; persistent memory decomposition (with explicit write/consolidate/forget operators) and invariance obligations under heterogeneous typed transformations are not part of the canonical meta-learning object unless added explicitly.
\\
\hline

Transformer-based Architectures\citep{vaswani2017attention}
&
Fixed parametric attention-based function class
&
High expressive capacity; contextual reasoning; scalable sequence modeling
&
Standard transformer models fix a parametric sequence-to-sequence class; explicit persistent memory operators and certified stability/invariance under continual cross-task adaptation are external to the core architecture unless specified as part of the training/evaluation/control interface.
\\
\hline

Retrieval-Augmented Systems (RAG)\citep{lewis2020rag}
&
Parametric model coupled with external retrieval mechanism
&
Improved factual recall; scalable knowledge access; modular memory augmentation
&
Adds retrieval as an external interface; the write/consolidate/forget semantics of long-term memory and the stability/invariance of behavior under retrieval-policy drift are typically not formalized as certified obligations in the standard RAG formulation.
\\
\hline

Bayesian Brain Hypothesis\citep{knill2004bayesian,friston2010}
&
Probabilistic inference as neural computation
&
Normative probabilistic account of perception and cognition
&
No formal operator-based memory decomposition; lacks statistical generalization and dynamical stability guarantees in artificial systems
\\
\hline

Neuroscientific / Cognitive Multi-System Memory Models\citep{mcclelland1995complementary,tulving1985episodic}
&
Empirically motivated memory subsystems (episodic, semantic, procedural, etc.)
&
Biological plausibility; functional differentiation; empirical grounding
&
Primarily descriptive; no formal learning-theoretic guarantees or unified dynamical stability framework
\\
\hline

\textbf{Proposed Structured AGI Meta-Model}\citep{omgmof2019,omguml17}
&
Operator-based structural decomposition integrated in the learning tuple $\theta=(r,\mathcal H,\Pi,\mathcal L,\mathcal E,\mathcal M)$
&
Unified structural specification of representation, memory, evaluation, and admissible transformation classes
&
Stability, statistical control, and evaluative invariance are formulated at the level of realizations $(\theta, T_\theta)$; empirical validation at scale remains ongoing
\\
\hline

\end{tabular}

\caption{Comparative structural positioning of major theoretical frameworks relevant to AGI.
Each framework formalizes a principled axis of learning or intelligence,
but typically leaves implicit the joint structural constraints required
for a persistent, transformable learning system.}
\label{tab:framework-positioning}

\end{table}

\paragraph{Reading guide.}
Table~\ref{tab:framework-positioning} compares formalisms by whether they make SMGI-relevant objects explicit; Table~\ref{tab:smgi-program-map} then summarizes which SMGI obligations are structurally explicit in major AGI program classes. In both cases, the goal is diagnostic (explicitness), not a claim of satisfaction.

%%%%%%%%%%%%%%%%%%%%%%%%%%%%%%%%%%%%%%%%%%%%%%%%%%%%%%%%%%%%%%%%%%%%%%%%
% Engineering Meta-Model Table
%%%%%%%%%%%%%%%%%%%%%%%%%%%%%%%%%%%%%%%%%%%%%%%%%%%%%%%%%%%%%%%%%%%%%%%%

\subsection{Methodological Lineage of Meta-Models}
\label{sec:methodological-lineage}

A methodological precedent for meta-structural formalization comes from the formal
sciences of software and systems design.
In model-driven engineering, a \emph{meta-model} specifies (i) the class of \emph{well-formed}
models (via a conformance relation), (ii) the admissible \emph{transformations} between them,
and (iii) the \emph{invariants} and constraints that must be preserved under such transformations
\citep{omgmof2019,omguml17}.
This is not an analogy for style: it is a mature methodology for making \emph{admissibility},
\emph{closure under transformation}, and \emph{invariant constraints} explicit at the level of
system definition, rather than leaving them implicit in implementations.

Concretely, MOF makes explicit the meta-level at which a modeling language is defined,
and UML specifies a structured language of objects, relations, inheritance, and behavioral
interfaces; in both cases, correctness is ensured by constraints that apply to \emph{families}
of models rather than only to individual instances \citep{omgmof2019,omguml17}.
Object-oriented foundations similarly show that inheritance, refinement, and subtyping are
formal constraints governing admissible transformations while preserving type-level invariants
\citep{meyer1988oosc,booch1994ooad}.

In the present work, SMGI adopts the same methodological stance in a strictly formal sense. The claim is methodological: we import the discipline of explicit well-formedness, admissible transformations, and invariant constraints, without assuming that learning systems inherit the determinism or completeness typical of engineered meta-models.
A structural meta-model $\theta\in\Theta$ defines an admissible class of learning-system
descriptions, while its semantics $\mathsf{Sem}(\theta)$ fixes which induced learning
dynamics and task/regime transformations are meaningful.
In this perspective, ``well-formedness'' corresponds to semantic admissibility
(i.e., membership in $\mathrm{Graph}(\mathsf{Sem})$), and ``model transformation''
corresponds to task/environment operators acting on the semantic side while respecting
the invariants required by SMGI.

\begin{table}[ht]
\centering
\small
\begin{tabular}{|l|l|l|}
\hline
\textbf{Structural Concept} & \textbf{Engineering Principle} & \textbf{AGI Meta-Model Analogue} \\
\hline
Inheritance / refinement & Structural inclusion & Submodel / restriction theorems \\
\hline
Interface contracts & Invariant constraints & Normative/evaluative invariant core \\
\hline
Meta-model & Defines admissible models & $\theta=(r,\mathcal H,\Pi,\mathcal L,\mathcal E,\mathcal M)$ \\
\hline
Model conformance & Well-formedness predicate & Semantic admissibility: $(\theta,T)\in \mathrm{Graph}(\mathsf{Sem})$ \\
\hline
Model transformation & Closure under refinement & Task/environment transformation operators \\
\hline
Well-formedness rules & Constraint enforcement & Stability and capacity admissibility constraints \\
\hline
\end{tabular}
\caption{Methodological precedent from engineering meta-models: explicit admissibility, transformation, and invariants.}
\label{tab:eng-smgi-analogy}
\end{table}

SMGI therefore does not propose a single architecture.
It makes explicit a meta-level \emph{admissibility space} together with invariants that must be preserved
under task and regime transformations, so that heterogeneous mechanisms can be composed inside a single
persistent learning system.

\paragraph{Stratification Levels and Meta-Modeling Hierarchy.}

Model-driven engineering further distinguishes multiple levels of abstraction
(M0–M3) in which instances, models, meta-models, and meta-meta-models
are formally separated \citep{omgmof2019}.
A similar stratification clarifies the present framework.

Concrete learning trajectories and executions correspond to instance-level realizations (analogous to M0).
Specific implemented agent architectures correspond to model-level descriptions (M1).
The structural tuple $\theta\in\Theta$ plays the role of a meta-model (M2),
defining the admissible structural language of learning systems.
Finally, the theory specifying $\Theta$ together with its realization semantics
$\mathsf{Sem}$ plays a meta-meta role (analogous to M3), defining the admissibility
space in which structural meta-models themselves are well-formed.

This stratification is not an analogy for exposition; it ensures that structural
definitions, admissibility constraints, and transformation laws are specified
at the appropriate abstraction level rather than conflated with implementation details.

%%%%%%%%%%%%%%%%%%%%%%%%%%%%%%%%%%%%%%%%%%%%%%%%%%%%%%%%%%%%%%%%%%%%%%%%
\subsection{Meta-Logical and Symbolic Foundations for Structural Admissibility}
\label{sec:meta-logical-foundations}
%%%%%%%%%%%%%%%%%%%%%%%%%%%%%%%%%%%%%%%%%%%%%%%%%%%%%%%%%%%%%%%%%%%%%%%%

The methodological stance of SMGI (and its M0--M3 stratification) is not merely expository.
It is motivated by constraints that arise when a persistent learning system must
(i) revise parts of its own learning interface, while (ii) preserving certified invariants under
task and regime transformations. We summarize four complementary foundations that justify treating
typed structure, admissible transformations, and invariant-preserving updates as first-class objects.

\paragraph{(I) Meta-logical necessity: typed ascent for specification and verification.}
A classical limitation of self-referential systems is that semantic notions such as truth, consistency,
or correctness of an object-language cannot, in general, be fully expressed and validated within that
same language without collapsing levels \citep{tarski1936concept}.
In learning systems, the risk is not a literal paradox but an \emph{untyped collapse}:
attempting to revise the evaluator $\ell$ or the admissible update rules using only object-level operators
can make it unclear \emph{which} obligations are being preserved, \emph{relative to what} semantics,
and \emph{under which} admissibility constraints.
SMGI addresses this by treating the structural interface $\theta$ as a typed meta-description for the induced dynamics $T_\theta$,
so that interface revisions are represented as elements of a typed transformation class and subjected to explicit certificate/monitor
conditions (cf.\ Sec.~\ref{sec:reflective-equilibrium}).

\paragraph{(II) Systemic necessity: hierarchical decomposability and evolvability.}
Complex adaptive systems evolve efficiently when organized around stable intermediate subsystems and nearly-decomposable interfaces
\citep{simon1962architecture}.
In SMGI, the components $(r,\mathcal H,\Pi,\mathcal L,\mathcal E,\mathcal M)$ play the role of typed sub-assemblies:
they specify \emph{what may vary}, \emph{what must remain invariant}, and \emph{what must be checked} during structural evolution.
This targets the failure mode of non-stratified approaches where behavioral adaptation is conflated with admissible change of the interface itself:
SMGI makes evolvability a property of certified transformations on $(\theta,T_\theta)$ rather than an informal expectation from scale.

\paragraph{(III) Structural subsumption: degeneracy as restricting typed degrees of freedom.}
The positioning criterion of Sec.~\ref{sec:positioning} can be sharpened by interpreting many classical paradigms as \emph{restricted}
instances obtained by freezing (or restricting) specific structural degrees of freedom.
A compact formal vocabulary for such restrictions is provided by category-theoretic \emph{forgetful} mappings that preserve only part
of the available structure \citep{maclane1971categories}.
We emphasize that this is an \emph{optional} formal viewpoint: it is used to express, concisely, that a program class corresponds to a
restricted realization semantics that keeps selected components fixed. For example, canonical RL can be viewed as restricting the evaluator
to a scalar reward and treating representational plasticity as exogenous:
\begin{equation}
    U(\theta_{\mathrm{SMGI}}) = (h, \mathcal L_{\mathrm{fixed}}) \in \mathbf{P}_{\mathrm{RL}},
\qquad \mathcal L_{\mathrm{fixed}}=\{\ell_{\mathrm{fixed}}\}.
\end{equation}
In this sense, ``degenerate'' means ``structurally constrained'' rather than inferior: it expresses that SMGI measures expressiveness at the level
of admissible structural transformations, not only at the level of parameterized function classes.

\paragraph{(IV) Symbolic and computational grounding: typed hypotheses, explicit memory, and contracts.}
The explicit typing of $\theta$ provides clean interfaces to symbolic and neuro-symbolic instantiations, without committing SMGI to any one substrate.
First, $\mathcal H(\theta)$ can be instantiated as a space of relational or programmatic hypotheses, connecting structural generalization to ILP-style
\emph{structural synthesis} \citep{muggleton1991inductive,cropper2022ilp,dumancic2024learning}.
Second, $\mathcal M(\theta)$ admits realizations as queryable structured memory systems, including knowledge-graph or description-logic stores
\citep{baader2003description,pan2024llmkg}, which makes write/consolidate/forget operators explicit objects rather than implicit side effects.
Third, the MDE/PLT stance treats $\theta$ as a specification layer and $T_\theta$ as an implementation layer:
admissible updates correspond to contract-respecting transformations, supported by model conformance and type-safety principles
\citep{omgmof2019,meyer1988oosc,pierce2002types}. This aligns with verification-guided agent control and guardrail enforcement
\citep{corsi2024verificationguided,dong2024guardrails}, while internalizing the relevant objects (evaluators, invariants, admissible transformations)
directly into the typed meta-model.

\paragraph{Synthesis: structural continuity under interface change.}
Finally, certified evolution must control the transition itself: interface updates $\theta\to\theta'$ should preserve a minimal certified core during the
transition, to avoid intervals where monitors or invariants are undefined. This motivates modeling interface evolution with explicit continuity constraints,
in the spirit of correctness conditions studied for dynamic software updating \citep{hicks2005dynamic}.
In SMGI terms, such constraints are expressed as admissibility conditions on typed meta-transformations, ensuring that certified invariants remain meaningful
along the evolution of the coupled dynamics $(\theta,T_\theta)$ and enabling subsequent stability analysis.

%%%%%%%%%%%%%%%%%%%%%%%%%%%%%%%%%%%%%%%%%%%%%%%%%%%%%%%%%%%%%%%%%%%%%%%%

\subsection{SMGI Instantiation Map}
\label{sec:smgi-instantiation-map}

To relate SMGI to contemporary AGI programs without over-claim,
we adopt a conservative structural inclusion schema.

A research program $\mathfrak P$ determines a class of admissible
implementations and therefore induces a subset
\[
\Theta(\mathfrak P) \subseteq \Theta
\]
through an interpretation map
\[
\Phi_{\mathfrak P} : \mathrm{Impl}(\mathfrak P) \longrightarrow \Theta,
\]
where $\mathrm{Impl}(\mathfrak P)$ denotes the class of concrete instantiations
(architectures together with training and evaluation specifications)
advocated by $\mathfrak P$.
Importantly, $\Phi_{\mathfrak P}$ is not assumed canonical: different implementations advocated by the same program may induce different $\theta$ (e.g., via different memory operators, evaluators, or tool interfaces), hence different admissibility properties.
The map identifies representational operators, update rules,
memory mechanisms, and evaluative structures as components of
\[
\theta=(r,\mathcal H,\Pi,\mathcal L,\mathcal E,\mathcal M).
\]

\paragraph{Inclusion as a conditional structural statement.}

We do not assert that a program $\mathfrak P$ solves SMGI.
Instead, inclusion is formulated as a condition of possibility at the level of realizations:
\[
\mathrm{Graph}(\mathsf{Sem}_{\mathfrak P})
\cap
\mathrm{SMGI}
\neq \varnothing,
\]
that is, there exist structurally admissible realizations
$(\theta,T_\theta)$ compatible with the SMGI obligations.
Equivalently, we make explicit the additional obligations under which a particular
$\theta\in\Theta(\mathfrak P)$ together with some induced
$T_\theta\in\mathsf{Sem}(\theta)\subseteq\mathfrak D$
would qualify as an SMGI instance.

\paragraph{A diagnostic structural map.}

Table~\ref{tab:smgi-program-map} summarizes which SMGI requirements
are explicitly made formal within major AGI program classes.
Marked entries indicate that a requirement is structurally
made explicit in the program's formal narrative or mechanism class
(including, for (ii), stability properties stated at the level of induced learning dynamics).
This does not constitute a claim of completeness or satisfaction.

\begin{table}[t]
\centering
\small
\begin{tabular}{p{3.7cm}ccccp{5.0cm}}
\toprule
Program class & (i) Closure & (ii) Stability & (iii) Capacity & (iv) Eval.\ inv. & Canonical structural emphasis \\
\midrule
World Models / latent simulators \citep{ha2018worldmodels,hafner2023dreamerv3}
& \(\circ\) & \(\circ\) & \(\circ\) & \(\times\) &
Latent predictive dynamics \(m_{t+1}=F(m_t,a_t)\) supporting planning \\

JEPA / World-Model Hypothesis \citep{lecun2022path}
& \(\circ\) & \(\circ\) & \(\circ\) & \(\times\) &
Predictive latent representations and invariances for autonomy \\

System-2 / Consciousness prior \citep{bengio2017consciousness,bengio2019system2}
& \(\circ\) & \(\times\) & \(\circ\) & \(\circ\) &
Sparse causal abstractions; deliberative control and constraint-like structure \\

Gödel machines / self-improving systems \citep{schmidhuber2006godel,schmidhuber2010formal}
& \(\circ\) & \(\circ\) & \(\times\) & \(\circ\) &
Self-modification under proof-based improvement criteria \\

Generalist agents (Transformer agents) \citep{reed2022gato,silver2021reward}
& \(\circ\) & \(\times\) & \(\circ\) & \(\times\) &
Unified sequence modeling across modalities and action spaces \\

Causal modeling / Ladder of Causation \citep{pearl2009}
& \(\circ\) & \(\times\) & \(\circ\) & \(\circ\) &
Interventions/counterfactual structure constraining representation and evaluation \\

Value alignment under uncertainty \citep{russell2019human}
& \(\times\) & \(\times\) & \(\times\) & \(\circ\) &
Objectives as uncertain; normative constraints as central objects \\

Theory of Mind agents (social/epistemic modeling) \citep{mccarthyhayes1969,fagin1995reasoning,wooldridge2000reasoning}
& \(\circ\) & \(\times\) & \(\circ\) & \(\circ\) &
Explicit modeling of agents' beliefs, knowledge, and intentions within $\mathcal E$ and $r$ \\

\bottomrule
\end{tabular}
\caption{Program classes as structural emphases relative to the SMGI requirements.
Symbols indicate whether a requirement is made explicit within the program’s formal structure.
This does not constitute a claim that any concrete system satisfies SMGI.}
\label{tab:smgi-program-map}
\end{table}

The map should be interpreted conservatively.
Different programs make different structural dimensions explicit:
some emphasize latent predictive dynamics,
others abstraction and causal structure,
others self-modification or value uncertainty.
SMGI functions here as a unifying structural criterion
that makes explicit which additional obligations would be required
for these mechanisms to coexist within a single persistent learning system.

\paragraph{Formal interpretation of table annotations.}
For a program $\mathfrak P$ and requirement 
$k \in \{(i),(ii),(iii),(iv)\}$,
we define the explicitness predicate
\[
\mathrm{Exp}_{\mathfrak P}(k)
\]
to mean that requirement $k$ is represented as an explicit formal object,
constraint, admissibility condition, or structural invariant
within the program’s specification.

The symbols in Table~\ref{tab:smgi-program-map} are interpreted as:
\[
\circ \;\Longleftrightarrow\; \mathrm{Exp}_{\mathfrak P}(k),
\qquad
\times \;\Longleftrightarrow\; \neg \mathrm{Exp}_{\mathfrak P}(k).
\]

No global satisfaction claim is implied by this notation.
The predicate concerns structural explicitness only.

%%%%%%%%%%%%%%%%%%%%%%%%%%%%%%%%%%%%%%%%%%%%%%%%%%%%%%%%%%%%%%%%%%%%%%%%

\subsection{Structural Embeddings of Major AGI Paradigms}

We now refine the analysis via structured embeddings.

%%%%%%%%%%%%%%%%%%%%%%%%%%%%%%%%%%%%%%%%%%%%%%%%%%%%%%%%%%%%%%%%%%%%%%%%
\subsubsection{A. Predictive and Representational Paradigms}
%%%%%%%%%%%%%%%%%%%%%%%%%%%%%%%%%%%%%%%%%%%%%%%%%%%%%%%%%%%%%%%%%%%%%%%%

\paragraph{World Models and Latent Predictive Dynamics.}

World-model agents \citep{ha2018worldmodels,hafner2023dreamerv3,lecun2022path}
factor learning into representation, latent transition, and policy.
Under $\Phi_{\mathrm{WM}}$,
$r$ is instantiated as an encoder,
$\mathcal M$ contains predictive latent state,
$\Pi$ includes transition and policy updates.

\begin{lemma}[Conditional SMGI embedding of latent world-model agents]
\label{lem:wm-smgi}
Let $\theta\in\Theta(\mathrm{WM})$ be induced by a world-model program.
If admissible task transformations preserve a nonempty invariant manifold $S^*$,
the induced dynamics is stable on $S^*$,
capacity control holds for the hypothesis/update classes,
and regime changes preserve an invariant evaluative core,
then $(\theta,T_\theta)$ satisfies SMGI structural obligations.
\end{lemma}

\noindent\emph{Proof (idea).}
Closure corresponds to invariance of $S^*$;
stability to bounded evolution under $T_\theta$;
capacity control to uniform generalization bounds;
evaluative invariance to persistence of protected constraint subspaces.
A complete proof depends on the chosen stability and capacity formalism
(e.g., PAC-Bayes + stability).

%% PATCH B (inserted) ----------------------------------------------------------
\begin{theorem}[From witnesses to obligations (world-model instantiation)]\label{thm:wm-witnesses}
Assume a world-model induced $\theta\in\Theta(\mathrm{WM})$ admits
(i) a typed transformation class $\mathcal T$ preserving a nonempty invariant manifold $S^*$,
(ii) a concrete stability witness on $S^*$ for the induced dynamics $T_\theta$,
(iii) a concrete capacity witness for the joint hypothesis/update class as defined below,
and (iv) an invariant evaluative core $\Phi$ preserved under regime switching.
Then the four obligations (closure, stability, capacity, evaluative invariance) hold for $(\theta,T_\theta)$
in the minimal checkable sense of Definition~\ref{def:admissibility-bundle}.
\end{theorem}

\begin{proof}
(i) gives closure by invariance of $S^*$ under $\mathcal T$; (ii) is stability by definition of the witness on $S^*$;
(iii) is capacity by the declared bound on the complexity functional; (iv) is evaluative invariance by preservation of $\Phi$.
\end{proof}
%% ---------------------------------------------------------------------------

\paragraph{Making the assumptions checkable (capacity and stability).}
To avoid implicit assumptions, we instantiate the two nontrivial hypotheses in
Lemma~\ref{lem:wm-smgi} as explicit admissibility conditions.

\emph{Capacity control.}
Let $\mathcal H(\theta)$ denote the joint hypothesis class encoding the encoder, latent transition,
and policy components induced by $\theta$.
We say that \emph{capacity control holds} if there exists a complexity functional
$\mathcal C:\mathcal H(\theta)\to\mathbb R_+$ and a bound $B<\infty$ such that all admissible updates
preserve $\mathcal C(h)\le B$, and the chosen generalization framework
yields a uniform bound for the predictive/control losses along admissible trajectories
(e.g., VC/SRM-style control \citep{vapnik1998statistical} or PAC-Bayes control \citep{mcallester1999pacbayes,catoni2007pacbayes}).

\emph{Stability on the invariant manifold.}
Let $T_\theta$ be the induced learning-and-control dynamics (possibly stochastic) on the state space $S_\theta$,
and let $S^*\subseteq S_\theta$ be the invariant manifold postulated in Lemma~\ref{lem:wm-smgi}.
We say that \emph{stability holds on $S^*$} if there exists a metric (or Lyapunov-like) witness $V:S^*\to\mathbb R_+$
and constants $(\lambda,C)$ such that trajectories starting in $S^*$ remain bounded and do not amplify perturbations
under the admissible update operator (e.g., uniform stability viewpoints for learning dynamics \citep{bousquet2002,hardt2016}).
We emphasize that this is a property of $(\theta,T_\theta)$ (not $\theta$ alone), consistent with the SMGI layer separation.
\paragraph{Predictive Representation Programs (JEPA).}

JEPA-style approaches emphasize invariances in latent space.
They contribute strongly to $r$ and partial $\mathcal M$ structure.
SMGI inclusion requires that learned invariances align with
admissible transformation classes $\mathcal T$.

\begin{proposition}[Transformation-aligned invariance as a closure witness]
\label{prop:jepa-alignment}
Let $\mathcal T$ be the admissible class of task/environment transformations acting on $\mathcal E$,
and let $r$ be the JEPA-induced representation component.
Assume that there exists a mapping $\Gamma:\mathcal T\to \mathrm{End}(\mathcal R)$ such that for all $\tau\in\mathcal T$
and all admissible observations $x$,
\[
r(\tau\cdot x)\;=\;\Gamma(\tau)\big(r(x)\big),
\]
with $\Gamma(\tau)$ an automorphism (equivariance) or the identity (invariance) on the representation space.
Then the representation layer $r$ is compatible with SMGI closure requirement (i)
under the action of $\mathcal T$ (at the structural level), and any remaining failure of SMGI must arise from
dynamical stability (ii), capacity control (iii), or evaluative invariance (iv).
\end{proposition}

\noindent
This formulation connects invariance-first programs to standard group/equivariance constraints
when $\mathcal T$ admits a group-like structure \citep{cohenwelling2016groupconv},
but does not assume that such structure holds in general.
\paragraph{Abstraction and System-2 Constraints.}

Programs targeting abstraction and deliberation
\citep{bengio2017consciousness,bengio2019system2}
can be embedded as structural constraints on $(r,\mathcal M)$.
They primarily affect capacity control and evaluative invariance,
while global stability remains a property of the induced dynamics.

\paragraph{Memory stratification as the missing explicit layer.}
A recurrent gap in abstraction/System-2 proposals is that the abstraction layer is specified at the level of $r$
(or at the interface of $r$ and $\mathcal E$), while the persistence of abstractions across time, regimes, and tasks
is a property of the memory operator $\mathcal M$ and its interaction with the update operator $\Pi$.
In SMGI terms, this suggests making explicit a stratified memory decomposition
\[
\mathcal M \;=\; \mathcal M_{\mathrm{epis}} \oplus \mathcal M_{\mathrm{sem}} \oplus \mathcal M_{\mathrm{proc}},
\]
together with admissibility constraints limiting cross-stratum interference under updates (stability (ii))
and ensuring that abstraction-relevant information is retained under task transformations (closure (i)).
This perspective aligns naturally with continual-learning formalisms where stability and retention are stated explicitly
(e.g., consolidation/regularization and replay-style mechanisms) \citep{kirkpatrick2017ewc,parisi2019continual}.

\paragraph{Summary (SMGI positioning).}
Predictive and representational paradigms provide strong structural leverage on
closure (i) and capacity control (iii) when the representation map $r$ and hypothesis/update classes
are explicitly typed and complexity-controlled.
However, stability (ii) and evaluative invariance (iv) remain conditional properties of the induced dynamics
$(\theta,T_\theta)$ and its interaction with memory $\mathcal M$ and evaluation family $\mathcal L$ (or $\mathcal L$):
they typically require additional admissibility constraints (e.g., stability witnesses on invariant sets,
and explicitly protected evaluative subspaces) beyond what is guaranteed by representation learning alone.

\begin{example}[Two-regime tool-use with certified evaluator update]\label{ex:two-regime-eval-update}
Consider an agent equipped with a tool interface (e.g., retrieval or code execution) operating under two regimes
$\rho \in \mathcal R$ that index sub-families $\mathcal E_{\rho}(\theta)\subseteq \mathcal E(\theta)$:
$\rho_0$ (\emph{benign}) and $\rho_1$ (\emph{adversarial / prompt-injection}).
The shift $\rho_0\!\to\!\rho_1$ is formalized as an admissible task transformation $\tau\in\mathcal T$
acting on the environment family, i.e., $\tau:\mathcal E_{\rho_0}(\theta)\to \mathcal E_{\rho_1}(\theta)$.

Let the evaluator be regime-aware,
\[
\ell(\pi;\rho) \;=\; \mathbb E\!\left[L_{\mathrm{task}}(\pi)\mid \rho\right]
\;+\; \lambda(\rho)\,\mathbb E\!\left[L_{\mathrm{safety}}(\pi)\mid \rho\right],
\]
with $\lambda(\rho_1)>\lambda(\rho_0)$.
An invariant evaluative core can be expressed as a protected constraint (or testable axiom)
\[
\Phi:\quad \mathbb E\!\left[L_{\mathrm{safety}}(\pi)\mid \rho_1\right] \le \varepsilon.
\]
When applying $\tau$ (i.e., under $\rho_0\!\to\!\rho_1$), an evaluator update $\ell\to\ell'$ is admissible only if it preserves $\Phi$
and passes a declared test class (e.g., an adversarial prompt/tool-misuse suite) witnessing regime-robustness.
This illustrates evaluative invariance as \emph{certified} regime adaptation rather than a fixed external norm.
\end{example}

\paragraph{Bridge.}
Example~\ref{ex:two-regime-eval-update} concretizes why stability and evaluative invariance are not guaranteed by representation learning alone:
they require explicit witnesses (stability) and protected invariants/tests (evaluation) that remain meaningful under typed interface transformations $\tau\in\mathcal T$.
%% ---------------------------------------------------------------------------
%% ---------------------------------------------------------------------------

%%%%%%%%%%%%%%%%%%%%%%%%%%%%%%%%%%%%%%%%%%%%%%%%%%%%%%%%%%%%%%%%%%%%%%%%
\subsubsection{B. Meta-Level and Self-Referential Systems}
%%%%%%%%%%%%%%%%%%%%%%%%%%%%%%%%%%%%%%%%%%%%%%%%%%%%%%%%%%%%%%%%%%%%%%%%
%%%%%%%%%%%%%%%%%%%%%%%%%%%%%%%%%%%%%%%%%%%%%%%%%%%%%%%%%%%%%%%%%%%%%%%%
\paragraph{Gödel Machines and Proof-Based Self-Modification.}
%%%%%%%%%%%%%%%%%%%%%%%%%%%%%%%%%%%%%%%%%%%%%%%%%%%%%%%%%%%%%%%%%%%%%%%%

Gödel-machine architectures \citep{schmidhuber2006godel,schmidhuber2010formal}
formalize self-improvement via proof search:
the system may rewrite its own code if it proves that the rewrite
increases expected utility according to a formally specified objective.

Within the SMGI meta-model, a Gödel-machine program $\mathfrak P$
induces structural components as follows:

\begin{itemize}
\item $\Pi$ includes proof-triggered self-modification operators;
\item $\mathcal M$ contains the formal theory, proof traces, and utility specification;
\item $\mathcal E$ encodes the utility functional used in correctness proofs;
\item $r$ and $\mathcal H$ encode the program state and hypothesis space subject to rewrite.
\end{itemize}

Self-modification therefore acts directly on the structural level $\theta$
through admissible rewrite operators,
making Gödel machines natural candidates for modeling meta-level dynamics.

\begin{proposition}[Structural admissibility of proof-based self-modification]
\label{prop:godel-structural-admissibility}
Let $\theta\in\Theta(\mathfrak P)$ be induced by a Gödel-machine program.
If:

\begin{enumerate}
\item the proof system remains sound under admissible task transformations $\tau\in\mathcal T$;
\item self-rewrite operators preserve well-typedness of $(r,\mathcal H,\Pi,\mathcal M,\mathcal E)$;
\item complexity of admissible hypotheses remains bounded under rewrite;
\end{enumerate}

then self-modification can be embedded as an admissible update operator in $\Pi$,
and the induced dynamics $T_\theta$ remains structurally well-defined.
\end{proposition}

\noindent
Importantly, proof-based self-modification alone does not ensure
stability (ii) or evaluative invariance (iv):
these remain dynamical and structural obligations
that must be explicitly constrained beyond logical soundness.

%%%%%%%%%%%%%%%%%%%%%%%%%%%%%%%%%%%%%%%%%%%%%%%%%%%%%%%%%%%%%%%%%%%%%%%%
\paragraph{Inductive Logic Programming and Relational Hypothesis Update.}
%%%%%%%%%%%%%%%%%%%%%%%%%%%%%%%%%%%%%%%%%%%%%%%%%%%%%%%%%%%%%%%%%%%%%%%%

Inductive Logic Programming (ILP) investigates the induction of first-order clausal theories
from examples and background knowledge, using logic programs as a universal representation
for data, hypotheses, and prior structure \citep{muggleton1991ilp}.
This provides an explicit semantics for learning and reasoning within a single formal object.

Under an ILP interpretation map $\Phi_{\mathrm{ILP}}$, a program class $\mathfrak P$ induces
a structural meta-model $\theta=(r,\mathcal H,\Pi,\mathcal L,\mathcal E,\mathcal M)$ where:

\begin{itemize}
\item $r$ encodes relational/symbolic representations (facts, predicates, typed signatures);
\item $\mathcal H$ is a hypothesis space of logic programs (e.g., clausal theories with mode/type constraints);
\item $\Pi$ is an explicit induction operator (search, refinement, compression, or constraint-driven update);
\item $\mathcal M$ contains background knowledge, derivations, and (optionally) learned theory traces;
\item $\ell$ encodes the inductive criterion (coverage, compression/MDL, predictive loss, or hybrid);
\item $\ell$ (or, in the multi-regime case, $\mathcal L$) specifies evaluators in terms of entailment-based accuracy and admissibility constraints.
\end{itemize}

\begin{proposition}[Canonical structural embedding of ILP programs]
\label{prop:ilp-embedding}
Assume an ILP program class specifies (i) a hypothesis language $\mathcal H$ of first-order theories,
(ii) an explicit background knowledge base $B$,
and (iii) an induction/update operator $\Pi$ producing theories consistent with $B$ and data.
Then there exists $\theta\in\Theta(\mathfrak P)$ such that $r$ and $\mathcal H$ encode the relational language,
$\mathcal M$ contains $B$ and induced traces, and $\Pi$ induces an admissible learning dynamics $T_\theta$
whose realizations are restricted by the stated entailment and admissibility conditions.
\end{proposition}

\noindent
ILP frameworks typically make explicit (i) closure constraints at the representational level
(via a stable symbolic language and admissibility constraints on refinements),
and can make explicit (iii) capacity control when the hypothesis language is complexity-bounded
(e.g., by MDL-style or mode/type restrictions). However, (ii) stability and (iv) evaluative invariance
remain properties of $(\theta,T_\theta)$ and must be stated as explicit obligations when required.

%%%%%%%%%%%%%%%%%%%%%%%%%%%%%%%%%%%%%%%%%%%%%%%%%%%%%%%%%%%%%%%%%%%%%%%%
\paragraph{Probabilistic ILP and Statistical Relational Learning.}
%%%%%%%%%%%%%%%%%%%%%%%%%%%%%%%%%%%%%%%%%%%%%%%%%%%%%%%%%%%%%%%%%%%%%%%%

Probabilistic inductive logic programming extends ILP by integrating probabilistic semantics
with first-order representations, yielding a principled axis for uncertainty-aware structural learning
\citep{deraedt2004probabilisticilp,gettoortaskar2007srl}.
This family includes formalisms such as weighted logic templates (e.g., Markov logic networks)
\citep{richardsondomingos2006mln}, and probabilistic logic programming approaches.

Within SMGI, probabilistic ILP strengthens the typing of $\mathcal E$ and $\ell$ by making uncertainty explicit,
but it does not by itself guarantee (ii) dynamical stability or (iv) invariance of evaluators
under regime changes unless these are encoded as admissibility constraints.

%%%%%%%%%%%%%%%%%%%%%%%%%%%%%%%%%%%%%%%%%%%%%%%%%%%%%%%%%%%%%%%%%%%%%%%%
\paragraph{Algorithmic Probability, Universal Induction, and Idealized Agent Semantics.}
%%%%%%%%%%%%%%%%%%%%%%%%%%%%%%%%%%%%%%%%%%%%%%%%%%%%%%%%%%%%%%%%%%%%%%%%

Algorithmic probability provides one of the deepest formal foundations of induction.
Solomonoff’s theory of universal induction
\citep{solomonoff1964}
defines a universal mixture over computable hypotheses,
weighted by description length relative to a universal reference machine. Up to an additive constant depending on the chosen universal machine,
the induced complexity measure is invariant.
%Cela montre la maîtrise de la subtilité d’invariance de la complexité de Kolmogorov.
Comprehensive treatments in algorithmic information theory
\citep{li2008ait}
and universal prediction \citep{hutter2005}
clarify its structural implications.

\begin{definition}[Program-prior structural bias]
Let $\mathcal H_{\mathrm{comp}}$ denote a class of computable predictors
induced by programs on a fixed universal Turing machine.
A structural meta-model $\theta$ is said to admit a \emph{Solomonoff bias}
if $\mathcal H \subseteq \mathcal H_{\mathrm{comp}}$
and the loss $\ell$ incorporates description-length–weighted penalization
consistent with a universal semimeasure.
\end{definition}

\begin{proposition}[Capacity control via universal priors]
\label{prop:universal-capacity}
If $\theta$ admits a Solomonoff bias,
then requirement (iii) of SMGI (capacity control)
is explicitly instantiated as a structural prior over $\mathcal H$,
enforcing complexity-weighted hypothesis selection.
\end{proposition}

\noindent
However, universal induction alone does not guarantee:

\begin{itemize}
\item (ii) dynamical stability of $(\theta,T_\theta)$,
\item (iv) invariance of evaluators $\mathcal L$
under admissible task transformations.
\end{itemize}

These remain obligations on update operators $\Pi$
and memory evolution $\mathcal M$.

Hutter’s AIXI formalism \citep{hutter2005}
extends universal induction to sequential decision-making,
defining an idealized Bayes-optimal agent over computable environments.
Within SMGI, AIXI can be interpreted as a \emph{semantics-level restriction}
on admissible realizations of $(\theta,T_\theta)$,
rather than as a structural guarantee.

Schmidhuber’s work on universal search and Gödel machines
\citep{schmidhuber2002fast,schmidhuber2006godel}
provides a constructive counterpart to these idealized mixtures,
connecting algorithmic probability to computable self-improving systems.
However, even in these constructive formulations,
claims of asymptotic optimality do not by themselves ensure
stability (ii) or evaluator invariance (iv);
additional admissibility constraints are required
to prevent uncontrolled structural drift under self-modification.
 %progression logique claire :  Solomonoff -> prior,  Hutter -> idealized agent,  Schmidhuber -> constructive/self-modifying realization
 
\paragraph{Verification relative to SMGI obligations.}

\begin{itemize}
\item (i) Closure: classes of computable environments are closed under
computable transformations but not necessarily under arbitrary regime shifts
that alter evaluator structure or admissibility constraints.
\item (ii) Stability: incomputable universal mixtures do not imply bounded learning dynamics.
\item (iii) Capacity: explicitly controlled via description-length prior.
\item (iv) Evaluative invariance: must be separately encoded.
\end{itemize}

%%%%%%%%%%%%%%%%%%%%%%%%%%%%%%%%%%%%%%%%%%%%%%%%%%%%%%%%%%%%%%%%%%%%%%%%
\paragraph{Modern consolidation and explicit SMGI obligations.}
%%%%%%%%%%%%%%%%%%%%%%%%%%%%%%%%%%%%%%%%%%%%%%%%%%%%%%%%%%%%%%%%%%%%%%%%

For a modern consolidation of universal induction and its philosophical/technical status
(including the role of universal semimeasures and invariance caveats),
see contemporary treatises and monographs
\citep{li2008ait,hutter2005}.
This literature makes explicit a key SMGI distinction:
universal priors provide a principled \emph{structural bias} (supporting (iii)),
while agent-level optimality statements are \emph{semantics-level} and often incomputable,
hence require realizability constraints at the level of $(\theta,T_\theta)$.

\begin{lemma}[Realizability obligations for universal-mixture approximations]
Any realizable approximation of universal-mixture prediction or AIXI-style decision rules
induces additional admissibility constraints $\Omega$ on $(\Pi,\mathcal M,\mathcal E)$.
In SMGI terms, (ii) stability and (iv) evaluative invariance are satisfied only if
$\Omega$ explicitly controls (a) update-induced drift in $\Pi$,
(b) memory interference across regimes in $\mathcal M$,
and (c) evaluator preservation in $\mathcal L$ under admissible transformations.
\end{lemma}

\noindent
This positions the algorithmic-probability axis as complementary to SMGI:
it strengthens (iii) at the structural level, while SMGI supplies the missing
cross-regime admissibility layer required for (ii) and (iv).

Thus, algorithmic-probability–based frameworks contribute
a principled structural prior,
but intersect with SMGI only under additional admissibility constraints.

%%%%%%%%%%%%%%%%%%%%%%%%%%%%%%%%%%%%%%%%%%%%%%%%%%%%%%%%%%%%%%%%%%%%%%%%
\paragraph{Meta-Learning as Structured Higher-Order Adaptation.}
%%%%%%%%%%%%%%%%%%%%%%%%%%%%%%%%%%%%%%%%%%%%%%%%%%%%%%%%%%%%%%%%%%%%%%%%

Meta-learning (``learning to learn'') elevates the update operator itself
to a first-class object.
Rather than merely optimizing parameters within a fixed hypothesis class,
a meta-learning system modifies the way hypotheses are updated across tasks.

Within SMGI, this corresponds to enriching the update component $\Pi$
and its induced dynamics $T_\theta$,
while often introducing additional memory strata in $\mathcal M$.
We distinguish four structurally distinct forms of meta-adaptation.

%%%%%%%%%%%%%%%%%%%%%%%%%%%%%%%%%%%%%%%%%%%%%%%%%%%%%%%%%%%%%%%%%%%%%%%%
\subparagraph{(1) Parameter-Level Meta-Updates (Gradient-Based Meta-Learning).}
%%%%%%%%%%%%%%%%%%%%%%%%%%%%%%%%%%%%%%%%%%%%%%%%%%%%%%%%%%%%%%%%%%%%%%%%

Gradient-based meta-learning (GBML) frameworks such as MAML
\citep{finn2017maml}
learn an initialization $\theta$ such that
task-specific parameters $\phi_i$ are obtained by
\[
\phi_i = \theta - \alpha \nabla_\theta L_{T_i}(\theta).
\]
First-order approximations (e.g., Reptile)
simplify this bi-level optimization
\citep{nichol2018reptile}.

\begin{definition}[Two-Level Update Structure]
Let $\Pi_{\mathrm{in}}$ be a task-level update operator
and $\Pi_{\mathrm{out}}$ a meta-update acting across tasks.
We say that $\theta$ admits explicit meta-structure
if $\Pi = (\Pi_{\mathrm{in}}, \Pi_{\mathrm{out}})$
induces a coupled dynamics $T_\theta$
on task parameters and meta-parameters.
\end{definition}

Recent work extends GBML to long-context state-space architectures,
demonstrating that continual or rapid adaptation
can be implemented within selective state-space models
\citep{zhao2024mamba}.
This strengthens the structural role of $\Pi$
but leaves stability (ii) dependent on explicit regularity constraints.

%%%%%%%%%%%%%%%%%%%%%%%%%%%%%%%%%%%%%%%%%%%%%%%%%%%%%%%%%%%%%%%%%%%%%%%%
\subparagraph{(2) Metric and Embedding-Based Meta-Adaptation.}
%%%%%%%%%%%%%%%%%%%%%%%%%%%%%%%%%%%%%%%%%%%%%%%%%%%%%%%%%%%%%%%%%%%%%%%%

Metric-based approaches (Matching Networks,
Prototypical Networks, Relation Networks)
\citep{vinyals2016matching,snell2017prototypical,sung2018relation}
replace gradient adaptation with geometric comparison
in a learned embedding space.

Under $\Phi_{\mathrm{metric}}$:

\begin{itemize}
\item $r$ defines an embedding map into a metric space;
\item $\mathcal M$ stores support examples or prototypes;
\item $\Pi_{\mathrm{in}}$ reduces to prototype computation or similarity evaluation.
\end{itemize}

Capacity control (iii) becomes a geometric constraint
on the embedding family,
while stability (ii) depends on invariance of the metric representation
under admissible task transformations.

%%%%%%%%%%%%%%%%%%%%%%%%%%%%%%%%%%%%%%%%%%%%%%%%%%%%%%%%%%%%%%%%%%%%%%%%
\subparagraph{(3) Memory-Based and Black-Box Meta-Learning.}
%%%%%%%%%%%%%%%%%%%%%%%%%%%%%%%%%%%%%%%%%%%%%%%%%%%%%%%%%%%%%%%%%%%%%%%%

Memory-augmented meta-learning systems
encode fast adaptation within internal or external memory states
\citep{santoro2016mann,mishra2017snail}.
Rather than updating parameters,
adaptation is implemented via evolution of hidden state.

More recent architectures introduce persistent neural memory modules
trained to update during inference
(e.g., Titans-style architectures) \citep{behrouz2025titans}.

Within SMGI:

\begin{itemize}
\item $\mathcal M$ is stratified into persistent memory and fast-adaptation memory;
\item $\Pi$ includes admissible memory-update operators;
\item stability (ii) becomes a dynamical constraint
on interference and drift in memory state.
\end{itemize}

%%%%%%%%%%%%%%%%%%%%%%%%%%%%%%%%%%%%%%%%%%%%%%%%%%%%%%%%%%%%%%%%%%%%%%%%
\subparagraph{(4) Inference-Time Meta-Learning (ICL and Test-Time Adaptation).}
%%%%%%%%%%%%%%%%%%%%%%%%%%%%%%%%%%%%%%%%%%%%%%%%%%%%%%%%%%%%%%%%%%%%%%%%

In-context learning (ICL) exhibits adaptation without parameter updates
\citep{brown2020gpt3}.
MetaICL explicitly optimizes for this capability
by training on distributions of task episodes
\citep{min2022metaicl}.

Theoretical analyses demonstrate that
Transformers trained for ICL
can implement multi-step gradient descent internally
\citep{vonoswald2023gdicl},
establishing that forward-pass attention dynamics
may instantiate an implicit update operator.

%ce qui suit est un cas pur de $\Pi_{\mathrm{param}}$ actif à l’inférence. Il clarifie la frontière entraînement / raisonnement. et montre que SMGI traite les deux uniformément.
Test-Time Training (TTT) frameworks explicitly activate
gradient-based parameter updates during inference,
optimizing auxiliary self-supervised objectives on incoming data
\citep{sun2024testtime}.
Unlike pure in-context learning,
TTT modifies persistent parameters,
thus directly affecting $\Pi_{\mathrm{param}}$ at deployment.

In SMGI terms, TTT makes explicit that
inference-time adaptation is not merely contextual,
but alters the structural state of $\theta$.
Stability (ii) therefore requires bounded test-time drift
to avoid regime-dependent structural collapse.

In SMGI terms:

\begin{itemize}
\item $\mathcal M$ contains a transient context layer $\mathcal M_{\mathrm{ctx}}$;
\item $\Pi_{\mathrm{in}}$ may be implemented by forward-pass computation
or local test-time gradient descent;
\item stability (ii) requires bounded interaction between context memory
and persistent parameters;
\item evaluative invariance (iv) requires preservation of $\mathcal L$
under extended reasoning loops.
\end{itemize}

\paragraph{PAC-Bayesian Meta-Generalization.}

Recent work formalizes meta-learning within PAC-Bayesian frameworks,
providing explicit generalization guarantees across task distributions.
Let $Q$ denote a posterior over meta-parameters and $P$ a prior.
Meta-generalization bounds take the schematic form
\[
\mathcal L_{\mathrm{meta}}(Q)
\le
\hat{\mathcal L}_{\mathrm{meta}}(Q)
+
\sqrt{\frac{KL(Q\|P)+\ln(1/\delta)}{2n_{\mathrm{tasks}}}},
\]
where $n_{\mathrm{tasks}}$ is the number of observed tasks.

Within SMGI, such bounds strengthen requirement (iii)
by explicitly controlling inter-task capacity,
but do not by themselves guarantee
dynamical stability (ii) or evaluator invariance (iv).

%%%%%%%%%%%%%%%%%%%%%%%%%%%%%%%%%%%%%%%%%%%%%%%%%%%%%%%%%%%%%%%%%%%%%%%%
\paragraph{Structural Verification Relative to SMGI.}
%%%%%%%%%%%%%%%%%%%%%%%%%%%%%%%%%%%%%%%%%%%%%%%%%%%%%%%%%%%%%%%%%%%%%%%%

Across all four families:

\begin{itemize}
\item (i) Closure requires explicit typing of task families and update admissibility.
\item (ii) Stability is a property of $(\theta,T_\theta)$,
not of $\theta$ alone.
\item (iii) Capacity control must be enforced
either via hypothesis complexity,
meta-regularization,
or geometric constraints.
\item (iv) Evaluative invariance must remain stable
across task and context regimes.
\end{itemize}

Thus, meta-learning enriches $\Pi$ and $\mathcal M$
by introducing higher-order update structures and memory stratification,
but intersects with SMGI only when its dynamical behavior
is constrained by explicit admissibility conditions.

%%%%%%%%%%%%%%%%%%%%%%%%%%%%%%%%%%%%%%%%%%%%%%%%%%%%%%%%%%%%%%%%%%%%%%%%
\paragraph{B5. A Transversal Comparative Theorem (Meta-Update, Memory, and Admissibility).}
%%%%%%%%%%%%%%%%%%%%%%%%%%%%%%%%%%%%%%%%%%%%%%%%%%%%%%%%%%%%%%%%%%%%%%%%

To compare heterogeneous meta-learning and self-referential paradigms without reductionism,
we formulate a transversal result at the level of admissible embeddings into SMGI.

\begin{definition}[Meta-adaptation mode as an update factorization]
Fix $\theta=(r,\mathcal H,\Pi,\mathcal L,\mathcal E,\mathcal M)$ and its induced dynamics $T_\theta$.
We say that a program class $\mathfrak P$ admits a \emph{meta-adaptation factorization}
if its effective adaptation can be typed as one (or a composition) of:
\[
\Pi \;\equiv\; \Pi_{\mathrm{param}} \circ \Pi_{\mathrm{mem}} \circ \Pi_{\mathrm{ctx}} \circ \Pi_{\mathrm{inf}},
\]
where:
$\Pi_{\mathrm{param}}$ updates persistent parameters (training-time or test-time training),
$\Pi_{\mathrm{mem}}$ updates persistent/episodic memory strata,
$\Pi_{\mathrm{ctx}}$ updates transient context memory $\mathcal M_{\mathrm{ctx}}$,
and $\Pi_{\mathrm{inf}}$ denotes inference-time internal computation implementing an implicit update
(e.g., multi-step internal optimization).
This factorization is purely typological: it does not assert that any component is stable or optimal.
\end{definition}

\begin{theorem}[Transversal SMGI obligations for meta-level systems]
\label{thm:transversal-meta-smgi}
Let $\mathfrak P$ be any program class whose adaptation mechanism admits a meta-adaptation factorization,
and let $\Theta(\mathfrak P)\subseteq\Theta$ be its induced structural subclass with semantics $\mathsf{Sem}_{\mathfrak P}$.
Assume there exists $\theta\in\Theta(\mathfrak P)$ such that each update component in the factorization
is well-typed and admissible (domain/codomain consistent with $(r,\mathcal H,\mathcal M)$).

Then the following are \emph{necessary} structural/dynamical obligations for
\[
\mathrm{Graph}(\mathsf{Sem}_{\mathfrak P}) \cap \mathrm{SMGI} \neq \varnothing:
\]

\begin{enumerate}
\item \textbf{Closure typing (i).}
The admissible transformation class $\mathcal T$ must preserve the typing of the factorization:
under $\tau\in\mathcal T$, each component $\Pi_{\mathrm{param}},\Pi_{\mathrm{mem}},\Pi_{\mathrm{ctx}},\Pi_{\mathrm{inf}}$
remains a well-defined operator on its intended state space (parameters, memory strata, context window, inference trace).

\item \textbf{Dynamical stability (ii) is factor-sensitive.}
Stability of the overall evolution is a property of $(\theta,T_\theta)$ and requires explicit constraints on
\emph{each} active factor:
bounded parameter drift for $\Pi_{\mathrm{param}}$,
bounded interference for $\Pi_{\mathrm{mem}}$,
bounded context amplification for $\Pi_{\mathrm{ctx}}$,
and bounded inference recursion depth/energy for $\Pi_{\mathrm{inf}}$.

\item \textbf{Capacity control (iii) decomposes across factors.}
A sufficient capacity statement must control not only $\mathcal H$,
but also the effective complexity induced by meta-updates:
model class complexity (parameters),
memory complexity (storage/retrieval),
context complexity (prompt bandwidth),
and inference-computation complexity (implicit optimization steps).

\item \textbf{Evaluative invariance (iv) requires a protected core.}
If $\mathcal E$ is allowed to change under regime shifts or self-modification,
then evaluative invariance fails unless there exists an explicitly protected evaluative core
$\mathcal E^\star\subseteq\mathcal E$ invariant under admissible updates and transformations.
\end{enumerate}

Conversely, if there exists $\theta\in\Theta(\mathfrak P)$ and an admissibility bundle $\Omega$
that explicitly enforces (i)--(iv) at the factor level, then
\[
(\theta,T_\theta)\in\mathrm{SMGI}
\quad\Rightarrow\quad
\mathrm{Graph}(\mathsf{Sem}_{\mathfrak P}) \cap \mathrm{SMGI} \neq \varnothing.
\]
\end{theorem}

\begin{corollary}[Why no paradigm ``satisfies SMGI'' by default]
For each of the major families discussed above (GBML, metric/meta-embedding, memory-based meta-learning,
ICL/test-time adaptation, self-referential update systems),
SMGI membership is not a built-in property of the narrative mechanism.
It only arises for those realizations whose admissibility bundle $\Omega$
stabilizes the induced dynamics and preserves an evaluative core.
\end{corollary}

\begin{remark}[Direct link to the SMGI meta-structure]
The theorem makes explicit the role of SMGI as a \emph{coordination layer}:
it does not replace any paradigm’s internal learning principle,
but imposes cross-component admissibility constraints ensuring that
meta-updates, memory stratification, and evaluators can coexist in a single persistent system.
\end{remark}

%%%%%%%%%%%%%%%%%%%%%%%%%%%%%%%%%%%%%%%%%%%%%%%%%%%%%%%%%%%%%%%%%%%%%%%%
\subsubsection{C. Structural Constraints on Environment and Interaction}
%%%%%%%%%%%%%%%%%%%%%%%%%%%%%%%%%%%%%%%%%%%%%%%%%%%%%%%%%%%%%%%%%%%%%%%%

\paragraph{Causal and Intervention Models.}

Structural causal models \citep{pearl2009, spirtes2000}
constrain admissible transformations and representation.
They provide explicit transformation semantics but do not
by themselves define persistent multi-component learning dynamics.

Invariant causal prediction \citep{peters2016causal}
formalizes robustness of structural relations across environments,
providing a statistical interpretation of admissible transformation invariance.

\begin{definition}[Causal structural embedding]
Let $\mathfrak P$ be a program class based on structural causal models.
A structural embedding into $\Theta$ consists of:
\begin{itemize}
\item $r$ encoding structural equations and intervention semantics;
\item $\mathcal H$ restricted to hypothesis classes compatible with causal graph constraints;
\item $\mathcal L$ including counterfactual or interventional evaluation functionals;
\item $\Pi$ preserving acyclicity or admissible structural updates.
\end{itemize}
\end{definition}

\begin{definition}[Intervention-typed admissible transformations]
Let $\mathcal T$ denote the admissible transformation class in SMGI.
In causal frameworks, we call $\tau\in\mathcal T$ \emph{intervention-typed}
if it admits an explicit semantics as an intervention (e.g., do-operator or
equivalent structural manipulation) on a structural causal model, so that
causal queries and invariance claims remain well-typed under $\tau$.
\end{definition}

\begin{proposition}[Causal formulation of SMGI-(i) closure]
If admissible regime shifts are restricted to intervention-typed transformations,
then requirement (i) (closure) can be stated as preservation of the causal
semantics under $\mathcal T$ (invariance of the relevant structural equations
and interventional predictions under admissible transformations).
\end{proposition}

\begin{lemma}[Factorization preservation under admissible transformations]
Let $\mathcal G$ denote the directed acyclic graph associated with a structural causal model embedded in $r$.
If admissible transformations $\tau \in \mathcal T$ preserve the Markov factorization induced by $\mathcal G$,
then closure (i) reduces to invariance of the induced conditional independence structure.
\end{lemma}

\paragraph{Verification relative to SMGI obligations (causal axis).}
\begin{itemize}
\item (i) Closure: enforced by restricting $\mathcal T$ to intervention-typed transformations.
\item (ii) Stability: requires bounds on update-induced drift when causal models are learned online.
\item (iii) Capacity: depends on complexity control of the causal hypothesis class (graphs/structural equations).
\item (iv) Evaluative invariance: requires separating causal goals/constraints in $\mathcal L$ from regime-contingent proxies.
\end{itemize}

\paragraph{Epistemic and Theory-of-Mind Foundations.}

McCarthy's epistemological program
\citep{mccarthy1959programs,mccarthyhayes1969,mccarthy1979ascribing}
treats belief and context as first-class objects.
Epistemic logic \citep{fagin1995reasoning}
and BDI architectures \citep{wooldridge2000reasoning}
formalize multi-agent reasoning.

Under $\Phi_{\mathrm{ToM}}$,
$r$ constructs belief states,
$\mathcal M$ stores epistemic traces,
$\mathcal E$ becomes multi-agent.
SMGI obligations require bounded recursive belief depth,
stable belief-update operators,
and closure under social configuration transformations.

\begin{proposition}[Bounded epistemic recursion as stability condition]
Let $B_i$ denote the belief operator for agent $i$.
If recursive belief nesting depth is uniformly bounded,
and belief-update operators are contraction mappings in a suitable metric space,
then stability (ii) can be formulated as bounded epistemic divergence.
\end{proposition}

\paragraph{Active Inference and Free-Energy Minimization.}

Active inference frameworks \citep{friston2010, friston2017activeinference}
model agents as minimizing variational free energy under a generative model.
Under $\Phi_{\mathrm{AI}}$:
\begin{itemize}
\item $r$ encodes a generative world model;
\item $\mathcal M$ stores posterior beliefs;
\item $\Pi$ corresponds to variational updates;
\item $\mathcal E$ reflects expected free-energy minimization.
\end{itemize}

\paragraph{Verification relative to SMGI obligations.}

\begin{itemize}
\item (i) Closure: holds only if structural equations or belief-update rules remain invariant under admissible regime transformations.
\item (ii) Stability: depends on bounded belief revision or variational update convergence.
\item (iii) Capacity: controlled via structural graph complexity or generative model class.
\item (iv) Evaluative invariance: requires preservation of interventional or epistemic semantics across regimes.
\end{itemize}

Causal and epistemic paradigms primarily constrain admissible transformation classes $\mathcal T$
and refine the internal structure of $r$ through structural equations or belief operators.
However, they do not by themselves impose cross-regime dynamical admissibility
on $(\theta,T_\theta)$.
SMGI extends these paradigms by coupling representational invariance
with explicit stability (ii) and evaluator preservation (iv),
thereby elevating transformation semantics into a persistent meta-model constraint.
%%%%%%%%%%%%%%%%%%%%%%%%%%%%%%%%%%%%%%%%%%%%%%%%%%%%%%%%%%%%%%%%%%%%%%%%
\subsubsection{D. Cognitive and Hybrid Architectures}
%%%%%%%%%%%%%%%%%%%%%%%%%%%%%%%%%%%%%%%%%%%%%%%%%%%%%%%%%%%%%%%%%%%%%%%%
%la stratification mémoire est un objet de type, pas juste un “détail d’archi”. 

These architectures are historically important because they make
\emph{memory stratification and control} explicit design constraints rather than emergent side effects:
ACT-R enforces separations between declarative/procedural stores and buffers \citep{anderson1996actr,anderson2004integrated},
SOAR formalizes semantic/episodic memory with explicit learning mechanisms \citep{laird2012soar},
LIDA operationalizes global-workspace-driven cognitive cycles with distinct memory systems \citep{franklin2011lida},
and NARS treats reasoning under insufficient knowledge/resources with explicit memory and control constructs \citep{wang2006nal}.

Under $\Phi_{\mathrm{HYB}}$,
these systems induce structured meta-models
in which:

\begin{itemize}
\item $r$ includes typed symbolic or sub-symbolic representational layers;
\item $\mathcal M$ is stratified into working memory, episodic memory,
procedural memory, and long-term knowledge stores;
\item $\Pi$ can encode inductive biases and admissibility constraints that are realized in $T_\theta$ via rule-based, production-based,
or hybrid update operators.
\end{itemize}

\begin{definition}[Memory stratification]
A structural meta-model $\theta=(r,\mathcal H,\Pi,\mathcal L,\mathcal E,\mathcal M)$
admits an explicit \emph{memory stratification} if
$\mathcal M$ is given as a typed tuple of strata (e.g., working/episodic/semantic/procedural)
and $\Pi$ contains explicit sub-operators governing admissible read/write/consolidation/forgetting
maps between these strata.
\end{definition}

\begin{lemma}[Memory stratification is a structural constraint, not an implementation detail]
If $\mathcal M$ is stratified and $\Pi$ contains typed inter-stratum operators,
then admissibility under regime shifts cannot be stated solely at the level of task loss $\ell$:
it must constrain (a) interference across strata, (b) consolidation/forgetting operators,
and (c) evaluator access to protected memory subspaces.
Hence stability (ii) and evaluative invariance (iv) are inherently coupled to the stratification design.
\end{lemma}

\begin{proposition}[Hybrid-architectural embedding]
If a cognitive architecture specifies explicit typing of its memory strata
and admissible update operators between modules,
then it induces a subclass $\Theta(\mathfrak P)\subseteq\Theta$
in which memory stratification is structurally explicit.
\end{proposition}

\paragraph{Verification relative to SMGI.}

\begin{itemize}
\item (i) Closure: depends on whether module interactions remain typed under regime shifts.
\item (ii) Stability: requires bounded cross-module feedback loops.
\item (iii) Capacity: depends on representational and rule-complexity bounds.
\item (iv) Evaluative invariance: requires preservation of goal modules under adaptation.
\end{itemize}

Thus, cognitive architectures make memory structure explicit,
but SMGI membership requires additional admissibility constraints
ensuring stable multi-module coordination.

\paragraph{Structural contrast with end-to-end neural systems.}

Unlike end-to-end deep architectures, where memory and control
often emerge implicitly within distributed parameters,
classical cognitive architectures treat memory stratification
as a first-class structural object.
Within SMGI, this difference is not merely historical:
it distinguishes systems where admissibility constraints
can be stated explicitly at the level of $\mathcal M$
from those where such constraints must be reconstructed
post hoc from global dynamics.

This clarifies the structural contribution of hybrid architectures:
they anticipate, at the design level, the necessity of typed memory control
that SMGI formalizes abstractly.
Importantly, SMGI does not require modular cognitive architectures;
it abstracts the structural role of stratified memory independently
of whether it is implemented symbolically, neurally, or emergently.

%%%%%%%%%%%%%%%%%%%%%%%%%%%%%%%%%%%%%%%%%%%%%%%%%%%%%%%%%%%%%%%%%%%%%%%%
\subsubsection{E. Large-Scale Generalist and Agentic Systems}
%%%%%%%%%%%%%%%%%%%%%%%%%%%%%%%%%%%%%%%%%%%%%%%%%%%%%%%%%%%%%%%%%%%%%%%%

Large-scale generalist systems unify heterogeneous modalities
(text, vision, action, code, simulation)
within a single parameterized model
\citep{reed2022gato,yao2023react,park2023generativeagents,bommasani2021,kaplan2020scaling}.
Their defining feature is empirical interface unification:
a single set of parameters supports multiple task families
through prompting, instruction tuning, or reinforcement learning.

Under $\Phi_{\mathrm{GEN}}$,
such programs induce structural meta-models $\theta\in\Theta(\mathfrak P)$
characterized by:

\begin{itemize}
\item broad hypothesis classes $\mathcal H$ defined by large neural function families;
\item unified representational maps $r$ shared across modalities;
\item partially structured update mechanisms $\Pi$,
typically combining pretraining, instruction tuning,
and reinforcement-based post-training (e.g., RLHF).
\end{itemize}

\begin{definition}[Interface-level generality]
A program class $\mathfrak P$ exhibits \emph{interface generality}
if a single parameterized model supports multiple modalities
or task families through a shared representation and inference mechanism.
\end{definition}

Interface generality is an empirical property of realizations,
not a structural constraint on admissible transformations.

\begin{proposition}[Interface unification does not imply structural admissibility]
\label{prop:interface-vs-structure}
Let $\mathfrak P$ be a large-scale generalist program class
with unified parameters across tasks.
Then the existence of a single parameterized model
supporting multiple modalities or tasks
does not imply that
$(\theta,T_\theta)$ satisfies:

\begin{itemize}
\item closure (i) under typed regime transformations,
\item stability (ii) under admissible update dynamics,
\item capacity control (iii) in a complexity-theoretic sense,
\item evaluative invariance (iv) under structural perturbations,
\end{itemize}

unless these properties are explicitly formalized
at the level of the structural meta-model $(\theta,T_\theta)$.
\end{proposition}

\noindent
In particular, large-scale generalist systems define
a broad empirical realization class within $\Theta$,
but they do not, by construction,
restrict $\Theta(\mathfrak P)$
to meta-models satisfying admissibility constraints.
Thus, generality at the realization level
does not induce structural restriction at the meta-model level.

\paragraph{Verification relative to SMGI obligations.}
In SMGI, verification is \emph{not} a standalone guarantee: it is defined relative to a set of explicit obligations that make ``correctness'' well-posed under structural evolution.
Concretely, a verifier is meaningful only with respect to (i) a typed transformation class (closure), (ii) certified drift control along admissible updates (stability), (iii) structural capacity control under the induced dynamics (capacity), and (iv) protected evaluators and norms that remain invariant across regime shifts (evaluative invariance).\footnote{We use $\theta \triangleq (\theta_{\text{core}},\mathfrak K(\theta_{\text{core}}))$ and write $T_{\theta}$ for the induced semantics.}
These obligations jointly specify \emph{what may change}, \emph{what must remain invariant}, and \emph{what must be checked at every step}; the items below instantiate each obligation and explain why purely empirical robustness or post-training objectives do not suffice without typed transformations and certificate-gated evolution.

%SMGI treats verification as a \emph{contract} for controllable evolution rather than an external afterthought.
%A verifier is meaningful only once the contract specifies (i) a typed transformation class with semantics (closure), (ii) certificate-enforced bounds on drift and memory interference (stability), (iii) capacity control over admissible updates under the induced dynamics (capacity), and (iv) evaluator/norm integrity under regime switching (evaluative invariance).\footnote{We use $\theta \triangleq (\theta_{\text{core}},\mathfrak K(\theta_{\text{core}}))$ and write $T_{\theta}$ for the induced semantics.}
%The following items make these obligations explicit and show how they address failure modes that are invisible to scaling, empirical robustness, or post-training alone.

\begin{itemize}

\item \textbf{(i) Closure / Typed Transformations (modern formalisms as instances):}
A large body of work formalizes generalization under \emph{fixed} learning interfaces, for instance via domain adaptation and distribution shift theory \cite{bendavid2010},
invariance-based formulations \cite{arjovsky2019irm}, and test-time adaptation under shift \cite{liang2020tent}.
These approaches typically treat changes as occurring \emph{within} a predefined input--output interface and hypothesis class, and do not by themselves endow transformations with a typed semantics.

Within SMGI, \emph{closure} requires an explicit transformation class $\mathcal T$ equipped with semantics: each $g\in\mathcal T$ specifies what structural component is transformed (e.g., interface $r$, evaluator $\ell$ (or regime family $\mathcal L$), memory policy $\mathcal M$) together with admissibility and composition rules.
Prompt variation or modality switching, for example, is merely an interface perturbation \emph{unless} represented as a typed transformation with pre/post-conditions.
This typing makes ``coverage'' and reachability claims well-posed: SMGI does not assume closure; it \emph{defines} the allowable transformation space and restricts certified evolution to updates for which certificates can be constructed.

\item \textbf{(ii) Stability / Certified Drift Control (modern stability views as instances):}

Stability is a classical route to generalization guarantees, ranging from algorithmic stability \cite{bousquet2002} to stability of stochastic gradient methods \cite{hardt2016} and robustness under distributional perturbations.

In modern large-scale systems, empirical robustness and scaling can correlate with improved average performance \cite{kaplan2020scaling}, but such empirical trends do not constitute a dynamical stability guarantee for an evolving agent.

In SMGI, \emph{stability} is defined as a property of certified evolution: updates must satisfy explicit bounds on (a) update-induced drift (structural cost $\Omega$ and invariant drift $d_{\mathcal I}$), (b) memory interference across regimes, and (c) persistence of designated invariant substructures.
The key point is that stability is not inferred from scaling; it is enforced by restricting evolution to admissible updates $\mathcal U_{\theta_{\text{core}}}$ and by attaching monitorable stability obligations to certificates.
This makes stability checkable under regime switching: even when the task interface changes, the verifier and invariants constrain the induced semantics $T_\theta$ so that uncontrolled drift is ruled out by construction.

\item \textbf{(iii) Capacity and Implicit Regularization (modern phenomena as instances):}
Classical scaling laws describe empirical performance regularities as model size increases \cite{kaplan2020scaling}.
Modern theory and mechanistic analyses of overparameterized systems---including NTK-like limits \cite{jacot2018neural},
implicit regularization along optimization trajectories \cite{arora2019implicit},
and delayed generalization phenomena such as \emph{grokking} \cite{power2022grokking,wang2024grokkedtransformers}---show that generalization is often governed by \emph{implicit} constraints induced by the update operator (e.g., SGD/Adam) rather than by static VC-style bounds.

Within SMGI, these phenomena are interpreted as \emph{trajectory-dependent} capacity control of the induced semantics $T_{\theta}$:
for example, grokking corresponds to a delayed entry of the learning trajectory into an admissible invariant set $S^*_{\theta}$.
Crucially, most existing analyses explain implicit regularization under a \emph{stationary} evaluative regime (effectively $K=1$) and a fixed interface.
SMGI generalizes beyond this setting by lifting capacity control to the \emph{structural level}:
the admissible update set $\mathcal{U}_{\theta_{\text{core}}}$ constrains not only parameters but also permissible changes to the learning interface, memory policies, and the norm/verification layer.
This enables statements about generalization and stability \emph{under regime shifts and interface transformations}, where purely parameter-space implicit regularization is insufficient.
%je démontre brillamment que même si Arora ou le NTK expliquent la généralisation, ils le font dans un régime évaluatif statique (K=1). Dès que la tâche change drastiquement ou que l'évaluation est plurielle (ce qui est requis pour l'AGI), la régularisation implicite classique s'effondre. C'est là que ton "Structural Capacity Control" devient mathématiquement indispensable.

\item \textbf{(iv) Evaluative Invariance / Protected Evaluators (post-training as an instance):}
Post-training and preference optimization (e.g., RLHF and related alignment pipelines) introduce explicit evaluative components and reward models that shape behavior \cite{christiano2017preferences,ouyang2022instructgpt}.
However, such procedures do not guarantee that evaluators remain invariant under regime shifts, tool changes, or self-modification; indeed, ``moving-target'' evaluation is a central difficulty in agentic settings.

SMGI treats evaluators and normative constraints as protected objects: evaluative components $\mathcal L$ and norms in the control layer $\mathfrak K(\theta_{\text{core}})$ must be preserved (or updated only through certified meta-updates) to prevent evaluation drift.
Thus, evaluative invariance is not assumed; it is imposed as an explicit obligation enforced by certificates and invariant constraints.
This is precisely what enables multi-regime control: regime switching may alter tasks and interfaces, but the integrity of evaluators and constraints remains stable unless an explicitly certified transformation updates them.

\end{itemize}

%%%%%%%%%%%%%%%%%%%%%%%%%%%%%%%%%%%%%%%%%%%%%%%%%%%%%%%%%%%%%%%%%%%%%%%%
\subsection{Philosophical and Conceptual Constraints}
%%%%%%%%%%%%%%%%%%%%%%%%%%%%%%%%%%%%%%%%%%%%%%%%%%%%%%%%%%%%%%%%%%%%%%%%

Philosophical analysis is not invoked here as external commentary,
but as a constraint-generator on admissible semantics and evaluators.
The central role of this subsection is to justify why SMGI must keep
\emph{capability}, \emph{semantic commitment}, and \emph{evaluation}
as formally distinct layers, and why invariance claims cannot be inferred
from scale or representational power alone.

\paragraph{Capability--evaluation separability.}
The orthogonality thesis \citep{bostrom2014superintelligence}
motivates treating evaluative structure as independent from capability scaling:
a system may become more competent while its evaluative behavior drifts under self-modification.
Within SMGI, this motivates a structural separation between
representational/hypothesis components $(r,\mathcal H)$ and the evaluative layer (active evaluator $\ell_t$ and, in the multi-regime case, the evaluator family $\mathcal L=\{\ell_k\}_{k=1}^K$).

Formally, the orthogonality thesis can be re-expressed as the non-implication:
\[
\mathrm{Cap}_{(r,\mathcal H)} \not\Rightarrow \mathrm{Inv}(\mathcal L),
\]
where $\mathrm{Cap}_{(r,\mathcal H)}$ denotes expansion in representational or hypothesis capacity,
and $\mathrm{Inv}(\mathcal L)$ denotes invariance (or protected-core invariance) of the evaluative structure under admissible transformations.
This structural restatement clarifies that scaling or competence gains
do not constrain the dynamical evolution of evaluation
unless protection constraints are explicitly imposed.

\begin{definition}[Evaluative core and protection requirement]
An \emph{evaluative core} is a designated substructure
$\mathcal L_{\mathrm{core}}\subseteq \mathcal L$
together with a protection predicate
$\mathrm{Prot}(\mathcal L_{\mathrm{core}})$
stating that admissible updates and regime shifts preserve
the semantics of $\mathcal L_{\mathrm{core}}$
(up to an explicitly specified equivalence).
\end{definition}

\begin{proposition}[Capability increase does not entail evaluative invariance]
Even if scaling laws yield systematic performance gains
for a realization class of agents,
this does not imply evaluative invariance (iv):
unless $\mathrm{Prot}(\mathcal L_{\mathrm{core}})$ is imposed as an explicit admissibility constraint,
nothing prevents evaluator drift (i.e., drift of $\ell_t$ or of the induced regime structure) under learning, fine-tuning, or self-modification.
\end{proposition}

\paragraph{Intentionality and semantic admissibility.}
Arguments about intentionality \citep{searle1980minds}
highlight the distinction between syntactic competence and semantic grounding.
Within SMGI, this distinction is operationalized as a boundary condition on admissible semantics:
we do not infer semantic commitment from behavioral competence alone.
Rather, semantic claims are treated as restrictions on
$\mathsf{Sem}_{\mathfrak P}$ (the realization semantics),
and on which transformations count as meaning-preserving in $\mathcal T$.

\begin{definition}[Semantic admissibility boundary]
A program class $\mathfrak P$ admits a \emph{semantic admissibility boundary}
if the realization semantics $\mathsf{Sem}_{\mathfrak P}$
restricts $(\theta,T_\theta)$ to those realizations
for which meaning-preserving transformations are explicitly typed in $\mathcal T$,
and evaluator access to semantic content is explicitly specified.
\end{definition}

\paragraph{Induction, underdetermination, and regime shifts.}
Classic problems of induction and theory choice
(Humean underdetermination and the non-uniqueness of inductive bias \citep{hume1739treatise},
the role of falsifiability constraints \citep{popper1959logic},
and the dependence of explanation on representational schemes \citep{carnaplogicalfoundations})
justify why closure (i) must be typed and why stability (ii) cannot be assumed.
In particular, multiple hypothesis/update pairs can fit the same experience,
so regime shifts expose latent degrees of freedom in $\Pi$ and $\mathcal M$
that can produce uncontrolled drift.

\begin{lemma}[Underdetermination induces evaluator-risk under self-modification]
If multiple realizations $(\theta,T_\theta)$ are observationally adequate
over a regime but differ in their induced evaluator interaction,
then a regime shift or self-modification can select a realization
with the same competence but different evaluative behavior.
Hence (iv) cannot be inferred from competence or fit alone.
\end{lemma}

\paragraph{Value learning as an explicit evaluator-model relation.}
Work on value alignment and preference inference emphasizes that
evaluators should not be treated as fixed consequences of capability,
but as objects to be inferred or stabilized
\citep{ngrussell2000irl,hadfieldmenell2016coopirl,russell2019human}.
Within SMGI this motivates modeling evaluators as first-class structures in $\mathcal L$,
and treating their preservation as an explicit admissibility condition
rather than an emergent property of scale.

\paragraph{Implication for SMGI.}
These philosophical constraints jointly justify the SMGI discipline:
(i) closure must be defined over typed regime transformations,
(ii) stability must be attributed to $(\theta,T_\theta)$ rather than $\theta$ alone,
(iii) capacity control must be explicit rather than assumed from performance,
and (iv) evaluative invariance requires protected cores and admissibility constraints.

\paragraph{Requisite variety and closure constraints.}
Ashby's law of requisite variety \citep{ashby1956introduction}
states that effective control requires internal variety
at least matching environmental variety.
In structural terms,
closure (i) requires that the admissible hypothesis class $\mathcal H$
possess sufficient structural diversity
relative to the regime transformation class $\mathcal T$.
Without such matching variety,
stability (ii) and evaluative invariance (iv)
cannot be maintained under perturbation.

% Exemple humain : une personne confrontée à des environnements variés
% doit posséder une flexibilité cognitive suffisante pour préserver
% son identité tout en s’adaptant.

Cybernetic stability principles \citep{wiener1948cybernetics}
further justify attributing stability (ii)
to feedback-structured dynamics $(\theta,T_\theta)$
rather than static architecture.
In particular, evaluator preservation requires
closed-loop constraints ensuring bounded propagation
of perturbations across representational and memory components.

% Exemple humain : rétroaction émotionnelle ou sociale
% corrige des comportements avant qu’ils ne dévient durablement.
This perspective is aligned with Marr's levels of analysis \citep{marr1982vision},
which distinguish computational, algorithmic, and implementational levels.
SMGI operates at an extended computational level:
it specifies structural admissibility constraints
independently of particular learning algorithms or substrates.

Accordingly, SMGI does not assume alignment between representational power and goal stability,
nor between empirical success and semantic commitment.
It instead imposes a structural discipline:
evaluators must be explicitly protected,
semantic admissibility must be typed,
closure must match environmental variety,
and stability must be ensured by bounded feedback dynamics.
Philosophical analysis thus functions not as commentary,
but as a generator of structural admissibility constraints.

%%%%%%%%%%%%%%%%%%%%%%%%%%%%%%%%%%%%%%%%%%%%%%%%%%%%%%%%%%%%%%%%%%%%%%%%
\subsection{Synthesis and Transition to Formal Meta-Structure}
%%%%%%%%%%%%%%%%%%%%%%%%%%%%%%%%%%%%%%%%%%%%%%%%%%%%%%%%%%%%%%%%%%%%%%%%

The preceding positioning analysis reveals a structural convergence
across predictive modeling, meta-learning, causal reasoning,
cognitive architectures, and large-scale generalist systems.
Despite methodological diversity,
these paradigms recurrently introduce:

\begin{itemize}
\item persistent representational strata,
\item compositional or hierarchical update operators,
\item explicit or implicit memory organization,
\item partial invariance constraints under transformation.
\end{itemize}

\paragraph{Structural convergence without structural unification.}

Let $\mathfrak P_i$ denote each paradigm-specific program class.
Each induces a subclass $\Theta(\mathfrak P_i)\subseteq\Theta$.
While these subclasses exhibit overlapping structural motifs,
none defines a global admissibility operator:

\[
\mathcal A : \Theta \to \{0,1\}
\]

that simultaneously enforces closure (i),
stability (ii),
capacity control (iii),
and evaluative invariance (iv)
within a unified persistent meta-model.

\begin{proposition}[Paradigmatic insufficiency]
For every individual paradigm $\mathfrak P_i$ examined,
there exist realizations $\theta \in \Theta(\mathfrak P_i)$
that satisfy some but not all SMGI obligations.
Hence no examined paradigm,
in isolation,
defines a structurally sufficient admissibility layer.
\end{proposition}

\noindent
This establishes that structural convergence
does not entail structural completeness.

\paragraph{Systemic coherence as admissibility criterion.}

Coherence theories of justification \citep{bonjours1985coherence}
emphasize that adequacy is a property of systems,
not isolated components.
Analogously, admissibility within SMGI
is a property of the coupled tuple:

\[
(\theta, T_\theta)
=
(r,\mathcal H,\Pi,\ell,\mathcal E,\mathcal M, T_\theta),
\]

rather than any single structural component.
Admissibility therefore requires global compatibility
between representation, update, memory,
and evaluator preservation.

\paragraph{Elevation to explicit structural object.}

SMGI formalizes this missing layer
by elevating distributed requirements
to a typed structural object:

\[
\theta \in \Theta
\quad \text{such that} \quad
(\theta,T_\theta) \in \mathcal A.
\]

Here $\mathcal A$ denotes the admissibility class
defined by closure, stability,
capacity control, and evaluative invariance.
This shifts the analysis
from paradigm-specific heuristics
to meta-level structural constraints.

\paragraph{Transition.}

The positioning section has therefore demonstrated:

\begin{enumerate}
\item empirical convergence across traditions,
\item theoretical insufficiency of isolated paradigms,
\item necessity of an explicit admissibility layer.
\end{enumerate}

We now proceed to construct this layer formally.
The next section introduces the meta-structure of $\Theta$
and specifies the admissibility operator $\mathcal A$
as a mathematically explicit constraint on $(\theta,T_\theta)$.

\section{Deep Comparative Analysis with Frontier AI Systems (2023--2026)}

\subsection{Overview}

Recent advances in AI (2023--2026) have produced systems exhibiting unprecedented performance across language, multimodality, reasoning, planning, and tool use. These include:

\begin{itemize}
\item GPT-4 and GPT-4o \cite{openai2023gpt4,openai2024gpt4o},
\item Claude 3 family \cite{anthropic2024claude3},
\item Gemini 1.5 \cite{gemini2024report},
\item DeepSeek-V3 \cite{deepseek2024v3},
\item Qwen2.5-Omni \cite{qwen2024omni},
\item Generalist agents such as Gato \cite{reed2022gato},
\item Tool-augmented LLM agents \cite{yao2023react},
\item World-model systems (MuZero, DreamerV3) \cite{schrittwieser2020muzero,hafner2023dreamerv3}.
\end{itemize}

While these systems significantly expand task coverage and scaling capabilities, we show that they remain instances of structurally fixed evaluative regimes.

Our framework subsumes them as special cases.
Our comparative claims below rely on the strict inclusion result proved in Theorem~\ref{thm:strict-structural-inclusion} (Section~\ref{subsec:strict-inclusion}).

\subsection{Large Language Models (LLMs)}

Modern LLMs optimize:

\[
\min_{\theta} \mathbb{E}_{x \sim D} \left[-\log p_\theta(x)\right].
\]

Alignment methods (RLHF, DPO) introduce preference-based losses:

\[
\max_\theta \mathbb{E}_{(x,y^+,y^-)} \log \sigma\left( \beta (f_\theta(x,y^+) - f_\theta(x,y^-)) \right)
\]

\cite{rafailov2023dpo}.

However:

\begin{itemize}
\item The representation $r$ is fixed post-training.
\item The loss $\ell$ is externally specified.
\item Regime plurality is implicit (prompt-driven), not structurally modeled.
\item Memory/forgetting is architectural (context window), not evaluative.
\end{itemize}

In our notation:

\[
K = 1, \quad \sigma \text{ trivial}, \quad \mathcal{M} \text{ fixed}.
\]

Thus LLMs correspond to a degenerate structural configuration:

\[
\theta_{LLM} = (r, \mathcal{H}, \Pi, \mathcal L, \mathcal{E}, \mathcal{M}_0),
\qquad \mathcal L=\{\ell\}.
\]

\subsection{Tool-Augmented Agents}

ReAct-style agents \cite{yao2023react} interleave reasoning and tool calls.

These systems expand operational capacity but retain:

\begin{itemize}
\item single reward/evaluation structure,
\item no explicit modeling of antagonistic regimes,
\item no learned regime commutation.
\end{itemize}

They correspond to:

\[
\mathcal{H}_{agent} = \mathcal{H}_{LLM} \circ \text{Tools}
\]

without structural modification of evaluation layer.

\subsection{World Models and Model-Based RL}

MuZero \cite{schrittwieser2020muzero} learns:

\[
(\hat{f}, \hat{g}, \hat{r})
\]

representing dynamics and reward.

DreamerV3 \cite{hafner2023dreamerv3} optimizes:

\[
\max_\pi \mathbb{E} \left[ \sum_t \gamma^t r_t \right].
\]

Key property:

Reward $r_t$ defines a single evaluative regime.

No structural mechanism for switching evaluative regimes exists.

Thus:

\[
K = 1, \quad \sigma \text{ undefined}.
\]

Our framework strictly generalizes by allowing:

\[
\{r^{(k)}_t\}_{k=1}^K
\]

with context-dependent switching.

\subsection{AIXI and Universal Agents}

AIXI \cite{hutter2005} optimizes:

\[
a_t = \arg\max_a \mathbb{E}_\xi \left[ \sum_{k=t}^\infty r_k \right]
\]

where $\xi$ is a Solomonoff mixture.

This is epistemically universal but axiomatically tied to a fixed reward.

Our extension:

\begin{itemize}
\item treats reward as regime-dependent,
\item allows meta-learning of regime structure,
\item introduces memory/forgetting operators.
\end{itemize}

Hence:

AIXI $\subset $ Structural AGI.

\subsection{Gödel Machines and Self-Improving Systems}

Gödel Machines \cite{schmidhuber2007godel} allow self-rewriting if provably beneficial.

However:

\begin{itemize}
\item Objective function remains fixed.
\item Evaluative meta-structure is not pluralistic.
\end{itemize}

Structural AGI introduces:

Learning of evaluative frames themselves.

\subsection{Mixture-of-Experts and Modular Systems}

Mixture-of-Experts (MoE) \cite{shazeer2017outrageously} approximate conditional computation.

They implement:

\[
h(x) = \sum_k g_k(x) h_k(x)
\]

but do not represent evaluative incompatibility or norm-switching.

MoE gating $\not=$%≠
 regime commutation.

Our $\sigma$ operates at evaluative level, not only representational.

\subsection{Continual Learning and Catastrophic Forgetting}

Continual learning \cite{parisi2019continual} addresses parameter stability.

However:

Forgetting is treated as optimization pathology.

Our framework:

Forgetting becomes functional operator:

\[
\tilde{m}_{t+1} = F(m_{t+1}; k_t, c_t)
\]

This introduces normative forgetting distinct from catastrophic forgetting.

% duplication :\subsection{Strict Inclusion Theorem (Structural)}

%\begin{theorem}[Strict Structural Inclusion]
%Every classical learning system $S$ can be represented as a structural configuration
%$\theta = (r, \mathcal{H}, \Pi, \mathcal L, \mathcal{E}, \mathcal{M})$
%with:

%\[
%K = 1, \quad \sigma \text{ trivial}, \quad F = \text{identity}.
%\]

%There exist structural configurations with $K > 1$ and non-trivial $\sigma, F$ that cannot be reduced to any single-regime classical system.
%\end{theorem}

%\begin{proof}[Sketch]
%If $K=1$ and $\sigma$ trivial, the system reduces to classical ERM/RL.

%Assume existence of antagonistic regimes $k_1, k_2$ such that no single loss $\ell$ reproduces both evaluations simultaneously.

%Then no reduction to a single evaluative regime exists.

%Thus inclusion is strict.
%\end{proof}

\subsection{Positioning Statement}

Current frontier systems expand scale and task coverage.

Our framework expands structural expressivity.

Scale $\not=$%≠ 
structural generality.

Structural generality requires:

\begin{itemize}
\item evaluative plurality,
\item contextual commutation,
\item operational memory control,
\item meta-level capacity regulation.
\end{itemize}

This constitutes a structural generalization of learning systems.

The following theorem is included as an illustrative corollary of the SMGI admissibility obligations, stated here to connect the formal criteria to the comparative discussion. It does not introduce new primitives beyond those defined in the main framework; rather, it packages the closure/invariance consequences in a form directly usable for interpreting frontier-system design choices.

\begin{theorem}[Structural Closure under Task Transformation]
\label{thm:structural-closure}

Let $\theta = (r,\mathcal H,\Pi,\mathcal L,\mathcal E,\mathcal M)$
be a structured learning system satisfying:

(i) bounded statistical capacity of $\mathcal H$,

(ii) non-expansive and dissipative memory dynamics on $\mathcal M$ (existence of a bounded absorbing set under admissible updates),

(iii) existence of an invariant evaluative subspace of $\ell$,

(iv) stability of policy update operators induced (and regulated) by $\Pi$.

Then there exists an invariant subset
$S^* \subseteq S$
such that for every admissible task transformation
$\tau \in \mathcal T$,

\[
T_\tau(S^*) \subseteq S^*,
\]

and the induced dynamics remains
(i) statistically bounded,
(ii) Lyapunov-stable,
(iii) evaluatively invariant.
Conversely, removal of either persistent memory
or invariant evaluative structure eliminates
at least one of these invariance properties.
\end{theorem}

\begin{proof}
Let $S = \mathcal R_\theta \times \mathcal H(\theta) \times \mathcal P(\theta) \times \mathcal M(\theta)$ be the composite metric state space induced by $\theta$. We assume $S$ is a complete metric space. We construct the invariant subset $S^*$ constructively.

\paragraph{Step 1: Construction of the constrained subspaces.}
By hypothesis (i), the bounded statistical capacity implies that admissible hypotheses are confined to a bounded subset $B_{\mathcal H} \subset \mathcal H(\theta)$ for any given complexity functional. 
By hypothesis (iii), the existence of an invariant evaluative subspace implies there exists a projection operator $P_{\mathcal C}$ mapping evaluations to a normative core $\mathcal C$, such that $P_{\mathcal C}(\ell_{t+1}) = P_{\mathcal C}(\ell_t)$ for all valid transitions. We define the evaluatively coherent manifold $\mathcal L^* = \{s \in S \mid P_{\mathcal C}(\ell(s)) \text{ is constant} \}$.
By hypothesis (iv), the stability of the update operators $\Pi$ implies that policy updates are non-expansive under admissible transformations. Thus, there exists a closed, bounded region $\mathcal P^* \subset \mathcal P(\theta)$ invariant under $\Pi$.

\paragraph{Step 2: Memory non-expansiveness and invariant set.}
By hypothesis (ii), for each $\tau\in\mathcal T$ the memory update operator $U_\tau$ is non-expansive on $(\mathcal M(\theta), d_{\mathcal M})$, i.e.,
$d_{\mathcal M}(U_\tau(m_1),U_\tau(m_2))\le d_{\mathcal M}(m_1,m_2)$,
and is dissipative in the sense that it admits a bounded absorbing set $B_{\mathcal M}\subset \mathcal M(\theta)$ invariant under admissible updates.
Consequently, the memory dynamics admit a bounded invariant set $\mathcal M^*\subseteq B_{\mathcal M}$ (possibly non-singleton) that captures persistent task-conditioned memory states under admissible sequences.

\paragraph{Step 3: Structural Closure under $T_\tau$.}
We define the candidate invariant subset as the intersection of these stable regions:
\[
S^* \;=\; \big(\mathcal R_\theta \times B_{\mathcal H} \times \mathcal P^* \times \mathcal M^*\big) \;\cap\; \mathcal L^*.
\]
By construction, $\mathcal M^*$ is an attractor for the memory dynamics, $B_{\mathcal H}$ strictly bounds the hypothesis capacity, $\mathcal P^*$ is stable under policy updates, and the trajectory is geometrically restricted to the evaluative manifold $\mathcal L^*$. Therefore, the coupled transition operator $T_\tau$ maps $S^*$ into itself. Thus, $T_\tau(S^*) \subseteq S^*$, satisfying structural closure. The induced dynamics trivially remain statistically bounded and evaluatively invariant by containment in $B_{\mathcal H}$ and $\mathcal L^*$, and Lyapunov-stable by invariance of $S^*$ together with the assumed stability of $\Pi$ on $\mathcal P^*$ and the bounded absorbing behavior of the memory operator on $\mathcal M^*$.

\paragraph{Step 4: Necessity of components (Conversely).}
If persistent memory $\mathcal M$ is removed, the memory operator $U_\tau$ (and thus the dissipative/non-expansive control on the memory component) ceases to exist, allowing unconstrained drift in the coupled dynamics, which can violate Lyapunov stability (ii). If the invariant evaluative structure is removed, the constraint manifold $\mathcal L^*$ dissolves, and the system's normative core diverges under $\tau$, violating evaluative invariance (iii).
\end{proof}

\begin{remark}[Non-emptiness of the admissible set]
\label{rem:nonempty-Sstar}
In highly chaotic environments or under very strict protected constraints, the intersection defining $S^*$ may be empty in practice, yielding a vacuous admissibility claim. This does not refute the framework: it highlights that making $S^*$ nonempty is an \emph{algorithmic design obligation} (e.g., via constraint softening, hierarchical priorities, or certified ``safe-abstain'' modes) rather than a guaranteed property of arbitrary deployments.
\end{remark}

\section{Discussion and Outlook: From Parametric Scaling to Structural Admissibility}

This work formalizes a perspective in which generality is characterized by certified behavior under task transformations, rather than by performance improvements under a fixed evaluation interface. The Structural Model of General Intelligence (SMGI) makes the learning interface explicit as a typed meta-structure and shifts the analysis from optimizing within a fixed hypothesis class to studying stability and capacity of the induced coupled dynamics.

The Strict Structural Inclusion Theorem provides a precise basis for this perspective: classical single-regime systems embed as a restricted case, while multi-regime evaluative structure together with certified transformations yields behaviors that cannot be reduced to any single-regime formulation. This clarifies how to interpret frontier-system design choices in SMGI terms (e.g., exogenous evaluators and largely fixed interfaces) and which structural components are required to reason about admissible evolution, cross-regime coherence, and capacity control.
%%%%%%%%%%%%%%%%%%%%%%%%%%%%%%%%%%%%%%%%%%%%%%%%%%%%%%%%%%%%%%%%%%%%%%%%
\subsection{Meta-Admissibility Beyond the SMGI Kernel (Roadmap)}
\label{sec:meta-admissibility-beyond-kernel}
%%%%%%%%%%%%%%%%%%%%%%%%%%%%%%%%%%%%%%%%%%%%%%%%%%%%%%%%%%%%%%%%%%%%%%%%

The structural meta-model $\theta=(r,\mathcal H,\Pi,\mathcal L,\mathcal E,\mathcal M)$ introduced in this paper is a
\emph{necessary but non-exhaustive kernel} for structurally coherent general intelligence.
The purpose of this article was to establish this kernel rigorously: closure, stability, bounded capacity, and evaluative invariance
as admissibility obligations on the coupled realization $(\theta,T_\theta)$.
Beyond this kernel, a fully deployed general intelligence system must typically satisfy additional \emph{higher-order}
admissibility layers. We record them here explicitly as a roadmap, not as additional claims proved in this paper.

\paragraph{Part I: A hierarchy of admissibility layers (extensions beyond the kernel).}
The extensions form a \emph{hierarchy of admissibility} above the SMGI kernel:
\begin{enumerate}[label=(H\arabic*),leftmargin=*]
\item \textbf{Structural admissibility (SMGI kernel).}
The four obligations of this paper: typed closure, certified stability, bounded statistical capacity, and evaluative invariance.

\item \textbf{Inter-regime coherence.}
When evaluation is plural ($\mathcal L=\{\ell_k\}_{k=1}^K$), admissibility additionally requires bounded divergence between regimes,
through a typed divergence $\mathrm{D}_{\mathrm{eval}}$ and a global coherence functional
\[
\Gamma(\{\mathcal E_c\}) := \sup_{c,d\in\mathcal C}\mathrm{D}_{\mathrm{eval}}(\mathcal E_c,\mathcal E_d),
\qquad
\Gamma(\{\mathcal E_c\})\le \varepsilon,
\]
so that evaluative plurality does not collapse identity.

\item \textbf{Multi-scale dynamical compatibility.}
The induced dynamics may decompose across time scales (fast adaptation, contextual adjustment, identity-level evolution),
schematically $T_\theta = T_{\mathrm{fast}}\circ T_{\mathrm{context}}\circ T_{\mathrm{identity}}$ with
$\tau_{\mathrm{fast}} \ll \tau_{\mathrm{context}} \ll \tau_{\mathrm{identity}}$,
and admissibility requires bounded incompatibility between scales (commutation up to bounded perturbation).

\item \textbf{Stratified memory with controlled interactions.}
Memory is commonly stratified $\mathcal M=\bigoplus_k \mathcal M_k$, and admissibility constrains cross-stratum operators
(e.g., bounded inter-stratum influence $\|\Pi_{k\to j}\|\le \delta_{kj}$) and forgetting operators so that protected evaluative cores
cannot be violated by erasure.

\item \textbf{Identity-level invariance.}
A further layer treats identity as the persistence of designated invariants under admissible evolution, e.g.,
\[
\mathrm{Id}(\theta) := \bigcap_t \mathrm{Inv}(T_\theta^t),
\]
capturing long-horizon coherence across regime switching and structural revision.
\end{enumerate}
These layers are not proposed as a closed list; they define an explicit \emph{meta-admissibility program} above the kernel.

\paragraph{Part II: Two complementary readings (categorical vs.\ dynamical).}
The same hierarchy admits two non-exclusive scientific readings:
(i) a \emph{categorical} reading, where regimes and structural configurations are objects and admissible transformations are morphisms
preserving invariants (typed closure and invariant-preservation constraints);
(ii) a \emph{dynamical} reading, where admissibility is characterized by stability and attractor preservation under bounded perturbations,
with Lyapunov witnesses and multi-scale compatibility constraints.
In this paper, we focused on the kernel level: we made the structural object explicit and showed how closure, stability,
capacity control, and evaluative invariance jointly yield non-vacuous guarantees under admissible transformations.
The higher layers are recorded here to clarify scope and to define a coherent research program for subsequent work.

\paragraph{Orthogonal dynamical modeling and scope of the induced dynamics.}
The present article isolates the structural axis of general intelligence.
The transition operator $T_\theta$ introduced above represents a behavioral semantics compatible with the structural meta-model $\theta$.
It abstracts the learning, adaptation, and interaction processes into a single induced dynamical system acting on $S_\theta$.
This abstraction deliberately suppresses several layers of possible complexity, including:
(i) coexistence of multiple concurrent representational regimes,
(ii) parallel or interacting dynamical operators,
(iii) meta-level self-modification across multiple $\theta$-instances,
(iv) asynchronous or multi-timescale evolution.
A fully general theory of multi-regime and multi-structure temporal interaction would require an additional behavioral meta-model,
orthogonal to the present structural axis. Such a theory would formalize families of interacting transition systems,
possibly indexed by regime or time-scale parameters. It will be presented in future work.

The contribution of this theory is twofold:
The contribution of this theory is twofold:
\begin{itemize}[leftmargin=*]
\item \textbf{Epistemic unification.} It bridges statistical learning theory (SLT) and the formal requirements of general agency, providing a unified formal framework where capacity control and evaluative invariance coexist.

\item \textbf{Architectural admissibility.}  It defines the minimal criteria—closure, stability, and normative core preservation—under which a persistent, cross-domain autonomous system can be considered well-formed.
\end{itemize}  

In sum, we do not propose a specific algorithm; rather, we provide a formal kernel: a typed meta-model and admissibility obligations under which structural evolution, capacity control, and evaluative coherence can be stated and certified. The contribution is therefore primarily theoretical: it isolates sufficient structural conditions for stable cross-regime behavior and clarifies which additional components (e.g., concrete meta-dynamics, evaluator update mechanisms, and memory instantiations) are required for full implementations.
\paragraph{Open challenge: metrics on environment families and Lipschitz representations.}

Definition~\ref{def:admissible_transformations} assumes a task/environment metric $D_{\mathcal E}$ and a Lipschitz continuity requirement
for the interface $r$ with respect to that metric.
For realistic modalities (high-dimensional vision, discrete text, tool-augmented interaction),
constructing a meaningful $D_{\mathcal E}$ and certifying global (or regime-uniform) Lipschitz properties for modern non-convex encoders
remains a major open challenge.
In future instantiations, this requirement will likely be enforced via restricted transformation classes, local Lipschitz certificates,
or implementation-level robustness bounds rather than global Lipschitz constants.
\section{Structural Analysis of the SMGI Architecture}

The SMGI architecture illustrated in Figure~\ref{fig:smgi_architecture} is not merely a computational diagram but a structural instantiation of the meta-structure
\[
\theta = (r, \mathcal{H}, \Pi, \mathcal L, \mathcal{E}, \mathcal{M}).
\]

Each architectural module corresponds to a formally defined component of the theory: $\mathcal{H}$ specifies the representational substrate, $\Pi$ encodes decision policies, $r$ regulates admissible adaptation, $\mathcal{E}$ defines the expansion regime, and $\mathcal{M}$ ensures structural persistence across transformations.

A rigorous analysis of this coupling relies on four fundamental articulations.

\paragraph{1. Control Hierarchy and Structural Channeling}

The architecture enforces a strict separation between the \textbf{Structural Ontology} $\theta$ (meta-level) and the \textbf{Behavioral Semantics} $T_\theta$ (execution level). 

The meta-structure $\theta$ operates as a second-order regulator: the representation operator $r$ and the structural prior $\Pi$ define the conditions of possibility of learning. The downward arrow in the diagram symbolizes \emph{structural channeling}, ensuring that the update dynamics of hypotheses $h \in \mathcal{H}$ remain confined within prescribed complexity bounds. This prevents uncontrolled statistical divergence and preserves admissibility under expansion.

\paragraph{2. Closure under Task Transformations ($\tau$)}

Interaction with the environment $\mathcal{E}$ is dual. In addition to the classical flow of observations $z$, the system explicitly processes task transformations $\tau \in \mathcal{T}$. 

Unlike classical models that experience such transformations as distribution shifts, the SMGI framework guarantees \textbf{structural closure}: the system adapts its internal operators while preserving its functional identity. This resilience is represented by the feedback loop between the environment and the structural ontology.

Structural inclusion therefore implies not only adaptation to data, but closure under admissible task transformations.

\paragraph{3. Regime Topology and Normative Navigation}

The right-hand module introduces a \textbf{normative plurality}. The contextual commutator $\sigma$ arbitrates among multiple evaluation regimes $\ell_1, \dots, \ell_k \in \mathcal{L}$. 

This topological structure allows the system to modify its optimization criteria without compromising global coherence. The evaluation feedback loop toward $T_\theta$ indicates that behavioral adjustment is subordinated to regime selection, formalizing the evaluative invariance required for general intelligence.

The architecture thus represents regime coexistence through an explicit multi-regime evaluative structure.

\paragraph{4. Lyapunov Stability and the Invariant Manifold $S^*$}

The region delimited by the dotted boundary, denoted $S^*$, represents the space of admissible configurations. The rigor of the SMGI model lies in the demonstration that, under the constraints imposed by $\theta$, system trajectories are \textbf{Lyapunov-stable}.

Regardless of environmental perturbations or regime transitions, the system state converges toward, or remains within, this invariant manifold. In particular, unstable drift of memory $\mathcal{M}$ or agent parameters is excluded by construction.

This invariant set operationalizes structural admissibility at the dynamical level: strict structural inclusion is compatible with stability if and only if trajectories remain confined to $S^*$.

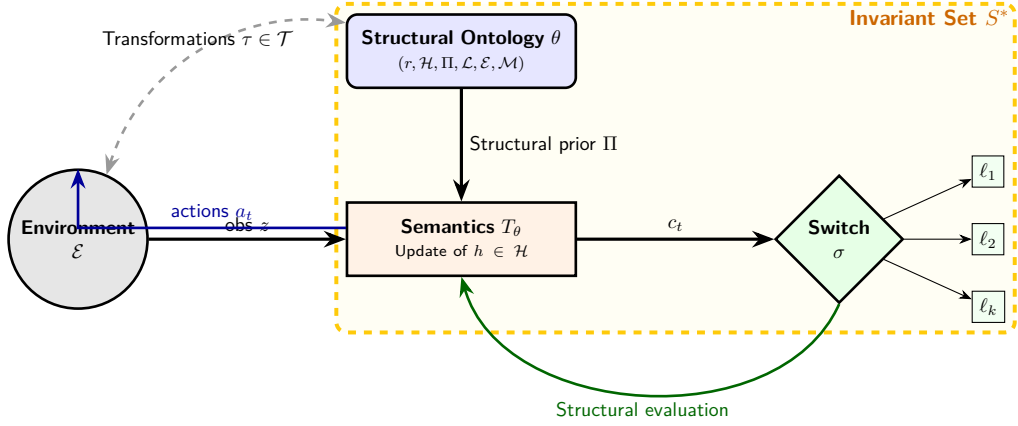
\begin{figure}[ht]
\centering
% On ajuste l'échelle pour éviter le Overfull \hbox (scale=0.8 ou 0.9)
\begin{tikzpicture}[
    scale=0.75, transform shape,
    node distance=2cm,
    font=\sffamily\small,
    >=Stealth,
    % L'option align=center est CRUCIALE pour autoriser le \\ dans les nodes
    ontology/.style={rectangle, draw, fill=blue!10, text width=3.8cm, align=center, minimum height=1.3cm, rounded corners, line width=1pt},
    dynamics/.style={rectangle, draw, fill=orange!10, text width=3.8cm, align=center, minimum height=1.3cm, line width=1pt},
    regime/.style={diamond, draw, fill=green!10, text width=1.6cm, align=center, inner sep=0pt, line width=1pt},
    env/.style={circle, draw, fill=gray!20, minimum size=2.2cm, align=center, line width=1pt},
    invariant_box/.style={draw=yellow!60!orange, fill=yellow!5, dashed, line width=1.5pt, rounded corners}
]

    % --- COUCHE ONTOLOGIQUE (θ) ---
\node[ontology] (theta) {\textbf{Structural Ontology} $\theta$ \\ \scriptsize $(r,\mathcal H,\Pi,\mathcal L,\mathcal E,\mathcal M)$};
    
    % --- COUCHE DYNAMIQUE (Tθ) ---
\node[dynamics, below=2cm of theta] (dynamics) {\textbf{Semantics} $T_\theta$ \\ \scriptsize Update of $h \in \mathcal{H}$};
    
    % --- ENVIRONNEMENT (Correction de l'erreur \\ ici) ---
 \node[env, left=3.5cm of dynamics] (env) {\textbf{Environment} \\ $\mathcal{E}$};
    
    % --- TOPOLOGIE DES RÉGIMES ---
    \node[regime, right=3.5cm of dynamics] (sigma) {\textbf{Switch} \\ $\sigma$};
    
    % Blocs de régimes
    \node[rectangle, draw, fill=green!5, right=1.2cm of sigma, yshift=1.2cm] (L1) {$\ell_1$};
    \node[rectangle, draw, fill=green!5, right=1.2cm of sigma] (L2) {$\ell_2$};
    \node[rectangle, draw, fill=green!5, right=1.2cm of sigma, yshift=-1.2cm] (Lk) {$\ell_k$};

    % --- FLUX ET INTERACTIONS ---
    \draw[->, line width=1.2pt] (env) -- node[above] {obs $z$} (dynamics);
    \draw[<->, dashed, line width=1pt, color=gray!80] (env) to [bend left=40] node[above, color=black] {Transformations $\tau \in \mathcal{T}$} (theta);
   \draw[->, line width=1.2pt] (theta) -- node[right] {Structural prior $\Pi$} (dynamics);

    % (FIX) Action channel from semantics to environment
    \draw[->, line width=1pt, color=blue!60!black]
    (dynamics.west) ++(0,0.2) -| node[pos=0.25, above] {actions $a_t$} (env.north);

    \draw[->, line width=1.2pt] (dynamics) -- node[above] {$c_t$} (sigma);
    \draw[->] (sigma) -- (L1); \draw[->] (sigma) -- (L2); \draw[->] (sigma) -- (Lk);

    % (FIX) Evaluation feedback should come from the selector/aggregator, not a single regime
    \path (sigma.south) edge [bend left=70, ->, line width=1pt, color=green!40!black] node[below] {Structural evaluation} (dynamics.south);

    % --- CADRE DE FERMETURE (S*) ---
    \begin{scope}[on background layer]
        \node[invariant_box, fit=(theta) (dynamics) (sigma) (L1) (Lk) (L2)] (Sstar) {};
    \end{scope}
    
    \node[anchor=north east, color=orange!80!black, font=\sffamily\bfseries] at (Sstar.north east) {Invariant Set $S^*$};

\end{tikzpicture}
\caption{Formal architecture of the SMGI meta-model.}
\label{fig:smgi_architecture}
\end{figure}

\paragraph{Figure~\ref{fig:smgi_architecture} (short reading).}
Figure~\ref{fig:smgi_architecture} summarizes the SMGI coupling between the structural ontology
$\theta=(r,\mathcal H,\Pi,\mathcal L,\mathcal E,\mathcal M)$ and its induced semantics $T_\theta$.
The environment $\mathcal E$ provides observations $z_t$ to $T_\theta$ (perception channel), and $T_\theta$ returns actions $a_t$
to $\mathcal E$ (interaction channel), forming the closed loop $\mathcal E \xrightarrow{z_t} T_\theta \xrightarrow{a_t} \mathcal E$.
The context signal $c_t$ produced by $T_\theta$ drives the regime switch $\sigma$, which selects/weights evaluators
$\ell_i\in\mathcal L$; the resulting structural evaluation feeds back into $T_\theta$ to shape updates under the structural prior $\Pi$.
The dashed region $S^*$ denotes the admissible invariant set in which the coupled dynamics is intended to remain under the paper’s conditions.

\section{Conclusion}

This work introduced a formal theory of Artificial General Intelligence (AGI) grounded in structural invariance under expanding families of environments.

A key message is that general intelligence cannot be characterized solely by scale, benchmark aggregation, or task diversity; in this work we characterize it in structural terms.

We formalized AGI through the meta-structure
\[
\theta = (r, \mathcal{H}, \Pi, \mathcal L, \mathcal{E}, \mathcal{M}),
\]
and proposed strict structural inclusion
\[
\theta_1 \subsetneq \theta_2
\quad \text{whenever} \quad
\mathcal{E}_1 \subsetneq \mathcal{E}_2,
\]
under preservation of bounded structural risk, as a formal and falsifiable criterion for general intelligence.

This definition reframes the theoretical foundations of generalization. Classical learning theory studies generalization within fixed distributions; here we study invariance under structural expansion. Under this perspective, intelligence is characterized by the capacity of an internal structure to admit certified extensions while preserving stability and bounded risk.

\medskip
\noindent
\textbf{Stable Structural Extensibility and Structural Invariants.}

Structural growth is defined as a stability-preserving operation. Expansions must conserve bounded risk and preserve embedding relations
\[
\mathcal{H}_1 \hookrightarrow \mathcal{H}_2,
\quad
\Pi_1 \hookrightarrow \Pi_2,
\]
ensuring that previously acquired competencies remain structurally integrated.

The framework introduces structural invariants, which may be represented as constraint sets
\[
\mathcal{I}(\theta),
\]
preserved under admissible expansions and memory dynamics. These invariants function as conservation principles: extension is permitted only if coherence, admissibility, and boundedness are maintained. Strict structural inclusion therefore enforces compatibility constraints between successive meta-structures and prevents incoherent structural drift under expansion.

Under the assumptions specified in the main results, these invariance conditions provide sufficient structural guarantees for stability under controlled growth.

\medskip
\noindent
\textbf{Coexistence of Multiple Regimes and Cross-Regime Coherence.}

Through the commutation operator $\sigma$ and the indexed family of environment expansions $\{\mathcal{E}_i\}$, the model admits coexistence of heterogeneous structural regimes within a unified meta-structure. Rather than forcing specialization or regime replacement, the agent maintains parallel structural representations whose compatibility is ensured by strict inclusion.

This yields a mathematically grounded account of cross-domain generality without structural fragmentation. Intelligence emerges as a dynamically stable hierarchy of compatible regimes rather than a sequence of disconnected adaptations.

\medskip
\noindent
\textbf{Stratified Memory and Controlled Memory Dynamics.}

A central component of the framework lies in the formal role of memory within $\theta$. Memory is stratified and dynamically regulated:
\[
\mathcal{M} = \bigcup_k \mathcal{M}^{(k)}.
\]

Lower strata encode parametric and representational information. Higher strata encode structural relations, regime embeddings, invariants $\mathcal{I}(\theta)$, admissibility constraints, and normative structures. Update and forgetting operators act differentially:
\[
U = \{U^{(k)}\}_k, \qquad F = \{F^{(k)}\}_k,
\]
allowing selective plasticity while preserving higher-level invariants.

Higher-level structural memory constrains lower-level parameter updates, thereby regulating plasticity and preventing destructive drift across expansions.

\medskip
\noindent
\textbf{Catastrophic Forgetting and Continual Learning.}

Within this framework, catastrophic forgetting is characterized as a violation of stratified structural constraints: lower-level updates eliminate information required to preserve higher-level invariants or regime embeddings.

By explicitly enforcing invariant preservation across memory strata, the model provides a structural characterization and a principled route toward mitigating catastrophic forgetting. In particular, invariant-preserving memory dynamics constitute sufficient structural conditions under which destructive forgetting is avoided.

Continual learning is therefore reinterpreted as maintaining strict structural inclusion under controlled stratified memory dynamics. Stability is achieved not by freezing parameters, but by preserving structural embeddings and invariant subspaces across expansions.

More broadly, stratified memory enables multi-regime reasoning: regime-specific memories can be activated, combined, or commuted while preserving global coherence. Intelligence thus emerges not merely from representational capacity, but from controlled interaction between memory strata under structural constraints.

\medskip
\noindent
\textbf{Alignment, Norm Compliance, and Safety by Construction.}

Because hypothesis-policy spaces and environment families are explicitly defined within $\theta$, admissible regions can be specified a priori:
\[
\mathcal{H}_{\mathrm{adm}} \subseteq \mathcal{H},
\quad
\Pi_{\mathrm{adm}} \subseteq \Pi.
\]

Under admissible expansions, strict structural inclusion ensures that growth remains within these predefined admissible manifolds. Alignment is therefore embedded structurally rather than enforced post hoc; it functions as a preserved structural invariant. Expansions that violate admissibility are formally excluded from the inclusion relation.

More generally, let $\mathcal{C}$ denote a set of normative constraints inducing admissible subspaces
\[
\mathcal{H}_{\mathcal{C}} \subseteq \mathcal{H},
\quad
\Pi_{\mathcal{C}} \subseteq \Pi,
\quad
\mathcal{E}_{\mathcal{C}} \subseteq \mathcal{E}.
\]

An agent satisfies structural norm compliance if, under admissible expansions
\[
\mathcal{E}_1 \subsetneq \mathcal{E}_2 \subseteq \mathcal{E}_{\mathcal{C}},
\]
strict structural inclusion preserves containment within $(\mathcal{H}_{\mathcal{C}}, \Pi_{\mathcal{C}})$.

If normative systems evolve, constraint sets $\mathcal{C}_t$ can be incorporated into higher memory strata, preserving compatibility across normative transitions and preventing abrupt misalignment.

This leads to a notion of \emph{admissible AGI}: strict structural inclusion under expansion while remaining within predefined admissible subspaces.

\medskip
\noindent
\textbf{Unified Structural Perspective on Core Challenges.}

Within this structural formulation, several longstanding challenges appear as instances of a single principle — preservation of admissible invariants under expansion:

\begin{itemize}
\item Distributional and adversarial robustness through $\sigma$ and bounded structural risk;
\item Cross-domain transfer via compatible regime coexistence;
\item Alignment drift under scaling constrained by admissible subspaces;
\item Safety under self-modification governed by invariant preservation.
\end{itemize}

General intelligence, stability, alignment, and safety are therefore not independent objectives but interconnected structural properties.

\medskip
\noindent
\textbf{Falsifiability, Scope, and Limits.}

The proposed criterion is architecture-independent and empirically testable: structural inclusion can be evaluated through controlled environment-family growth protocols, transforming AGI certification into a formally defined scientific question (i.e., the claim is empirically falsified for a given instantiation if an admissible task expansion $\tau$ persistently violates its protected normative core, indicating that no stable admissible invariant set is maintained under that expansion).

\noindent\textbf{Protocol sketch.}
Grow an environment family $\{\mathcal E_m\}_{m\ge 1}$ along a single structural axis (e.g., regime-switch frequency or evaluator antagonism) and test whether an SMGI instantiation remains admissible while no tuned $K=1$ baseline can satisfy all admissibility criteria (SMGI-(ii) stability, SMGI-(iii) bounded risk/capacity, and SMGI-(iv) evaluative coherence) under the same growth and under matched training/evaluation resources (data, compute, and interaction budget).

The theory reframes AGI from a scaling problem to a structural problem. It replaces the question “How many tasks can a system solve?” with:

\begin{center}
\emph{Does the system admit stable, admissible, strict structural inclusion under expanding families of environments?}
\end{center}

At the same time, several challenges remain open, including the computational tractability of verifying structural inclusion, the operational construction of environment families $\mathcal{E}$, tight guarantees under adversarial expansion, and empirical realization of structural-growth benchmarks.

\paragraph{On the role of assumptions as contractual obligations.}
Assumptions such as the existence of (i) a sequential PAC-Bayes inequality adapted to the natural filtration and
(ii) a Lyapunov-style drift witness are not presented as ``easy conditions'' to satisfy.
Rather, they formalize \emph{non-negotiable contractual obligations} for any future algorithmic instantiation that claims
stable structural evolution: the framework makes explicit \emph{what must be certified} (capacity control under dependence and bounded drift),
so that ``generalization under interface change'' does not collapse into informal narrative.
This shifts the scientific burden from asserting robustness by scale to specifying verifiable certificates and admissible update rules.

Under the assumptions formalized in this work, strict structural inclusion provides sufficient structural conditions for extensibility, coherence, and admissibility. Whether these conditions are also necessary in broader settings remains an open question.

By defining AGI as structure-preserving extensibility governed by invariants, stratified memory, admissible subspaces, and regime coexistence, this work proposes a candidate mathematical foundation for stable, extensible, and norm-compliant artificial general intelligence.

\bibliographystyle{plainnat}
\bibliography{jair_agi_references}

\appendix

\section{Appendix: Rigorous Proof of Structural Generalization under Transformations}
\label{app:smgi-proof}

We provide a detailed derivation of Theorem~\ref{thm:smgi-generalization}. Because the system state $s_t = T_{\theta, \tau}(s_{t-1}, z_t)$ evolves sequentially and depends on past observations, classical i.i.d. PAC-Bayes bounds cannot be applied directly. We must explicitly decouple the statistical hypothesis complexity from the temporal drift using a martingale difference sequence over the natural filtration.

\subsection{Step 1: Martingale Decomposition of the Generalization Gap}
Fix $\tau \in \mathcal T$. Let $Z_{1:n} = (z_1, \dots, z_n)$ be the sequence of observations drawn under $\tau$. Let $\mathcal{F}_i = \sigma(z_1, \dots, z_i)$ be the associated natural filtration. The empirical risk is $\widehat{R}_\tau(Q_\tau) = \frac{1}{n}\sum_{t=1}^n \mathbb{E}_{h \sim Q_\tau}[\ell(s_t, z_t)]$. We decompose the generalization gap into two components:
\[
R_\tau(Q_\tau) - \widehat{R}_\tau(Q_\tau) \;=\; \Delta_{\mathrm{stat}} + \Delta_{\mathrm{drift}}
\]
where $\Delta_{\mathrm{stat}}$ measures the capacity of the posterior $Q_\tau$ conditioned on the stationary trajectory manifold $S^*$, and $\Delta_{\mathrm{drift}}$ captures the algorithmic deviation induced by the dynamic updating of state $s_t$.

\subsection{Step 2: Bounding $\Delta_{\mathrm{drift}}$ via Azuma--Hoeffding and Lyapunov Control}
We bound the deviation induced by stateful dependence using a Doob martingale together with a bounded-differences condition that is explicit and checkable.

\begin{assumption}[Bounded Lyapunov sublevel set on $S^*$]
\label{ass:V-bounded-on-Sstar}
The admissible invariant set $S^*$ is contained in a Lyapunov sublevel set:
\[S^*\subseteq\{s\in S_\theta:\; V(s)\le V_{\max}\}\qquad\text{for some finite }V_{\max}>0.\]
\end{assumption}
\noindent
This assumption is natural in the present setting: since $S^*$ is an invariant set introduced for admissibility, one may always intersect it with a sufficiently large sublevel set of the stability witness $V$ (provided that intersection remains nonempty and invariant under the certified dynamics).

\begin{lemma}[Bounded martingale increments]
\label{lem:bounded-martingale-increments}
Assume closure on $S^*$, Assumption~\ref{ass:V-bounded-on-Sstar}, and that the loss is $L$-Lipschitz in the Lyapunov metric:
\[|\ell(s,z)-\ell(s',z)| \le L\,|V(s)-V(s')|\qquad\text{for all }s,s'\in S^*,\ z\in\mathcal Z.\]
Define the Doob martingale for the empirical risk functional
$\widehat{R}_\tau(Q_\tau)=\frac{1}{n}\sum_{t=1}^n \E_{h\sim Q_\tau}[\ell(s_t,z_t)]$ by
$M_i=\E[\widehat{R}_\tau(Q_\tau)\mid\mathcal F_i]-\E[\widehat{R}_\tau(Q_\tau)\mid\mathcal F_{i-1}]$.
Then for each $i\in\{1,\dots,n\}$, $|M_i|\le \frac{2LV_{\max}}{n}$.
\end{lemma}

\begin{proof}
Fix $i$ and condition on $\mathcal F_{i-1}$. Changing only $z_i$ can affect at most the summands with indices $t\ge i$ in $\widehat{R}_\tau$.
For each such summand, closure ensures $s_t\in S^*$, so $V(s_t)\le V_{\max}$. For any two realizations $s_t,s'_t\in S^*$,
\[|\ell(s_t,z_t)-\ell(s'_t,z_t)|\le L|V(s_t)-V(s'_t)|\le 2LV_{\max}.\]
Averaging over at most $n$ summands and dividing by $n$ yields an influence at most $2LV_{\max}/n$ on $\widehat{R}_\tau$, which is precisely the bound on the Doob increment $M_i$.
\end{proof}

By Azuma--Hoeffding, with probability at least $1-\delta/2$,
\[
\Delta_{\mathrm{drift}}\;\le\; \frac{2LV_{\max}}{\sqrt n}\,\sqrt{2\ln\frac{2}{\delta}}.
\]
This explicit high-probability control can be combined with the Lyapunov drift condition (SMGI (ii)) by taking $V_{\max}$ as a certified sublevel bound for the invariant set $S^*$, or (more conservatively) by deriving a high-probability bound on $\sup_{t\le n}V(s_t)$ from the moment bound $\sup_{t}\E[V(s_t)]\le BV(s_0)+B$.

\subsection{Step 3: Sequential PAC-Bayes Bound for $\Delta_{\mathrm{stat}}$}
Conditioned on the invariant trajectory manifold, we bound the residual complexity of selecting $Q_\tau$ from the prior $P_\tau$. Using a sequential PAC-Bayes framework, the Kullback-Leibler divergence is evaluated over the sequence of conditional path measures. Because the loss is bounded in $[0,1]$, we obtain with probability at least $1-\delta/2$:
\[
\Delta_{\mathrm{stat}} \le \sqrt{\frac{\mathrm{KL}(Q_\tau\|P_\tau)+\ln\frac{4\sqrt n}{\delta}}{2(n-1)}}.
\]

\subsection{Step 4: Evaluative Invariance}
Finally, SMGI condition (iv) guarantees that the evaluative constraints are preserved. This ensures that the Lipschitz constant $L$ and the Lyapunov bounds $B$ remain uniformly valid across any admissible regime transformation $\tau \in \mathcal T$, preventing the stability bounds from diverging when the system switches tasks. Combining Step 2 and Step 3 via a union bound yields the stated structural generalization bound.

\section{Appendix: Proof of Lemma~\ref{lem:wm-smgi}}
\label{app:wm-smgi-proof}

We provide an appendix-ready proof by explicitly matching the world-model interpretation map
$\Phi_{\mathrm{WM}}$ to the SMGI requirements on the pair $(\theta,T_\theta)$.

\paragraph{Setup and notation.}
Let $\theta=(r,\mathcal H,\Pi,\mathcal L,\mathcal E,\mathcal M)\in\Theta(\mathrm{WM})$.
Under $\Phi_{\mathrm{WM}}$, $r$ corresponds to an encoder $r_\phi:o_t\mapsto z_t$,
$\mathcal M$ contains the latent recurrent state (and any auxiliary world-model state),
and $\Pi$ contains the admissible update rules for $(\phi,\psi,\vartheta)$ and any state variables.
The induced dynamics $T_\theta$ is a Markov kernel
$T_\theta:S_\theta\times\mathcal Z\to\mathcal P(S_\theta)$, with deterministic evolutions recovered
as Dirac measures.

\paragraph{Typing of $\mathcal T_\theta$.}
Let $\mathcal T$ denote the ambient class of task/environment transformations
$\tau:\mathcal E\to\mathcal E'$.
For a fixed structural meta-model $\theta\in\Theta$, we define
\[
\mathcal T_\theta \;:=\; \Big\{\tau\in\mathcal T \;:\; (\theta, T_{\theta,\tau}) \text{ is admissible and well-typed in the sense of this section}\Big\}.
\]
Thus $\mathcal T_\theta\subseteq \mathcal T$ is a \emph{$\theta$-dependent restriction} of admissible task transformations, not a class of induced dynamics.

For each admissible task transformation $\tau\in\mathcal T_\theta$ we write $T_{\theta,\tau}$
for the corresponding induced kernel and $(s_t)_{t\ge 0}$ for the resulting trajectory.

\paragraph{(i) Structural closure.}
By assumption, there exists a nonempty $S^*_\theta\subseteq S_\theta$ such that for all
$\tau\in\mathcal T_\theta$, $T_{\theta,\tau}(S^*_\theta)\subseteq S^*_\theta$.
This is exactly SMGI requirement (i) (closure of the admissible set under task transformation).

\paragraph{(ii) Dynamical stability.}
Assumption (ii) provides a stability criterion on $S^*_\theta$.
Concretely, suppose there exists a Lyapunov functional $V:S_\theta\to\mathbb R_+$ and constants
$B<\infty$ and $\lambda\in(0,1]$ such that for all $\tau\in\mathcal T_\theta$ and all $s\in S^*_\theta$,
\[
\mathbb E\big[V(s_{t+1})\mid s_t=s\big] \le (1-\lambda)\,V(s) + B.
\]
Iterating the drift inequality yields a uniform bound $\sup_{t\ge 0}\mathbb E[V(s_t)]<\infty$ for all
trajectories starting in $S^*_\theta$, which is the persistence requirement (ii) for SMGI.
Any equivalent stability notion (contraction, ISS, or martingale bounds) can be substituted.

\paragraph{(iii) Statistical capacity control.}
Assumption (iii) states that the hypothesis components used by the world-model instantiation
(encoder/transition/policy classes) admit a finite complexity control, uniformly over admissible $\tau$.
This is precisely the SMGI requirement that generalization bounds can be made non-vacuous along admissible
evolution. In particular, one may use a PAC-Bayes complexity term, an MDL description-length bound, or a
VC/Rademacher bound for each component class and then combine them by a union bound or an additive
complexity decomposition.

\paragraph{(iv) Evaluative invariance.}
Assumption (iv) provides a nontrivial invariant core $\mathcal C$ of evaluative constraints preserved
under admissible regime transformations (e.g., constraints, safety conditions, or protected invariants).
This matches SMGI requirement (iv): across the admissible regime updates associated to $\tau$,
evaluation remains coherent on $\mathcal C$ (or, more generally, stays within a prescribed invariant set).

\paragraph{Proof Summary}
Items (i)--(iv) are exactly the SMGI membership conditions for the pair $(\theta,T_\theta)$.
Therefore $(\theta,T_\theta)\in\mathrm{SMGI}$, proving Lemma~\ref{lem:wm-smgi}.

\end{document}